\documentclass[12pt]{article}
\usepackage[margin=1in]{geometry}

\usepackage[numbers]{natbib}
\usepackage[utf8]{inputenc} 
\usepackage[T1]{fontenc}    
\usepackage{url}            
\usepackage{booktabs}       
\usepackage{amsfonts}       
\usepackage{nicefrac}       
\usepackage{microtype}      

\usepackage{morenotations,rotating}
\usepackage{color,colortbl,soul,upgreek}
\usepackage{arydshln}
\usepackage{wrapfig}
\usepackage{soul}
\usepackage{amsmath}
\usepackage{empheq}
\usepackage{mdframed,enumerate}

\usepackage{graphicx}
\usepackage{caption}
\usepackage{subcaption}

\usepackage{authblk}
\definecolor{Gray}{gray}{0.9}

\usepackage{xr-hyper}
\usepackage{tikz, pgfplots}
\usetikzlibrary{intersections, calc, positioning, decorations.pathreplacing}
\pgfplotsset{compat=1.8, xlabel style={anchor=west, align=center}, ylabel style={anchor=south, align=center}, samples=200, ymin=0, ymax=1, width=3cm, height=3cm, axis lines=middle, xticklabel style={/pgf/number format/.cd,frac,frac TeX=\frac}, yticklabel style={/pgf/number format/.cd,frac,frac TeX=\frac}, xtick=\empty,ytick=\empty, no markers, cycle list={{black,solid}}, samples=200}

\newcommand{\E}[2]{\underset{#1}{\mathbb{E}}\left[ #2 \right]}

\newcommand{\Real}{\mathbb{R}}

\newcommand{\amari}{\cite{aPCOMM}}
\newcommand{\gan}{\mbox{\tiny{\textsc{gan}}}}
\newcommand{\vout}{v_{\mbox{\tiny{\textsc{out}}}}}
\newcommand{\vpout}{v'_{\mbox{\tiny{\textsc{out}}}}}
\newcommand{\vvout}{\ve{v}_{\mbox{\tiny{\textsc{out}}}}}
\newcommand{\net}{{\mbox{{\tiny net}}}}
\newcommand{\chinet}{\chi_\net}

\title{$f$-GANs in an Information Geometric Nutshell}
\author{
  Richard Nock$^{\dagger,\ddagger,\S}$ \qquad Zac
  Cranko$^{\ddagger,\dagger}$ \qquad Aditya Krishna
  Menon$^{\dagger,\ddagger}$

Lizhen Qu$^{\dagger,\ddagger}$ \qquad Robert
C. Williamson$^{\ddagger,\dagger}$\\

 $^{\dagger}$Data61, $^\ddagger$the Australian National University
  and $^\S$the University of Sydney\\
{\normalsize \texttt{\{firstname.lastname, aditya.menon, bob.williamson\}@data61.csiro.au}} \\
}
\begin{document}

\date{}

\maketitle

\begin{abstract}
Nowozin \textit{et al} showed
last year how to extend the GAN \textit{principle} to all $f$-divergences. The
approach is elegant but falls short of a full description of the supervised game, and
says little about the key player, the generator: for example,
what does the generator actually converge to if solving the GAN game means convergence in some
space of parameters? How does that provide hints on the generator's design and
compare to the flourishing but almost exclusively experimental literature on the
subject?

In this paper, we unveil a broad class of distributions for which such
convergence happens --- namely, deformed exponential families, a wide
superset of exponential families --- and show tight connections with the three
other key GAN parameters: loss, game and architecture. In particular, we show that current deep architectures are
able to factorize a very large number of
such densities using an especially compact design, hence displaying the power of deep architectures and their concinnity in
the $f$-GAN game. This result holds given a sufficient condition on
\textit{activation functions} ---  which turns out to be
satisfied by popular choices. The key to our results is a variational
generalization of an old theorem that relates the KL divergence between regular exponential
families and divergences between their natural
parameters. We complete this picture with additional results and experimental insights on
how these results may be used to ground further improvements of GAN
architectures, via (i) a principled design of the activation
functions in the generator and (ii) an explicit integration of proper composite losses' link function in the discriminator.
\end{abstract}

\newpage

\section{Introduction}

In a recent paper, Nowozin \textit{et al.} \cite{nctFG} showed that the GAN
principle \cite{gpmxwocbGA} can be extended to the variational formulation of all
$f$-divergences. In the GAN game, there is an unknown distribution
$\dist{P}$ which we want to approximate using a parameterized
distribution $\dist{Q}$. $\dist{Q}$ is learned by a \textbf{generator} by finding a saddle point of a function which
we summarize for now as $f$-GAN($\dist{P}$, $\dist{Q}$), where $f$ is
a convex function (see
eq. (\ref{defVARGAN}) below for its formal expression). A part of the
generator's training involves as a subroutine a supervised \textit{adversary} --- hence, the saddle
point formulation -- called \textbf{discriminator}, which tries to
guess whether randomly generated observations come from $\dist{P}$ or
$\dist{Q}$. Ideally, at the end of this \textit{supervised game}, we want $\dist{Q}$ to be
close to $\dist{P}$, and a good measure of this is the $f$-divergence
$I_f (\dist{P}\|\dist{Q})$, also known as Ali-Silvey distance \cite{asAG,cIT}. Initially, one choice of
$f$ was considered \cite{gpmxwocbGA}. Nowozin
\textit{et al.} significantly grounded the game and expanded its scope
by showing that for any $f$ convex and suitably
defined, it actually holds that \citep[Eq. 4]{nctFG}:
\begin{eqnarray}
\boxed{\mbox{$f$-GAN($\dist{P}$, $\dist{Q}$) $\leq$ $I_f (\dist{P}\|\dist{Q})$}} \:\:.\label{eqNOVO}
\end{eqnarray}
Furthermore, the inequality is an
\textit{equality} if the discriminator is powerful enough: so, solving the
$f$-GAN game can give guarantees on how $\dist{P}$ and $\dist{Q}$ are
distant to each other in terms of $f$-divergence.
This elegant characterization of the
supervised game unfortunately falls short of
justifying or elucidating all parameters of the supervised game
\citep[Section 2.4]{nctFG}. 

The paper
is also silent regarding a key part of the game: the link between
distributions in the variational formulation and the \textit{generator},
the main player which learns a parametric model of a density. In doing so,
the $f$-GAN approach and its members
remain within an information theoretic framework that relies on
divergences between distributions only \cite{nctFG}. In the GAN world
at large,
this position contrasts with other 
prominent approaches that explicitly optimize \textit{geometric} distortions between the
parameters or support of distributions \cite{lbcAA}: moment matching
methods optimize distortions between expected parameters \cite{lszGM},
Wasserstein-1 method and optimal transport methods (regularized or not) optimize
transportation costs between supports
\cite{acbWG,gaadcIT,gpcSA}. This problem of connecting the information
theoretic and (information) geometric understanding of GANs is not just a
theoretical question: there is growing experimental evidence that a careful geometric
optimization, either on the support of the distributions \cite{acbWG,gaadcIT} or
directly on these parameters \cite{sgzcrcIT} (which is related
to the $f$-GAN framework) improves further GANs.

So, how can we link the $f$-GAN approach to any sort of
information \textit{geometric} optimization? The variational
formulation of the GAN game in eq. (\ref{eqNOVO}) hints on a specific 
direction of research to answer this question: the identity between 
information-theoretic distortions on \textit{distributions} and
information-geometric distortions on their
\textit{parameterization} \cite{anMO}. One such identity is well known: The Kullback-Leibler (KL) divergence between two distributions of the
\textit{same (regular) exponential family} equals a Bregman divergence
$D$ between their
natural parameters \cite{aDM,anMO,awRLBEF,bnnBV,tdAB}, which we can
summarize for now (the complete statement is in Theorem \ref{thmSTAND} below) as:
\begin{eqnarray}
\boxed{\mbox{$I_{f_{\textsc{kl}}}(\dist{P}\|\dist{Q})$ =
    $D(\ve{\ve{\theta} \| \vartheta})$}} \:\:.\label{eqKLBREG1}
\end{eqnarray}
Here, $\ve{\theta}$ and $\ve{\vartheta}$ are respectively the natural
parameters of $\dist{P}$ and $\dist{Q}$. Hence, distributions are represented by points on
a manifold on the right-hand side, which is a powerful geometric
statement \cite{anMO}; however, being restricted to KL divergence or
"just" exponential families, it certainly falls short of the power
to explain the GAN game.
To our knowledge, there is no previously known "GAN-amenable" generalization
of this identity above exponential families. Related identities have
recently been proven for two generalizations of exponential families
\citep[Theorem 9]{aomGO}, \citep[Theorem 3]{frCF}, but fall short of the
$f$-divergence formulation and are not amenable to the variational GAN
formulation. \\

\noindent\textbf{Our first contribution} is such an identity that
connects the general $I_f$-divergence formulation in eq. (\ref{eqNOVO})
to the general $D$ (Bregman) divergence formulation in
eq. (\ref{eqKLBREG1}). We
now briefly state it, postponing the details to Section \ref{secGEN}:
\begin{eqnarray}
\boxed{\mbox{$f$-GAN($\dist{P}$, \textit{escort}($\dist{Q}$)) =
  $D$($\ve{\theta}\|\ve{\vartheta}$) + Penalty($\dist{Q}$)}} \:\:,\label{eqROUGH}
\end{eqnarray}
for $\dist{P}$ and $\dist{Q}$ (with respective parameters $\ve{\theta}$ and
$\ve{\vartheta}$) which happen to lie in a superset of exponential families called
\textit{deformed exponential families}, that have received extensive treatment
in statistical physics and differential information geometry over the
last decade \cite{aIG,nGT}. The right-hand side of eq. (\ref{eqROUGH}) is the information
geometric part \cite{anMO}, in which $D$ is a Bregman divergence. Therefore, whenever the Penalty is small, solving
the $f$-GAN game solves a geometric optimization problem \cite{anMO},
like for the Wasserstein GAN and its variants \cite{acbWG}, but with
the difference that the geometric part is essentially implicit.
Notice also that $\dist{Q}$ appears in the game in
the form of an \textit{escort}: its density is obtained from
$\dist{Q}$'s density through a mapping (in general non-linear)
completed with a simple normalization \cite{aomGO}. These differences
vanish only for exponential families: the mapping is the identity and
thus \textit{escort}($\dist{Q}$) = $\dist{Q}$; also,
Penalty($\dist{Q}$) = 0 and $f$ = KL. This raises questions as to how
eq. (\ref{eqROUGH}) and these differences relate
to GAN
architectures and the common understanding and implementation of the
general ($f$-)GAN game \cite{gpmxwocbGA,nctFG}.\\

\noindent\textbf{Our second contribution} answers several of these
questions via several independent results. A subset is relevant to the
$f$-GAN game at large:
\begin{itemize}
\item [(a)] we completely specify
the parameters of the supervised game, unveiling a key parameter left
arbitrary in \cite{nctFG} (explicitly incorporating the link function of proper composite
losses \cite{rwCB});
\item [(b)] we develop a novel min-max game
interpretation of eq. (\ref{eqROUGH}) in the context of the expected
utility theory \cite{bbtCE};
\item [(c)] we show that
relevant choices for escorts yield explicit upper bounds on the
Penalty which vanish with the normalization coefficient of the
escort. 
\end{itemize}
Another subset dwells on deep architectures:
\begin{itemize}
\item [(d)] we show that
typical deep generator architectures are indeed powerful at modelling
\textit{complex} escorts of any deformed exponential family,
factorising a number of escorts in
order of the \textit{total} inner layers' dimensions; this provides theoretical
support for the widespread empirical speculations that deep architectures may be powerful at modeling
highly multimodal densities, which is a hot topic in the
field \cite{azDG};
\item [(e)] we show that this factorisation happens on an especially compact
  model design, compared \textit{e.g.} to shallow architectures;
\item [(f)] we derive a connection between the parameters of the deformed exponential families
and those of the
generator. Quite notably,
the activation function gives the deformed exponential family.
\end{itemize}
The connection
between the generator and escorts via eq. (\ref{eqROUGH}) supports the
use of
geometric parameter based optimisation in the GAN game \cite{sgzcrcIT}. It
suggests the existence of a large class of activation functions for
which the factorisation in deformed exponential families holds as
described in (d-f). In a field where such functions have been
the subject of intensive research \cite{cuhFA,mhnRN,nhRL} and face
numerous constraints in their design \cite{acbWG,pmbOT}, the study of
this class is not just important for the theory at hand: it is also of high practical relevance.\\

\noindent\textbf{Our last contribution} studies this class and details
several theoretical and experimental findings. We show that a simple sufficient
condition on the activation function guarantees the escort modelling
in (d), (f). Such a condition still allows for properties in activation
that handle sparsity, gradient vanishing, gradient exploding and/or
Lipschitz continuity \cite{acbWG,gbbDS,pmbOT}. In fact, this condition
is satisfied, exactly or in a limit sense, by most popular
activation functions (ELU, ReLU, Softplus, ...). We also provide experiments that display the uplift that
can be obtained through tuning the activations (generator), or
the link function (discriminator).\\
The rest of this paper is as follows. Section $\S$ \ref{secDEF}
presents definition, $\S$ \ref{secGEN} formally presents
eq. (\ref{eqROUGH}), $\S$ \ref{sec-sup} completes the supervised game
picture of \cite{nctFG}, 
$\S$ \ref{sec-deep-all} derives a number of
consequences for deep learning, including distributions achieved
by deep architectures for the generator. Section $\S$ \ref{sec-expes} presents experiments
and a last Section concludes. An \supplement~contains all proofs
and complementary experiments. Since our paper drills down into the
four components of the GAN game (loss, distribution, game and
architectures = models), we summarize for clarity in \SI~(Section \ref{sec-sum}) our main
notations, the objects they refer to and their relationships through
some of our key results.\\

\noindent\textbf{Code availability} --- the code used for our
experiments is available through
\begin{eqnarray*}
\boxed{\mbox{https://github.com/qulizhen/fgan\_info\_geometric}}
\end{eqnarray*}

\section{Definitions}\label{secDEF}

Throughout this paper, the \textit{domain} $\mathcal{X}$ of
\textit{observations} is a measurable set.
We begin with two important classes of distortion measures, $f$-divergences
and Bregman divergences.
\begin{definition}\label{defDISTORT}
For any two distributions $\dist{P}$ and $\dist{Q}$ having respective densities
$\dens{P}$ and $\dens{Q}$ absolutely continuous with
respect to a base measure $\mu$, the
$f$-divergence between $\dist{P}$ and $\dist{Q}$, where $f : \mathbb{R}_+
\rightarrow \mathbb{R}$ is convex with $f(1) = 0$, is
\begin{eqnarray}
I_f(\dist{P}\|\dist{Q}) & \defeq & \expect_{\X\sim \dist{Q}}
\left[f\left(\frac{\dens{P}(\X)}{\dens{Q}(\X)}\right)\right] = \int_{\mathcal{X}} \dens{Q}(\ve{x}) \cdot
f\left(\frac{\dens{P}(\ve{x})}{\dens{Q}(\ve{x})}\right) \mathrm{d}\mu(\ve{x})\:\:.\label{defFDIV}
\end{eqnarray}
For any convex 
differentiable $\varphi : \mathbb{R}^d \rightarrow \mathbb{R}$, the
($\varphi$-)Bregman divergence between $\ve{\theta}$ and $\ve{\varrho}$ is:
\begin{eqnarray}
D_\varphi(\ve{\theta}\| \ve{\varrho}) & \defeq &
\varphi(\ve{\theta}) - \varphi(\ve{\varrho}) - (\ve{\theta} - \ve{\varrho})^\top \nabla\varphi(\ve{\varrho})\:\:, \label{defBDIV}
\end{eqnarray}
where $\varphi$ is called the generator of the Bregman divergence.
\end{definition}
$f$-divergences are the key distortion measure of
information theory. Under mild assumptions, they are the only
distortions that satisfy the
\textit{data processing inequality} \cite{jcnvwID,pvAD}. Bregman
divergences are the key distortion measure of information geometry.
Under mild assumptions, they are the only
distortions that elicitate the
\textit{sample average as a population minimizer}
\cite{anMO,bgwOT,nnaOCD,vehRD}.

A distribution $\dist{P}$ from a (regular) exponential
family with cumulant $C : \Theta \rightarrow \mathbb{R}$ and
sufficient statistics $\ve{\phi}: \mathcal{X} \rightarrow \mathbb{R}^d$ has
density 
\begin{eqnarray}
\dens{P}_C(\ve{x} | \ve{\theta}, \ve{\phi}) & \defeq & \exp(\ve{\phi}(\ve{x})^\top
\ve{\theta} - C(\ve{\theta}))\:\:,\label{defEXPFAM1}
\end{eqnarray}
where $\Theta$ is a convex open set, $C$ is convex and ensures
normalization on the simplex (we leave implicit the
associated dominating measure \cite{aIG}). A fundamental
Theorem ties Bregman divergences and $f$-divergences.
\begin{theorem}\label{thmSTAND}\cite{aIG,bnnBV}
Suppose $\dist{P}$ and $\dist{Q}$ belong to the same exponential
family, and denote their respective densities $\dens{P}_C(\ve{x} | \ve{\theta}, \ve{\phi})$ and $Q_C(\ve{x} |
\ve{\vartheta}, \ve{\phi})$. Then,
\begin{eqnarray}
I_{\textsc{kl}}(\dist{P}\|\dist{Q}) & = & D_C(\ve{\vartheta}\| \ve{\theta})\:\:.
\end{eqnarray} 
Here,
$I_{\textsc{kl}}$ is Kullback-Leibler (KL) $f$-divergence ($f\defeq x
\mapsto x \log x$).
\end{theorem}
Remark that the arguments in the Bregman divergence are permuted with
respect to those in eq. (\ref{eqKLBREG1}) in the introduction. This
also holds if we consider $f_{\textsc{kl}}$ in eq. (\ref{eqKLBREG1})
to be the Csisz\'ar dual of $f$ in Theorem \ref{thmSTAND}
\cite{bbtCE}, namely $f_{\textsc{kl}} : x \mapsto -\log x$, since in
this case $I_{f_{\textsc{kl}}}(\dist{P}\|\dist{Q}) =
I_{\textsc{kl}}(\dist{Q}\|\dist{P}) = D_C(\ve{\theta}\|\ve{\vartheta})$. We made
this choice in the introduction for the sake of readability in
presenting eqs. (\ref{eqNOVO} --- \ref{eqROUGH}). Theorem
\ref{thmSTAND} is useful because it shows that
distributions can be replaced by their parameterisation (and \textit{vice
  versa}) to tackle a problem --- we just need to pick the right
distortion for the objects at hand. There is analytic convenience
in this: for example, the Bregman divergence bypasses sampling issues to estimate the integral in the
$f$-divergence --- at the expense of the estimation of the parameters,
though.  In fact, Theorem \ref{thmSTAND} is so important that we
state and prove a generalization of it in \SI, Section
\ref{SECGEN}, showing that dropping the "same family" constraint
does not change the $f$-divergence (information-theoretic) vs Bregman
divergence (information-geometric) picture.

We now define generalizations of exponential families, following \cite{aomGO,frCF}. Let $\chi :
\mathbb{R}_+ \rightarrow \mathbb{R}_+$ be non-decreasing
\citep[Chapter 10]{nGT}. We define the
$\chi$-logarithm, $\log_\chi$, as 
\begin{eqnarray}
\log_\chi(z) & \defeq & \int_1^z
\frac{1}{\chi(t)} \mathrm{d}t\:\:.\label{defLOGCHI}
\end{eqnarray} 
The $\chi$-exponential is 
\begin{eqnarray}
\exp_\chi(z)
& \defeq & 1 + \int_0^z \lambda(t) \mathrm{d}t\:\:,\label{defEXPCHI}
\end{eqnarray} 
where $\lambda$ is defined
by $\lambda(\log_\chi(z)) \defeq \chi(z)$. In the case where the
integrals are improper, we consider the corresponding limit in the
argument / integrand.
\begin{definition}\cite{aomGO}\label{defDEF}
A distribution $\dist{P}$ from a $\chi$-exponential family (or deformed exponential family, $\chi$ being
implicit) with convex cumulant $C :
\Theta \rightarrow \mathbb{R}$ and sufficient statistics $\ve{\phi} :
\mathcal{X} \rightarrow \mathbb{R}^d$ has density given by: 
\begin{eqnarray}
\dens{P}_{\chi, C}(\ve{x}|\ve{\theta}, \ve{\phi})
& \defeq & \exp_\chi(\ve{\phi}(\ve{x})^\top \ve{\theta} -
C(\ve{\theta}))\:\:,\label{defDEFORMED}
\end{eqnarray} 
with respect to a dominating measure $\mu$. Here, $\Theta$ is a convex
open set and 
$\ve{\theta}$ is
called the coordinate of $\dist{P}$. 
The
\textbf{escort density} (or $\chi$-escort) of $\dens{P}_{\chi, C}$ is
\begin{eqnarray}
\escort{P}_{\chi, C} & \defeq & \frac{1}{Z} \cdot \chi(\dens{P}_{\chi,
  C})\:\:,\label{defescort} 
\end{eqnarray}
where 
\begin{eqnarray}
Z & \defeq & \int_{\mathcal{X}}
\chi(\dens{P}_{\chi, C}(\ve{x}|\ve{\theta}, \ve{\phi}))
\mathrm{d}\mu(\ve{x})\label{defNORMAL}
\end{eqnarray} 
is the escort's normalization constant.
\end{definition}
We leaving implicit the dominating measure and denote $\escort{\dist{P}}$
the escort distribution of $\dist{P}$ whose density is given by eq. (\ref{defescort}).
 We shall name $\chi$ the \textit{signature} of the deformed (or
 $\chi$-)exponential family, and sometimes drop indexes to save
 readability without ambiguity, noting \textit{e.g.} $\escort{P}$ for $\escort{P}_{\chi, C}$.
Notice that normalization in the
escort is ensured by a simple integration \citep[Eq. 7]{aomGO}. For
the escort to exist, we require that $Z<\infty$ and therefore
$\chi(P)$ is finite almost everywhere. Such a requirement would naturally be satisfied in the
GAN game.

There is another
generalization of regular exponential families, known as \textit{generalized
exponential families} \cite{frCF} (\SI, Section \ref{SECGEN}). Their
densities are defined from the subdifferential of a convex
function, but involves an inner product similar to eq. (\ref{defDEFORMED}). 
There is no known strict equivalent of Theorem \ref{thmSTAND} for whichever of
the generalizations. For example, \cite[Theorem 3]{frCF} provides a
generalization of Theorem \ref{thmSTAND} \textit{but} replaces KL by a Bregman
divergence\footnote{Under mild assumptions on support and functions,
  $\{$KL$\}$ = $f$-divergences $\cap$
Bregman divergences \cite{jcnvwID}.}. The closest result appears for
deformed exponential families \citep[Theorem 9]{aomGO}\cite{vcOP}.
\begin{theorem}\label{thmSTANDCHI}\cite{aomGO}\cite{vcOP} for any two
  $\chi$-exponential distributions $\dist{P}$ and $\dist{Q}$ with
  respective densities $\dens{P}_{\chi, C},
Q_{\chi, C}$ and coordinates $\ve{\theta}$, $\ve{\vartheta}$,
\begin{eqnarray}
D_C(\ve{\theta}\|\ve{\vartheta}) & = &
\expect_{\X \sim \escort{\dist{Q}}}[\log_\chi(\dens{Q}_{\chi, C}(\X)) -
\log_\chi(\dens{P}_{\chi, C}(\X))]\:\:.\label{eqDCDEF}
\end{eqnarray}
\end{theorem}
Theorem \ref{thmSTANDCHI} is a generalization of Theorem
\ref{thmSTAND} for $\chi(z) \defeq z$, in which case $\log_\chi =
\log, \exp_\chi = \exp$ and escorts disappear: $\escort{\dist{Q}} = \dist{Q}$.
There are two important things to notice in eq. (\ref{eqDCDEF}):
\begin{itemize} 
\item the expectation is computed
over the \textit{escort} of $\dist{Q}$; 
\item the difference of two
$\chi$-logarithms is in general \textit{not} the $\chi$-logarithm of
the density ratio. 
\end{itemize}
The $f$-GAN game relies on
distortions being formulated via convex functions over density
ratios. As such, 
Theorem \ref{thmSTANDCHI} is not amenable to the
variational $f$-GAN formulation \cite[Section 2.2]{nctFG}. 
In the
following Section, we show how to achieve this goal, but before, we
briefly frame the now popular ($f$-)GAN adversarial learning
\cite{gpmxwocbGA,nctFG}. 

We have a true unknown distribution $\dist{P}$
over a set of objects, \textit{e.g.} 3D pictures, which we want
to learn. In the GAN setting, this is the objective of a
\textit{generator}, who learns a distribution $\dist{Q}_{\ve{\theta}}$ parameterized by vector
$\ve{\theta}$. $\dist{Q}_{\ve{\theta}}$ works by passing (the
support of) a simple,
uninformed distribution, \textit{e.g.} standard Gaussian, through a
possibly complex function, \textit{e.g.} a deep net whose parameters are
$\ve{\theta}$ and maps to the
support of the objects of interest. Fitting $\dist{Q}_.$ involves an
\textit{adversary} (the discriminator) as subroutine, which fits \textit{classifiers}, \textit{e.g.} deep
nets, parameterized by $\ve{\omega}$. The generator's objective is to
come up with
$\arg \min_{\ve{\theta}} L_f(\ve{\theta})$ with $L_f(\ve{\theta})$ the
discriminator's objective:
\begin{eqnarray}
L_f(\ve{\theta}) & \defeq & \sup_{\ve{\omega}} \{\expect_{\X \sim
  \dist{P}}[T_{\ve{\omega}}(\X)] - \expect_{\X \sim
  \dist{Q}_{\ve{\theta}}}[f^\star(T_{\ve{\omega}}(\X))]\}\label{defVARGAN}\:\:,
\end{eqnarray}
where $\star$ is Legendre conjugate \cite{bvCO} and $T_{\ve{\omega}} :
\mathcal{X}\rightarrow \mathbb{R}$ integrates the classifier of the
discriminator and is therefore parameterized by $\ve{\omega}$. $L_f$ is a variational approximation to a $f$-divergence \cite{nctFG};
the discriminator's objective is to segregate
true ($\dist{P}$) from fake ($\dist{Q}_.$) data. The original GAN choice,
\cite{gpmxwocbGA} 
\begin{eqnarray}
f_{\gan}(z) & \defeq & z\log z -
(z+1)\log(z+1) + 2\log 2\label{defFGAN}
\end{eqnarray} 
(the constant ensures $f(1)=0$) can be replaced by any convex $f$ meeting mild
assumptions.

\section{A variational information
  geometric identity for the $f$-GAN game}\label{secGEN}

We now make a series of Lemmata and Theorems that will bring us to
formalize eq. (\ref{eqROUGH}), in two main steps: first, we
show that the right-hand side of eq. (\ref{eqDCDEF}) in Theorem
\ref{thmSTANDCHI} can be reformulated using a new set of distortion
measures which is amenable to the variational $f$-GAN
formulation. Second, we connect this variational formulation to the
classical $f$-GAN game \cite{nctFG} by
showing that, modulo finiteness conditions that make sense to the GAN game, this new set of distortion
measures essentially coincides with $f$-divergences.\\

\noindent\textbf{$KL_{\chi}$ divergences} --- First, we define this new
set of distortion measures, that we call $KL_{\chi}$ divergences.
\begin{definition}\label{defKLCHI}
For any $\chi$-logarithm and distributions $\dist{P}, \dist{Q}$ having
respective densities $P$ and $Q$ absolutely continuous with respect to
base measure $\mu$, the $KL_\chi$
divergence between $\dist{P}$ and $\dist{Q}$ is defined as:
\begin{eqnarray}
KL_{\chi}(\dist{P}\|\dist{Q}) & \defeq & \expect_{\X\sim \dist{P}}\left[-\log_{\chi}\left(\frac{\dens{Q}(\X)}{P(\X)}\right)\right]\:\:.\label{klchi}
\end{eqnarray}
\end{definition}
Since $\chi$ is non-decreasing, $-\log_{\chi}$ is convex and so any $KL_\chi$
divergence is an $f$-divergence. When $\chi(z) \defeq z$, $KL_\chi$ is
the KL divergence. In what follows, base measure $\mu$ and absolute
continuity are implicit, as well as that $\dens{P}$ (resp. $\dens{Q}$) is
the density of $\dist{P}$ (resp. $\dist{Q}$).
In the same way as $f$ divergences are invariant to
specific affine translations (see the proof of Theorem \ref{thVIGGEN1}), $KL_\chi$ divergences satisfy an
interesting invariance.
\begin{lemma}\label{lemINV1}
For any $\chi$-logarithm, distributions $\dist{P}, \dist{Q}$ and constant $k\in
\mathbb{R}_+$, 
\begin{eqnarray}
KL_{\chi}(\dist{P}\|\dist{Q}) & \defeq & KL_{\frac{\chi}{1+k\chi}}(\dist{P}\|\dist{Q}) \:\:.\label{klchiTRSF}
\end{eqnarray}
\end{lemma}
(Proof in \SI, Section \ref{proof_lemINV1}) Hence, we can in fact
assume that any $KL_\chi$ divergence is obtained for a signature which
is bounded.\\

\noindent\textbf{$KL_{\chi}$ divergences vs $f$-divergences} --- Let $\partial
f$ be the subdifferential of convex $f$ and $\mathbb{I}_{\dens{P},\dens{Q}} \defeq [\inf_{\ve{x}} \dens{P}(\ve{x})/\dens{Q}(\ve{x}),
\sup_{\ve{x}} \dens{P}(\ve{x})/\dens{Q}(\ve{x})) \subseteq \mathbb{R}_+$ denote the range of density ratios
of $\dens{P}$ over $\dens{Q}$. Our first result states that if there is an element
of the subdifferential which is upperbounded on $\mathbb{I}_{\dens{P},\dens{Q}}$,
the $f$-divergence $I_f(\dist{P}\|\dist{Q})$ is equal to a $KL_\chi$ divergence.
\begin{theorem}\label{thVIGGEN1}
Suppose that $\dist{P}, \dist{Q}$ are such that 
$\exists \xi \in \partial f$ with $\sup \xi(\mathbb{I}_{P,\dens{Q}}) <
\infty$. Then $\exists
\chi : \mathbb{R}_+ \rightarrow \mathbb{R}_+$ non decreasing such that $I_f(\dist{P}\|\dist{Q}) = KL_\chi(\dist{Q}\|\dist{P})$.
\end{theorem}
\begin{remark}
Notice that because the constraint relies on the subdifferential, it actually does not
prevent the $f$-divergence to diverge. Also, Theorem \ref{thVIGGEN1}
essentially covers most if not all relevant GAN cases, as the assumption 
has to be satisfied in the GAN game for its solution not to be
vacuous up to a large extent (eq. (\ref{defVARGAN})). Indeed, if the subdifferential diverges on
a finite ratio, then the optimal $\ve{\omega}$ makes $T_{\ve{\omega}}$
explode on some $\ve{x}$ \citep[Eq. 5]{nctFG}. If it diverges on an
infinite ratio, then $\lim_{+\infty}f(z) = +\infty$ and essentially
$I_f(\dist{P}\|\dist{Q})$ is unbounded. In this case, $\dens{Q}$
vanishes in the neighborhood of some $\ve{x} \in \mathcal{X}$ for
which $P>0$. We can make
$L_f(\ve{\theta})$ artificially large by just picking $\ve{\omega}$
such that $T_{\ve{\omega}}$ is
as large as necessary in such a neighborhood: the discriminator only
focuses on one "pit" of $\dens{Q}$ (relative to $\dens{P}$) to detect natural
examples, which is not an appealing solution to the GAN game.
\end{remark}\\
The proof of Theorem \ref{thVIGGEN1} (in \SI, Section \ref{proof_thVIGGEN1}) is constructive: it
shows how to pick $\chi$ which satisfies all requirements. It brings
the following interesting corollary: under mild assumptions on $f$,
there exists a $\chi$ that fits for all densities $\dens{P}$ and $\dens{Q}$. A
prominent example of $f$ that fits is the original GAN choice for
which we can pick 
\begin{eqnarray}
\chi_{\gan}(z) & \defeq &
\frac{1}{\log\left(1+\frac{1}{z}\right)}\:\:.\label{defCHIGAN}
\end{eqnarray}
\begin{corollary}\label{corGEN}
Suppose $\exists \xi \in \partial
f$ with $\sup \xi(\mathrm{int } \mathrm{dom} f) <
\infty$. Then $\exists
\chi : \mathbb{R}_+ \rightarrow \mathbb{R}_+$ increasing such that
for any distributions $\dist{P}$, $\dist{Q}$, $I_f(\dist{P}\|\dist{Q}) = KL_\chi(\dist{Q}\|\dist{P})$.
\end{corollary}
\begin{remark}
Even when $f$ does not satisfy Corollary \ref{corGEN}, it may well be
the case that its Csisz\'ar dual does \cite{bbtCE}, or equivalently, that Corollary
\ref{corGEN} holds if we permute the arguments in one of the
distortions. Let $f_\diamond(z) \defeq z
\cdot f(1/z)$. We have $I_f(\dist{P}\|\dist{Q}) = I_{f_\diamond}(\dist{Q}\|\dist{P})$. Then, for
example, picking $f(z) = z\log z$ (KL) does not fit to Corollary
\ref{corGEN} but picking $f_\diamond(z) = -\log z$ (reverse KL)
does. Picking Pearson $\chi^2$ ($f(z) = (z-1)^2)$ does not fit to Corollary
\ref{corGEN} but picking $f_\diamond(z) = (1/z)\cdot (z-1)^2$ (Neyman
$\chi^2$) does.
\end{remark}\\
We now show that when the subdifferential diverges (but $I_f$ is
finite), it it still
possible to approximate $I_f(\dist{P}\|\dist{Q})$ by some  $KL_\chi$ divergence, up
to any required precision. 
\begin{theorem}\label{thVIGGEN2}
Suppose that $\dist{P}, \dist{Q}$ are such that 
$\sup \xi(\mathbb{I}_{P,\dens{Q}}) = +\infty, \forall \xi \in \partial f$, but 
$I_f(\dist{P}\|\dist{Q}) < +\infty$, then $\forall \delta
> 0$, $\exists
\chi : \mathbb{R}_+ \rightarrow \mathbb{R}_+$ increasing such that
\begin{eqnarray}
KL_\chi(\dist{Q}\|\dist{P}) \leq I_f(\dist{P}\|\dist{Q}) \leq KL_\chi(\dist{Q}\|\dist{P}) + \delta\:\:.
\end{eqnarray}
\end{theorem}
(Proof in \SI, Section \ref{proof_thVIGGEN2})\\

\noindent\textbf{A $KL_{\chi}$ divergences formulation for Theorem \ref{thmSTANDCHI}} --- To connect $KL_\chi$-divergences
and Theorem \ref{thmSTANDCHI}, we need a slight generalization of
$KL_\chi$-divergences and
allow for $\chi$ in eq. (\ref{klchi}) to depend on the choice of the
expectation's $\X$, granted that for any of these choices, it will
meet the constraints to be $\mathbb{R}_+ \rightarrow \mathbb{R}_+$ and
  also increasing, and therefore define a valid signature. For any $f
  : \mathcal{X} \rightarrow \mathbb{R}_+$, we denote
\begin{eqnarray}
KL_{\chi_f}(\dist{P}\|\dist{Q}) & \defeq & \expect_{\X\sim \dist{P}}\left[-\log_{\chi_{f(\X)}}\left(\frac{Q(\X)}{P(\X)}\right)\right]\:\:,\label{klchiF}
\end{eqnarray}
where for any $p \in \mathbb{R}_+$,
\begin{eqnarray}
\chi_{p}(t) & \defeq & \frac{1}{p} \cdot \chi
(t p)\:\:.
\end{eqnarray}
Whenever $f = 1$, we just write $KL_\chi$ as we already did in
Definition \ref{defKLCHI}. We note that for any $\ve{x}\in \mathcal{X}$,
$\chi_{f(\ve{x})}$ is increasing and non negative because of the
properties of $\chi$ and $f$, so $\chi_{f(\ve{x})}(t)$ defines a
$\chi$-logarithm. We also note that
the invariance of Lemma
\ref{lemINV1} holds as well for $KL_{\chi_f}(\dist{P}\|\dist{Q})$. With
this generalization of $KL_\chi$, we are ready to state a Theorem that
connects $KL_\chi$-divergences
and Theorem \ref{thmSTANDCHI}.
\begin{theorem}\label{thVIGGEN3}
Letting
$P\defeq P_{\chi, C}$ and $Q \defeq Q_{\chi, C}$ for short in Theorem
\ref{thmSTANDCHI}, we have:
\begin{eqnarray}
\expect_{\X\sim \escort{\dist{Q}}}[\log_\chi(Q(\X)) - \log_\chi(P(\X))] & = &
KL_{\chi_{\escort{{Q}}}}(\escort{\dist{Q}}\|\dist{P})  - J(\dist{Q})\:\:,
\end{eqnarray} 
with 
\begin{eqnarray}
J(\dist{Q}) & \defeq &
KL_{\chi_{\escort{{Q}}}}(\escort{\dist{Q}}\|\dist{Q})\:\:.
\end{eqnarray} 
\end{theorem}
(Proof in \SI, Section \ref{proof_thVIGGEN3}) To summarize, we know
that under mild assumptions relatively to the GAN game,
$f$-divergences coincide with $KL_\chi$ divergences (Theorems
\ref{thVIGGEN1}, \ref{thVIGGEN2}). We also know from Theorem
\ref{thVIGGEN3} that $KL_{\chi_.}$ divergences quantify the
geometric proximity between the coordinates of generalized exponential
families (Theorem \ref{thmSTANDCHI}). Hence, finding a geometric
(parameter-based) interpretation of the variational $f$-GAN game as described in eq. (\ref{defVARGAN})
can be done via a variational formulation of the $KL_\chi$ divergences
appearing in Theorem \ref{thVIGGEN3}.\\

\noindent\textbf{A variational formulation for $KL_{\chi}$ divergences} --- Since penalty
$J(\dist{Q})$ does not belong to the GAN game (it does not depend on $\dist{P}$), it
reduces our focus on $KL_{\chi_{\escort{{Q}}}}(\escort{\dist{Q}}\|\dist{P})$.
\begin{theorem}\label{thVIGGEN4}
$KL_{\chi_{\escort{{Q}}}}(\escort{Q}\|P)$
admits the variational formulation
\begin{eqnarray}
KL_{\chi_{\escort{{Q}}}}(\escort{\dist{Q}}\|\dist{P}) & = & \sup_{T \in \overline{\mathbb{R}_{++}}^{\mathcal{X}}} \left\{\expect_{\X\sim \dist{P}} [T(\X)]  -
\expect_{\X\sim\escort{\dist{Q}}}
[(-\log_{\chi_{\escort{Q}}})^\star(T(\X))]\right\} \:\:, \label{varprob00}
\end{eqnarray}
with $\overline{\mathbb{R}_{++}} \defeq \mathbb{R}\backslash
\mathbb{R}_{++}$. Furthermore,
letting $Z$ denoting the normalization constant of the $\chi$-escort of $Q$, the optimum $T^* :
\mathcal{X}\rightarrow \overline{\mathbb{R}_{++}}$ to eq. (\ref{varprob00}) is
\begin{eqnarray}
T^*(\ve{x}) & = & - \frac{1}{Z} \cdot \frac{\chi(Q(\ve{x}))}{\chi
   (P(\ve{x}))}\:\:.
\end{eqnarray}
\end{theorem}
(Proof in \SI, Section \ref{proof_thVIGGEN4}) Hence, the variational $f$-GAN formulation can be captured in an
information-geometric framework by the following
identity using Theorems \ref{thmSTANDCHI}, \ref{thVIGGEN1},
\ref{thVIGGEN3}, \ref{thVIGGEN4}.
\begin{corollary}(the \textbf{v}ariational
  \textbf{i}nformation-\textbf{g}eometric \textbf{$f$-GAN} identity) Using notations from
  Theorems \ref{thVIGGEN3}, \ref{thVIGGEN4}, we have
\begin{eqnarray}
\boxed{\sup_{T \in \overline{\mathbb{R}_{++}}^{\mathcal{X}}} \left\{\expect_{\X\sim \dist{P}} [T(\X)]  -
\expect_{\X\sim\escort{\dist{Q}}}
[(-\log_{\chi_{\escort{Q}}})^\star(T(\X))]\right\} =
D_C(\ve{\theta}\|\ve{\vartheta}) +  J(\dist{Q})}\:\:,\label{eqFUND1}
\end{eqnarray}
where $\ve{\theta}$ (resp. $\ve{\vartheta}$) is the coordinate of
$\dist{P}$ (resp. $\dist{Q}$).
\end{corollary}
We shall also name for short \textit{vig-$f$-GAN} the identity in
eq. (\ref{eqFUND1}). Even when it
is not needed to understand the high-level picture of the
identity, we can
reduce the
Legendre conjugate $(-\log_{\chi_{\escort{Q}}})^\star$ to an equivalent
"dual" (negative) $\chi^\bullet$-logarithm in the variational problem.
\begin{theorem}\label{thVIGGEN5}
The variational formulation of $KL_{\chi_{\escort{{Q}}}}(\escort{\dist{Q}}\|\dist{P})$
(Theorem \ref{thVIGGEN4}) satisfies:
\begin{eqnarray}
\lefteqn{\sup_{T \in \overline{\mathbb{R}_{++}}^{\mathcal{X}}} \left\{\expect_{\X\sim \dist{P}} [T(\X)]  -
\expect_{\X\sim\escort{\dist{Q}}}
[(-\log_{\chi_{\escort{Q}}})^\star(T(\X))]\right\}}\nonumber\\
 & = & \sup_{T \in \overline{\mathbb{R}_{++}}^{\mathcal{X}}} \left\{\expect_{\X\sim \dist{P}} [T(\X)]  -
\expect_{\X\sim\escort{\dist{Q}}}
\left[-
  \log_{(\chi^\bullet)_{\frac{1}{\escort{Q}}}}(-T(\X))\right]\right\} - K(\dist{Q})\:\:,\label{eqSIMPLCONJ}
\end{eqnarray}
where $K(.)$ is a function of $\dist{Q}$ only and 
\begin{eqnarray}
\chi^\bullet(t) & \defeq & \frac{1}{\chi^{-1}\left(\frac{1}{t}\right)}\:\:.\label{defCHIBULLET}
\end{eqnarray}
\end{theorem}
(Proof in \SI, Section \ref{proof_thVIGGEN5}) Since only the "$\sup$"
part is of interest in the supervised discriminator-generator game, the main interest of Theorem
\ref{thVIGGEN5} is to give a more precise shape to the losses involved
in the supervised game
(See Section \ref{sec-sup}).
\begin{remark}
The left hand-side of Eq. (\ref{eqFUND1}) has the exact same overall shape as
the variational objective of \cite[Eqs 2, 6]{nctFG}, in which we would
have 
equivalently $f = -\log_{\chi_{\escort{Q}}}, f^\star = -
  \log_{(\chi^\bullet)_{1/\escort{Q}}}$, eq. (\ref{defVARGAN}). However, it tells the
formal story of GANs in significantly greater
details, in particular for what concerns the generator. 
For example, eq. (\ref{eqFUND1}) yields a new
characterization of the generators' convergence: because $D_C$ is a
Bregman divergence, it satisfies the identity of the
indiscernibles. So, up to the proximity of $\dist{Q}$ to
its escort (to have $J(\dist{Q})$ small), solving the $f$-GAN game
\cite{nctFG} guarantees convergence in the parameter space
($\ve{\vartheta}$ vs $\ve{\theta}$). In the realm of GAN
applications, it makes sense to consider that $\dist{P}$ (the true distribution) can be extremely
complex. Therefore, even when deformed exponential families are significantly more expressive
than regular exponential families \cite{nGT}, extra care should be put
before arguing that complex applications comply with such a geometric
convergence in the parameter space. One way to circumvent this problem is to build
distributions in $\dist{Q}$ that factorize many deformed exponential families. This is one strong point
of deep architectures that we shall prove in Section \ref{sec-deep-all}.

We also remark two key component of the vig-$f$-GAN identify in deformed exponential families which are absent
from Theorem \ref{thmSTAND}:
\begin{itemize}
\item [(1)] the generator ($\dist{Q}$) appears in the form of an
  \textit{escort} in the variational component --- this distinction vanishes for exponential
  families, where $\escort{\dist{Q}} = \dist{Q}$;
\item [(2)] an information theoretic penalty appears in the identity ($J(\dist{Q})$) --- this
  penalty vanishes for exponential families, for which $J(\dist{Q}) = 0$.
\end{itemize}
These two components are crucial to link the $f$-GAN variational
optimization to the geometric convergence in the parameter space. We
shall drill down into both in Section \ref{sec-deep-all}.
\end{remark}


\section{A complete proper loss picture of the supervised GAN game}\label{sec-sup}

In their generalization of the GAN objective, Nowozin \textit{et al.} \cite{nctFG}
leave untold a key part of the supervised game: they split in eq. (\ref{defVARGAN}) the
discriminator's contribution in two, $T_{\ve{\omega}} = g_f
\circ V_{\ve{\omega}}$, where $V_{\ve{\omega}} :
  \mathcal{X}\rightarrow \mathbb{R}$ is the actual discriminator, and
  $g_f$ is essentially a technical constraint to ensure that $V_{\ve{\omega}}(.)$ is in the domain of $f^\star$. They leave the
  choice of $g_f$ "somewhat arbitrary" \citep[Section 2.4]{nctFG}. We now show that if one wants the
  supervised loss to have the desirable property to be \emph{proper
    composite} \cite{rwCB}\footnote{informally, Bayes rule realizes the optimum
  and the loss accommodates for any real valued predictor.}, then $g_f$
  is not arbitrary. We proceed in three steps, first unveiling a broad
  class of \emph{proper $f$-GANs} that deal with this property.\\

\noindent \textbf{Proper $f$-GANs} --- The initial motivation of
eq. (\ref{defVARGAN}) was
that the inner maximisation may be seen as the $f$-divergence between
$\dist{P}$ and $\dist{Q}_{\ve{\theta}}$ \cite{Nguyen:2010}, $L_f(\ve{\theta}) =
I_f(\dist{P}\| \dist{Q}_{\ve{\theta}})$. 
In fact, this variational representation of an $f$-divergence holds more generally:
by \cite[Theorem 9]{rwID}, we know that for any convex $f$,
and 
invertible \emph{link function} $\Psi \colon (0, 1) \to \Real$,
we have:
\begin{equation}
	\label{eqn:proper-variational}
	\inf_{T \colon \mathcal{X} \to \Real} \E{( \X, \Y ) \sim {\dist{D}}}{
          \ell_\Psi( \Y, T( \X ) ) } = - \frac{1}{2} \cdot I_f( \dist{P} \,\|\, \dist{Q} )
\end{equation}
where  ${\dist{D}}$ is the distribution over (observations $\times$ $\{$fake,
real$\}$) and the loss function $\ell_\Psi$ is defined by:
\begin{eqnarray}
	\label{eqn:proper-composite-f}
		\ell_\Psi( +1, z ) \defeq -f'\left(
                  \frac{\Psi^{-1}(z)}{1-\Psi^{-1}(z)}\right) & \:\:;
                \:\: &
		\ell_\Psi( -1, z ) \defeq f^\star\left( f'\left( \frac{\Psi^{-1}(z)}{1-\Psi^{-1}(z)} \right) \right)\:\:,
\end{eqnarray}
assuming
$f$ differentiable. Note now that picking $\Psi( z ) = f'(z/(1 - z))$ with $z \defeq
T(\ve{x})$ and simplifying eq. (\ref{eqn:proper-variational}) with
$\pr[\Y = \mbox{fake}] = \pr[\Y = \mbox{real}] = 1/2$ in the GAN game
yields eq. (\ref{defVARGAN}).
For other link functions, however, we get an equally valid class of losses whose optimisation will yield a meaningful estimate of the $f$-divergence.
The losses of eq. (\ref{eqn:proper-composite-f}) belong to the class of \emph{proper composite losses} with \emph{link function} $\Psi$ \cite{rwCB}.
Thus (omitting parameters $\ve{\theta}, \ve{\omega}$), we rephrase
eq. (\ref{defVARGAN}) and refer to the
\emph{proper $f$-GAN} formulation as $\inf_{\dist{Q}} L_\Psi(\dist{Q})$ with ($\ell$ is as per eq. (\ref{eqn:proper-composite-f})):
\begin{equation}
	\label{eqn:proper-gan}
	 L_\Psi(\dist{Q})  \defeq \sup_{T \colon \mathcal{X} \to \Real} \left\{\E{\X \sim \dist{P}}{ -\ell_\Psi( +1, T( \X ) ) } + \E{\X \sim \dist{Q}}{ -\ell_\Psi( -1, T( \X ) ) }\right\}\:\:.
\end{equation}
Note also that it is trivial to start from a suitable proper composite
loss, and derive the corresponding generator $f$ for the
$f$-divergence as per eq. (\ref{eqn:proper-variational}). Finally,
our proper composite loss view of the $f$-GAN game allows us to
elicitate $g_f$ in \cite{nctFG}: it is the
composition of $f'$ and $\Psi$ in eq. (\ref{eqn:proper-composite-f}).\\

\noindent\textbf{Proper $f$-GANs and density ratios} ---
The use of proper composite losses as part of the supervised GAN formulation sheds further light on another aspect the game:
the connection between the value of the optimal discriminator, and the
density ratio between the generator and discriminator
distributions. Instead of the optimal $T^*( \ve{x} ) = f'( P(\ve{x})/Q(\ve{x}) )$ for
eq. (\ref{defVARGAN}) \citep[Eq. 5]{nctFG}, we now have with the more general
eq. (\ref{eqn:proper-gan}) the result $T^*( \ve{x} ) = \Psi( (1 + Q(\ve{x})/P(\ve{x}))^{-1} )$.\\

\noindent\textbf{Proper vig-$f$-GANs} --- We now show
that proper $f$-GANs can easily be adapted to
eq. (\ref{eqFUND1}). 
\begin{theorem}\label{thmSUP}
For any $\chi$, define $\ell_{\ve{x}}( -1, z) \defeq
-
  \log_{(\chi^\bullet)_{\frac{1}{\tilde{Q}(\ve{x})}}}(-z)$,
and let $\ell( +1, z ) \defeq -z$. Then $L_\Psi(\dist{Q})$ in
eq. (\ref{eqn:proper-gan}) equals 
eq. (\ref{eqFUND1}). Its link in eq. (\ref{eqn:proper-gan}) is 
\begin{eqnarray}
\Psi_{\ve{x}}( z ) & = & -\frac{1}{\chi_{\tilde{Q}(\ve{x})}\left(
    \frac{z}{1 - z} \right) }\:\:.
\end{eqnarray}
\end{theorem}
(Proof in \SI, Section \ref{proof_thmSUP}) Hence, in the proper
composite view of the vig-$f$-GAN identity, the generator rules over
the supervised game: it
tempers with both the link function and the loss ---
but only for fake examples. Notice also that when $z=-1$, the fake
examples loss satisfies $\ell_{\ve{x}}( -1, -1) = 0$
regardless of $\ve{x}$ by definition of the $\chi$-logarithm. 

\section{Consequences for deep learning}\label{sec-deep-all}

\begin{figure}[t]
\begin{center}
\begin{tabular}{c}
\includegraphics[trim=0bp 470bp 370bp
80bp,clip,width=0.99\columnwidth]{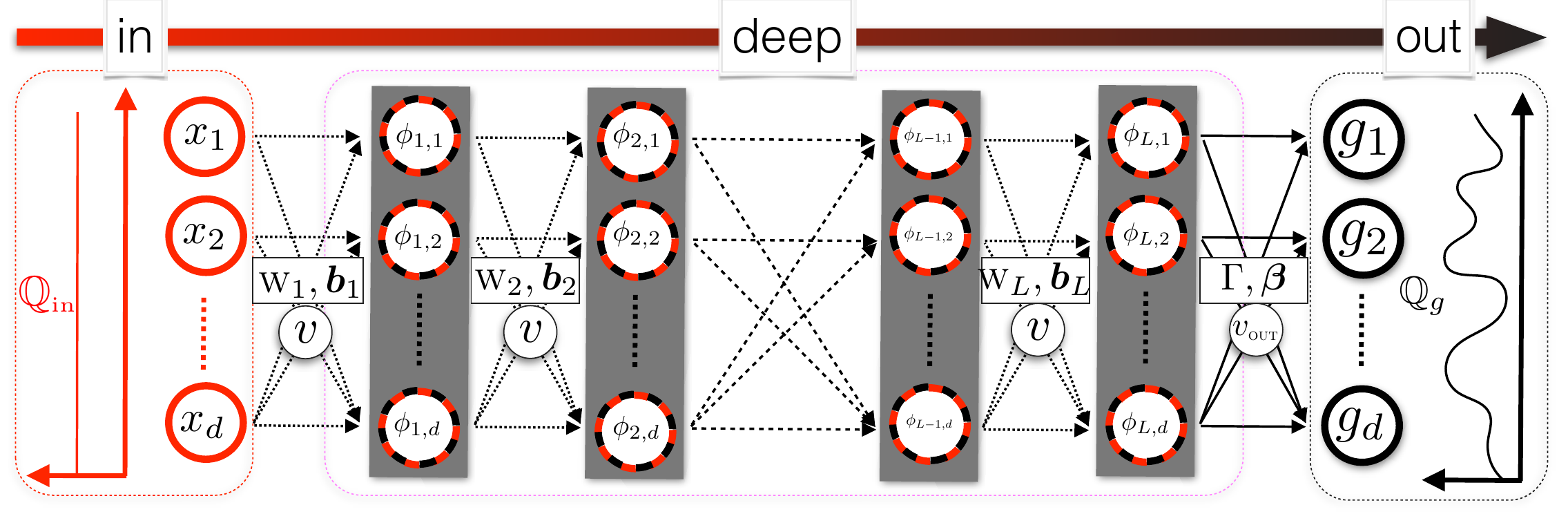}
\end{tabular} 
\end{center}
\caption{Deep architecture for the generator; it takes as input a simple
  distribution ($\dist{Q}_{\mbox{\tiny{in}}}$) and outputs a complex
  distribution ($\dist{Q}_{g}$) through a (deep) series of 
  non-linear transformations (best viewed in color, see text).}
\label{f-gene1}
\end{figure}

In this Section, we highlight a number of consequences of our results,
from the standpoint of deep learning. 
Eq. (\ref{eqFUND1}) shows the importance for the generator to be able
to model escorts --- and complex ones, in the realm of the GAN
applications. We start here with a proof that, when
used for the generator, mainstream
deep architectures \cite{lgmraOT} are amenable to such complex factorizations of escorts
using an especially compact design.

\subsection{Deep architectures and escorts in the vig-$f$-GAN game}\label{subDEEPFACT}

 In the GAN game, distribution $\dist{Q}$ in eq. (\ref{eqFUND1}) is built by the generator
(call it $\dist{Q}_{g}$),
by passing the
support of a
simple distribution (\textit{e.g.} uniform, standard Gaussian), $\dist{Q}_{\mbox{\tiny{in}}}$, through a series of
non-linear transformations (Figure \ref{f-gene1}). Letting
$\dens{Q}_{\mbox{\tiny{in}}}$ denote the corresponding density, we now
compute $\dens{Q}_{g}$. Our generator $\ve{g} :
\mathcal{X} \rightarrow \mathbb{R}^{d}$ consists of two
parts: a deep part and a last layer. The deep part is, given some $L\in \mathbb{N}$, the
computation of a non-linear transformation $\ve{\phi}_L : \mathcal{X} 
\rightarrow \mathbb{R}^{d_L}$ as
\begin{eqnarray}
\mathbb{R}^{d_\newell} \ni \ve{\phi}_{\newell}(\ve{x}) & \defeq & \ve{v}(\matrice{w}_{\newell}
\ve{\phi}_{\newell-1}(\ve{x}) + \ve{b}_{\newell})\:\:, \forall \newell \in \{1,
2, ..., L\}\:\:,\label{defg0}\\
\ve{\phi}_0 (\ve{x}) & \defeq & \ve{x} \in \mathcal{X}\:\:. \label{defg1}
\end{eqnarray}
$\ve{v}$ is a function computed coordinate-wise, such as (leaky)
ReLUs, ELUs \cite{cuhFA,hsmdsDS,mhnRN,nhRL},
$\matrice{w}_{\newell} \in \mathbb{R}^{d_{\newell}\times d_{\newell-1}},
\ve{b}_{\newell} \in \mathbb{R}^{d_\newell}$. The last layer computes the
generator's output from $\ve{\phi}_{L}$: 
\begin{eqnarray}
\ve{g}(\ve{x}) & \defeq & \vvout (\Gamma
\ve{\phi}_{L}(\ve{x}) + \ve{\beta})\:\:,
\end{eqnarray}
with $\Gamma \in \mathbb{R}^{d \times d_{L}},
\ve{\beta} \in \mathbb{R}^{d}$;
in general, $\vout \neq v$ and $\vout$ fits the output to the domain
at hand, ranging from linear \cite{acbWG,lgmraOT} to non-linear functions like
$\tanh$ \cite{nctFG}. Our generator, sketched in Figure \ref{f-gene1}
captures the high-level
features of some state of the art generative approaches
\cite{rmcUR,wtpUC,zmlEB}. 

To carry our analysis, we make the assumption that the network
is reversible, which is going to reguire that $\vout, \Gamma,
\matrice{w}_{\newell}$ ($l \in \{1, 2, ..., L\}$) are invertible. Since $\vout$ would be in many
experimental cases
(identity, $\tanh$, etc.), we essentially assume that dimensions match
like in Figure \ref{f-gene1} and so the simple input density is in fact
of dimension $d$ (\textit{e.g.} uniform over $\mathcal{X}=$ a
hypercube). At this reasonable price, we get in closed form the generator's
density and it shows the following: for any continuous signature
$\chinet$, there exists an activation function
$v$ such that the deep, most important part
 in the network (Figure \ref{f-gene1}) can factor \textit{exactly} as escorts for the
 $\chinet$-exponential family. Let $\ve{1}_i$ denote the $i^{th}$
 canonical basis vector. 
\begin{theorem}\label{factorDEEP}
$\forall \vout, \Gamma, \matrice{w}_\newell$ invertible ($\newell \in \{1, 2, ...,
 L\}$), for any continuous signature $\chinet$, there
 exists activation $v$ and $\ve{b}_\newell \in \mathbb{R}^d$ ($\forall \newell \in \{1, 2, ...,
 L\}$) such that for any output $\ve{z}$, letting $\ve{x} \defeq \ve{g}^{-1}(\ve{z})$, $Q_{g}(\ve{z})$ factorizes as:
\begin{eqnarray}
Q_{g}(\ve{z})  & = &
\frac{Q_{\mbox{\tiny{in}}}(\ve{x})}{\escort{Q}_{\mbox{\tiny{deep}}}(\ve{x})} \cdot
\frac{1}{H_{\mbox{\tiny{out}}}(\ve{x})\cdot  Z_\net}\:\:,\label{deflikeTH}
\end{eqnarray}
with $Z_\net > 0$ a constant, $H_{\mbox{\tiny{out}}}(\ve{x}) \defeq
  \prod_{i=1}^{d}
   |\vpout (\ve{\gamma}^\top_{i} \ve{\phi}_{L}(\ve{x}) + \beta_{i})|$,
   $\ve{\gamma}_{i} \defeq \Gamma^\top \ve{1}_i$, and (letting $\ve{w}_{\newell, i} \defeq \matrice{w}^\top_\newell \ve{1}_i$):
\begin{eqnarray}
\escort{Q}_{\mbox{\tiny{deep}}}(\ve{x}) & \defeq & \prod_{\newell=1}^{L}\prod_{i=1}^{d}
   \escort{P}_{\chinet, b_{\newell, i}}(\ve{x}|\ve{w}_{\newell, i}, \ve{\phi}_{\newell-1})\:\:.\label{defUUstate}
\end{eqnarray}
\end{theorem}
(Proof in \SI, Section \ref{proof_factorDEEP})
The relationship between the inner layers of a
deep net and deformed exponential
families (Definition \ref{defDEF}) follows from the Theorem: 
\begin{itemize}
\item rows in $\matrice{w}_\newell$s define coordinates;
\item $\ve{\phi}_\newell$ define "deep" sufficient statistics; 
\item $\ve{b}_\newell$ are
cumulants;
\item the crucial part, the $\chi$-family, is
given by the activation function $v$. 
\end{itemize}
Notice also that the
$\ve{b}_\newell$s are learned, and so the deformed exponential
families' normalization is in fact \textit{learned} and not specified. The proof of the Theorem comments on a simplification
of the constant when we also suppose that the escorts' normalization
is not specified. The proof of the Theorem also comments on two additional
keypoints:
\begin{itemize}
\item [(i)] how $Q_{g}(\ve{z})$ may factor as a likelihood on a graphical model
defined by the inner layers of $\ve{g}$;
\item [(ii)] how the "twist" introduced by
  $H_{\mbox{\tiny{out}}}(\ve{x})$ can be absorbed in a "$\mathrm{det}(.)$" volume
element with general sigmoid activations \cite{nctFG,wtpUC,rmcUR,zmlEB}, which is standard to the change of variable formula
\cite{dsbDE}. We also note that with linear activation \cite{acbWG,lgmraOT}, $H_{\mbox{\tiny{out}}}(\ve{x})$ is
constant.
\end{itemize}
We see that 
$\escort{Q}_{\mbox{\tiny{deep}}}$ factors escorts, and in number, which is good news with respect
to the power of deep
architectures and their adequation to the GAN framework. What is
remarkable is the compactness achieved by the deep representation: the
total dimension of all deep sufficient statistics in $\escort{Q}_{\mbox{\tiny{deep}}}$
(eq. (\ref{defUUstate})) is $L\cdot d$. To handle this, a shallow net with a single
inner layer would require a matrix $\matrice{w}$ of space $\Omega
(L^2 \cdot d^2)$. The
deep net $\ve{g}$ requires only $O(L \cdot d^2)$ space to store all
$\matrice{w}_\newell$s.

\subsection{Escort-compliant design of inner activations in the generator} 

The proof of Theorem 
\ref{factorDEEP} is constructive:
it builds $v$ as a function of $\chi$. In fact, the proof also
shows how to build $\chi$ \textit{from} the activation function
$v$ in such a way that $\escort{Q}_{\mbox{\tiny{deep}}}$ factors
$\chi$-escorts. The following Lemma essentially says that this is
possible for all \textit{strongly admissible} activations $v$.
\begin{definition}
Activation function $v$ is strongly
admissible iff $\mathrm{dom}(v) \cap
\overline{\mathbb{R}_+} \neq \emptyset$ and $v$
is $C^1$, lowerbounded,
strictly increasing and convex.
\end{definition}
\begin{lemma}\label{lemACT}
For any strongly admissible $v$, there exists signature
$\chi$ such that Theorem \ref{factorDEEP} holds.
\end{lemma}
(proof in \SI, Section \ref{proof_lemACT}) ($\gamma$,$\gamma$)-ELU
(for any $\gamma > 0$), Softplus are
strongly admissible, which leaves open the status of more general
ELUs, leaky ReLU and, or course, ReLU
\cite{cuhFA,dbbngIS,mhnRN,nhRL}. We note that these latter
activations satisfy parts of the constraints already, as they are
increasing, convex and meet the domain requirement. 
We shall analyze them through the property that they can be arbitrarily
closely approximated by a strongly admissible activation, a property
that we define as weak admissibility.
\begin{definition}\label{defWeak}
Activation $v$ is weakly admissible
iff for any $\epsilon > 0$, there exists $v_\epsilon$ strongly
admissible such that $||v -v_\epsilon||_{L_1} < \epsilon$, where
$||f||_{L_1}\defeq \int |f(t)| \mathrm{d}t$.
\end{definition}
Notice that the constraint is stronger than just controlling
$\sup_z |v(z) -v_\epsilon(z)|$. Nevertheless, we can prove the
following.
\begin{lemma}\label{lemACT_ReLU}
ReLU is weakly admissible.
\end{lemma}
\begin{figure}[t]
\begin{center}
\begin{tabular}{c}
\includegraphics[trim=0bp 30bp 0bp
0bp,clip,width=0.80\columnwidth]{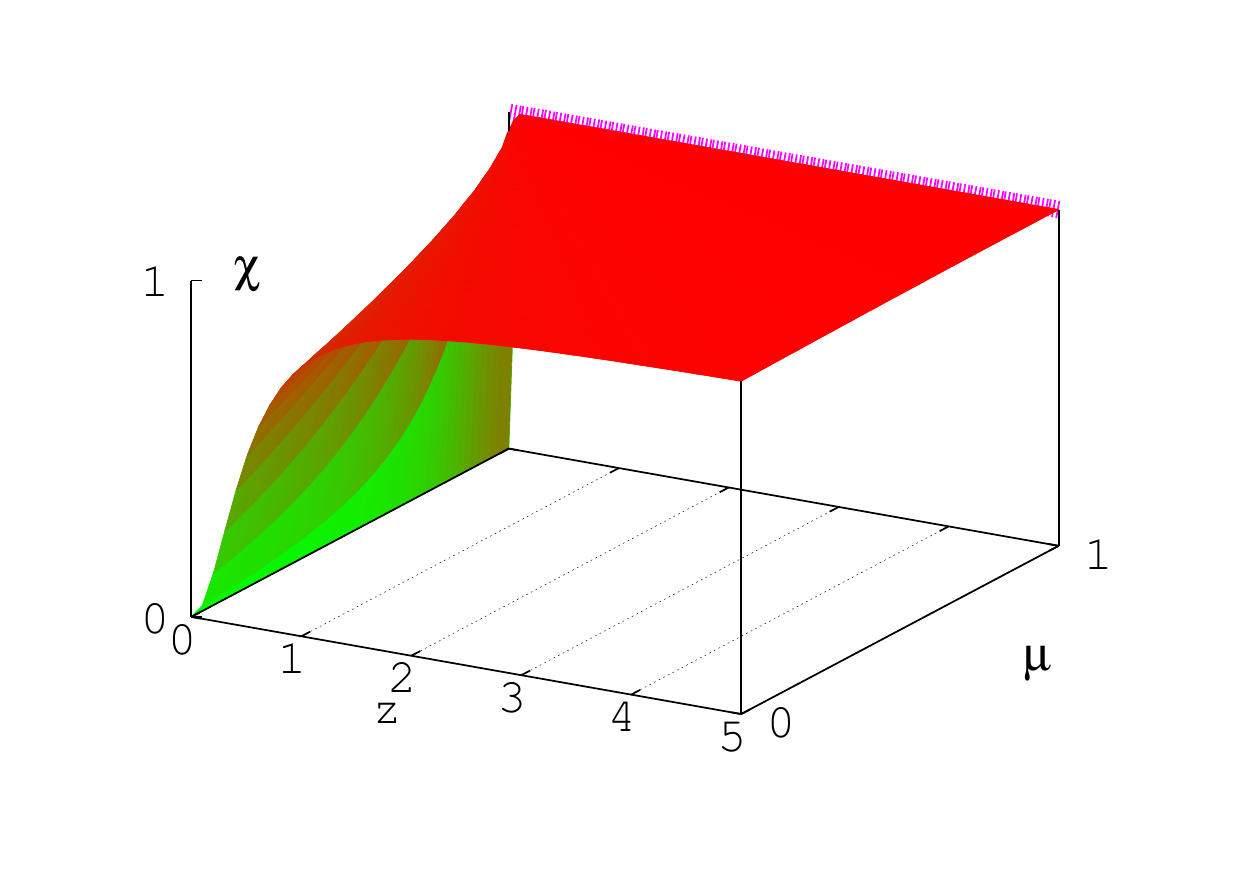} 
\end{tabular} 
\end{center}
\caption{Convergence of the signature $\chi$ for $\mu$-ReLU to that of
  ReLU (dashed pink at the back, also displayed in Figure \ref{f-allX}).}
\label{f-sequ}
\end{figure}
(proof in \SI, Section \ref{proof_lemACT_ReLU}) The trick is simple:
approximate the function by a strongly admissible smooth activation,
to get rid of the fact that ReLU is not differentiable everywhere and
not strictly increasing. For this reason, this trick can easily be repeated for $(\alpha, \beta)$-ELU. For
leaky-ReLU, we need to add the constraint that the domain is
lowerbounded, and then the trick is the same. Table \ref{t-neu-synt}
presents several couples $(v, \chi)$ for which $v$ is (strongly or
weakly) admissible. In the case where $v$ is strongly admissible, we
give the signature $\chi$ that would be obtained through Lemma
\ref{lemACT}. If it is weakly admissible, we give the limit $\chi$ for
the sequence of strong admissible activations in Definition
\ref{defWeak}. Figure \ref{f-sequ} gives an example of such a sequence
for the $\mu$-ReLU activation. Table \ref{t-neu-synt} includes a wide
class of so-called "prop-$\tau$
activations", where $\tau$ is negative a concave entropy, defined on
  $[0,1]$ and symmetric around $1/2$ \cite{nnBD}. Softplus
  \cite{dbbngIS} is a prop-$\tau$ activation. We also remark that ReLU
  $=\lim_{\mu\rightarrow 1} \mu$-ReLU (in the sense that
  $\lim_{\mu\rightarrow 1} \sup_z |\mbox{ReLU}(z) -
  \mbox{$\mu$-ReLU}(z)| = 0$). One
property of prop-$\tau$ activations is especially handy for
Wasserstein GANs \citep[Eq. 3]{acbWG}: prop-$\tau$ activations are Lipschitz (proof in \citep[Section
3]{nnOT}). Finally, the LSU activation should in theory be constrained
to domain $[-1,1]$, so we have linearly extended it to $\mathbb{R}$ by
linearity, keeping convexity and differentiability.
\begin{table}[t]
\begin{center}
\begin{tabular}{c|c|c}\hline\hline
Name & $v(z)$ & $\chi(z)$\\ \hline
ReLU$^{(\S)}$ & $\max\{0,z\}$
& $1_{z>0}$\\
Leaky-ReLU$^{(\dagger)}$  & $\left\{ \begin{array}{rcl} 
z  & \mbox{ if } & z> 0\\
\epsilon z   & \mbox{ if } & z\leq 0\\
\end{array}\right.$
& $\left\{ \begin{array}{rcl} 
1 & \mbox{ if } & z > -\delta\\
\frac{1}{\epsilon}  & \mbox{ if } & z\leq -\delta \\
\end{array}\right.$\\
$(\alpha, \beta)$-ELU$^{(\heartsuit)}$ & $\left\{ \begin{array}{ccl} 
\beta z  & \mbox{ if } & z> 0\\
\alpha(\exp(z)-1) & \mbox{ if } & z\leq 0\\
\end{array}\right.$
& $\left\{ \begin{array}{rcl} 
\beta & \mbox{ if } & z > \alpha \\
z  & \mbox{ if } & z\leq \alpha \\
\end{array}\right.$\\\hdashline
\rowcolor{Gray}
prop-$\tau$$^{(\clubsuit)}$ &
$k+\frac{\tau^\star(z)}{\tau^\star(0)}$ & $\frac{\tau'^{-1} \circ
 (\tau^\star)^{-1} (\tau^\star(0) z)}{\tau^\star(0)}$ \\\hdashline
\rowcolor{Gray}
Softplus$^{(\diamondsuit)}$ &
$k+\log_2(1+\exp(z))$ & $\frac{1}{\log 2}\cdot\left(1-2^{-z}\right)$
\\
\rowcolor{Gray}
$\mu$-ReLU$^{(\spadesuit)}$ & $k+\frac{z +
  \sqrt{(1-\mu)^2+z^2}}{2}$ & $\frac{4z^2}{(1-\mu)^2+4z^2}$\\
\rowcolor{Gray} LSU$^{(\P)}$ & $k+\left\{ \begin{array}{ccl} 
0  & \mbox{ if } & z<-1\\
(1+z)^2  & \mbox{ if } & z\in [-1,1]\\
4 z   & \mbox{ if } & z> 1\\
\end{array}\right.$ & $\left\{ \begin{array}{ccl} 
2\sqrt{z}  & \mbox{ if } & z<4\\
4   & \mbox{ if } & z> 4\\
\end{array}\right.$\\
\hline\hline\end{tabular}
\end{center}
\caption{Some (strongly or weakly) admissible couples $(v, \chi)$. $(\S)$ : $1.$ is the
  indicator function; ($\dagger$) :
  $\delta\leq 0$, $0<\epsilon\leq 1$ and 
  $\mathrm{dom}(v) = [\delta/\epsilon, +\infty)$. ($\heartsuit$) :
  $\beta \geq \alpha > 0$; 
  $(\clubsuit)$ : $\star$ is Legendre
  conjugate;
  ${(\spadesuit)}$ : $\mu \in [0,1)$. Shaded: prop-$\tau$ activations;
  $k$ is a constant (\textit{e.g.} such that $v(0)=0$);
  ${(\P)}$ : LSU = Least Square Unit (see text).
\label{t-neu-synt}}
\end{table}
Figure \ref{f-allX} plots several choices of signatures $\chi$, corresponding
to different choices of activation functions, distributions or
$f$-divergences (Figure \ref{f-summa} in \SI~provides the
correspondence from the choice of $\chi$).
\begin{figure}[t]
\begin{center}
\begin{tabular}{c}
\includegraphics[trim=30bp 350bp 450bp
20bp,clip,width=0.60\columnwidth]{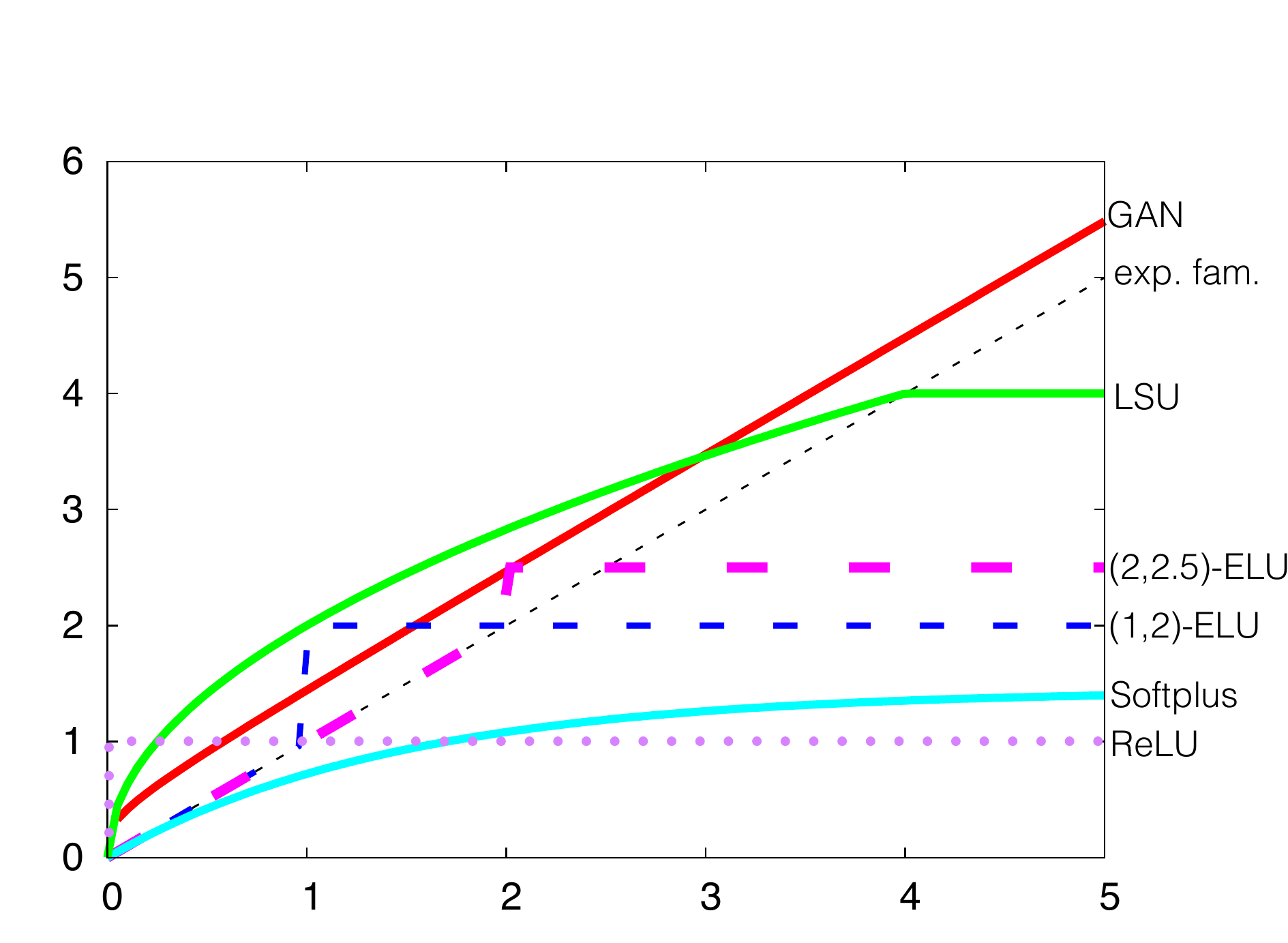} 
\end{tabular} 
\end{center}
\caption{Choices of $\chi$ corresponding to various activation
  functions (LSU, Softplus, ($\alpha$, $\beta$)-ELU, ReLU, see Table
  \ref{t-neu-synt}), distributions (exp. fam. = exponential families)
  or $f$-divergences (GAN, see eq. (\ref{defCHIGAN})).}
\label{f-allX}
\end{figure}

\subsection{$J(\dist{Q})$ vs not $J(\dist{Q})$}
By focusing on the left hand side of eq. (\ref{eqFUND1}), the usual $f$-GAN approaches
\cite{nctFG}
guarantee convergence in the parameter spaces which is all the
better as $J(\dist{Q})$ is small after convergence. This is happening when $\chi$ is
(close enough to) identity because in this case $\escort{\dist{Q}}\rightarrow \dist{Q}$, but this is not really interesting in the
context of deep learning where non-linear transformations imply $\chi$ is not going to comply (Theorem \ref{factorDEEP}). For several
interesting cases, we show an upperbound on
$J(\dist{Q})$ which is \textit{decreasing} with $Z$, the normalization parameter of the escort
(Definition \ref{defDEF}). Recall that $J(\dist{Q}) \defeq
KL_{\chi_{\escort{{Q}}}}(\escort{\dist{Q}}\|\dist{Q})$, so there needs to be two
components to specify $J$: $\chi$ and $\dist{Q}$. In theory, there is no need
for $\dist{Q}$ to belong to the $\chi$-family for $J(\dist{Q})$ to be measurable, so
our results will be general in the sense that we shall make no
assumption about $\dist{Q}$; $\chi$ will be fixed either directly (original
$f$-GAN choice) or as a function of the activation function (\textit{e.g.}
Table \ref{t-neu-synt}).

For any predicate $\pi : \mathcal{X} \rightarrow
\{\texttt{false}, \texttt{true}\}$, $\textsc{m}(\pi) \defeq
\int_{\ve{x}: \pi(\ve{x}) = \texttt{true}}
\mathrm{d}\mu(\ve{x})$ denotes the total measure of the support satisfying
$\pi$. 

\begin{theorem}\label{thmKLQQ}
The following bounds on $J(\dist{Q})$ and $Z$ hold, for any $\dist{Q}$:\\
\noindent (i) for the original GAN choice of $\chi$, we have $Z > 1$ and 
\begin{eqnarray}
  J(\dist{Q}) & \leq & \frac{1}{Z} \cdot \textsc{m}\left(Q(.) < \frac{1}{Z-1}\right)\:\:.
\end{eqnarray}
\noindent (ii) for $\mu$-ReLU activation, letting $L\defeq 1/(1-\mu)$, we have $Z\leq L$ and 
\begin{eqnarray}
J(\dist{Q}) & \leq & \frac{1}{Z}\cdot \left(1+\frac{L}{Z}\right)\:\:.
\end{eqnarray}
\noindent (iii) for the $(\gamma, \gamma)$-ELU activation with $\gamma
\geq 1$, we have
\begin{eqnarray}
J(\dist{Q}) & \leq & \frac{\log\gamma }{Z} + \frac{1-Z}{Z^2} + \frac{H_*(\dist{Q})}{Z}\:\:,
\end{eqnarray}
where $H_{*}(\dist{Q}) \defeq \expect_{\X \sim \dist{Q}}[\max\{0,-\log
Q(\X)\}]$. 
\end{theorem}
Proof in \SI, Section \ref{proof_thmKLQQ}. These results seems to
display the pattern that reducing $J(.)$ can be obtained via
maximizing $Z$, the normalization coefficient for the
escort. How $Z$ depends \textit{in fine} on $\chi, v$ is non
trivial. It seems that picking $\chi$ that augments the "contrast"
(blows up high density regions) is a good idea. Figure
\ref{fig:chi_effect} presents some examples of density shapes
(not normalized)
obtained from a simple density passed through various $\chi$,
showing how one can control such a contrast.  Figure \ref{f-gene2}
does the same for a standard Gaussian,
where the resulting densities (in color) are normalized.
\begin{figure}
\begin{tikzpicture}
    \begin{axis}[name=chi, xmin=0, xmax=1]
        \addplot [thick, black, domain=0:1] {x};
        \coordinate (f) at (axis cs:0.367879,0.367879);
    \end{axis}
    \coordinate (output) at ($(chi.south east) + (0.1,0)$);
    \coordinate (input) at ($(chi.south west) + (0,-1.7)$);
    
    \begin{axis}[at=(input), clip=false, rotate around={-90:(current axis)}, , xtick=\empty, ytick=\empty, xmin=-2.5, xmax=2.5 ]
        \addplot [thick, red, domain=-2.5:2.5] {exp(-x^2)};
        \coordinate (in) at (axis cs:-1,0.367879);
    \end{axis}
    
    \begin{axis}[ at=(output), clip=false, , xtick=\empty, ytick=\empty, xmin=-2.5, xmax=2.5]
        \addplot [thick, red, domain=-2.5:2.5] {exp(-x^2)};
        \coordinate (out) at (axis cs:-1,0.367879);
    \end{axis}
    \path[red] (in) node[draw, red, thick, circle, inner sep=1pt, fill=red] {} edge[densely dashed, black, -latex, shorten >=1pt] (f)  node[draw, red, thick, circle, inner sep=1pt, fill=red] {} (f) edge[densely dashed, black, -latex, shorten >=1pt] (out);
    \draw (out) node[draw, red, thick, circle, inner sep=1pt, fill=red] {};
    \draw (f) node[draw, black, thick, circle, inner sep=1pt, fill=black] {};
    
    \coordinate (chiloc) at ($(chi.south east) + (1.8,0)$);
    \coordinate (chiloc2) at ($(chi.south east) + (1.8,0)$);
    \begin{axis}[name=chi, at=(chiloc), xmin=0, xmax=1]
        \addplot [thick, black, domain=0:1] {x^4};
    \end{axis}
    \coordinate (output) at ($(chi.south east) + (0.1,0)$);
    \begin{axis}[ at=(output), clip=false, , xtick=\empty, ytick=\empty, xmin=-2.5, xmax=2.5 ]
        \addplot [thick, red, domain=-2.5:2.5] {exp(-x^2)^5} ;
    \end{axis}
    
    \coordinate (chiloc) at ($(chi.south east) + (1.8,0)$);
    \begin{axis}[name=chi, at=(chiloc), xmin=0, xmax=1]
        \addplot [thick, black, domain=0:1] {(tanh(4*x - 2) + 1)/2};
    \end{axis}
    \coordinate (output) at ($(chi.south east) + (0.1,0)$);
    \begin{axis}[ at=(output), clip=false, , xtick=\empty, ytick=\empty, xmin=-2.5, xmax=2.5]
        \addplot [thick, red, domain=-2.5:2.5] {(tanh(4*exp(-x^2)^2 - 2) + 1)/2} ;
    \end{axis}
    
    \coordinate (chiloc) at ($(chiloc2.south west) + (0,-1.8)$);
    \begin{axis}[name=chi, at=(chiloc), xmin=0, xmax=1,]
        \coordinate (b) at  (axis cs:0.666,0);
        \coordinate (o) at  (axis cs:0,0);
        \addplot [thick, black, domain=0.666:1] {3*x-2};
    \end{axis}
    \coordinate (output) at ($(chi.south east) + (0.1,0)$);
    \draw[thick, black] (o) -- (b);
    \begin{axis}[ at=(output), clip=true, , xtick=\empty, ytick=\empty, xmin=-2.5, xmax=2.5]
        \addplot [thick, red, domain=-0.636761:0.636761] {(3*exp(-x^2)-2)} ;
        \coordinate (l) at (axis cs:-0.636761,0);
        \coordinate (r) at (axis cs:0.636761,0);
        \coordinate (ll) at (axis cs:-2.5,0);
        \coordinate (rr) at (axis cs:2.5,0);
    \end{axis}
    \draw[thick, red] (ll) -- (l);
    \draw[thick, red] (r) -- (rr);

    \coordinate (chiloc) at ($(chi.south east) + (1.8,0)$);
    \begin{axis}[name=chi, at=(chiloc), xmin=0, xmax=1]
        \coordinate (o) at (axis cs:0,0);
        \coordinate (l) at (axis cs:0.5,0);
        \coordinate (r) at (axis cs:0.5,1);
        \coordinate (rr) at (axis cs:1,1);
    \end{axis}
    \draw[thick, black, shorten >=1pt] (o) -- (l);
    \draw (l) node[draw, black, thick, circle, inner sep=1pt] {};
    \draw (r) node[thick, black, inner sep=1pt, circle, fill=black] {};
    \draw[thick, black] (r) -- (rr);
    \coordinate (output) at ($(chi.south east) + (0.1,0)$);
    \begin{axis}[ at=(output), clip=false, , xtick=\empty, ytick=\empty, xmin=-2.5, xmax=2.5]
        \coordinate (oo) at (axis cs:-2.5,0);
        \coordinate (l1) at (axis cs:-0.832555,0);
        \coordinate (l2) at (axis cs:-0.832555,1);
        \coordinate (r1) at (axis cs:0.832555,1);
        \coordinate (r2) at (axis cs:0.832555,0);
        \coordinate (rr) at (axis cs:2.5,0);
    \end{axis}
    \draw[thick, red, shorten >=1pt] (oo) -- (l1) node[draw, red, thick, circle, inner sep=1pt] {};
    \draw (l2) node[draw, red, thick, circle, fill=red, inner sep=1pt] {} edge[red, thick] (r1);
    \draw (r1) node[draw, red, thick, circle, fill=red, inner sep=1pt] {};
    \draw[thick, red, shorten <=1pt] (r2)  node[draw, red, thick, circle, inner sep=1pt] {} -- (rr);
\end{tikzpicture}
    \centering
    \caption{Illustration of the effect of passing a density
      (upper-left) through some signature $\chi$ (black curves), \textit{without normalization}. } \label{fig:chi_effect}
\end{figure}
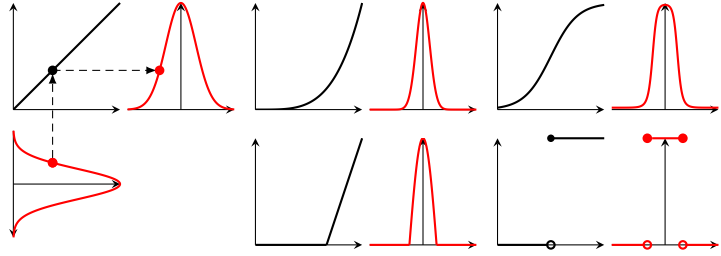
\begin{figure}[t]
\begin{center}
\begin{tabular}{c}
\includegraphics[trim=20bp 4bp 10bp
10bp,clip,width=0.50\columnwidth]{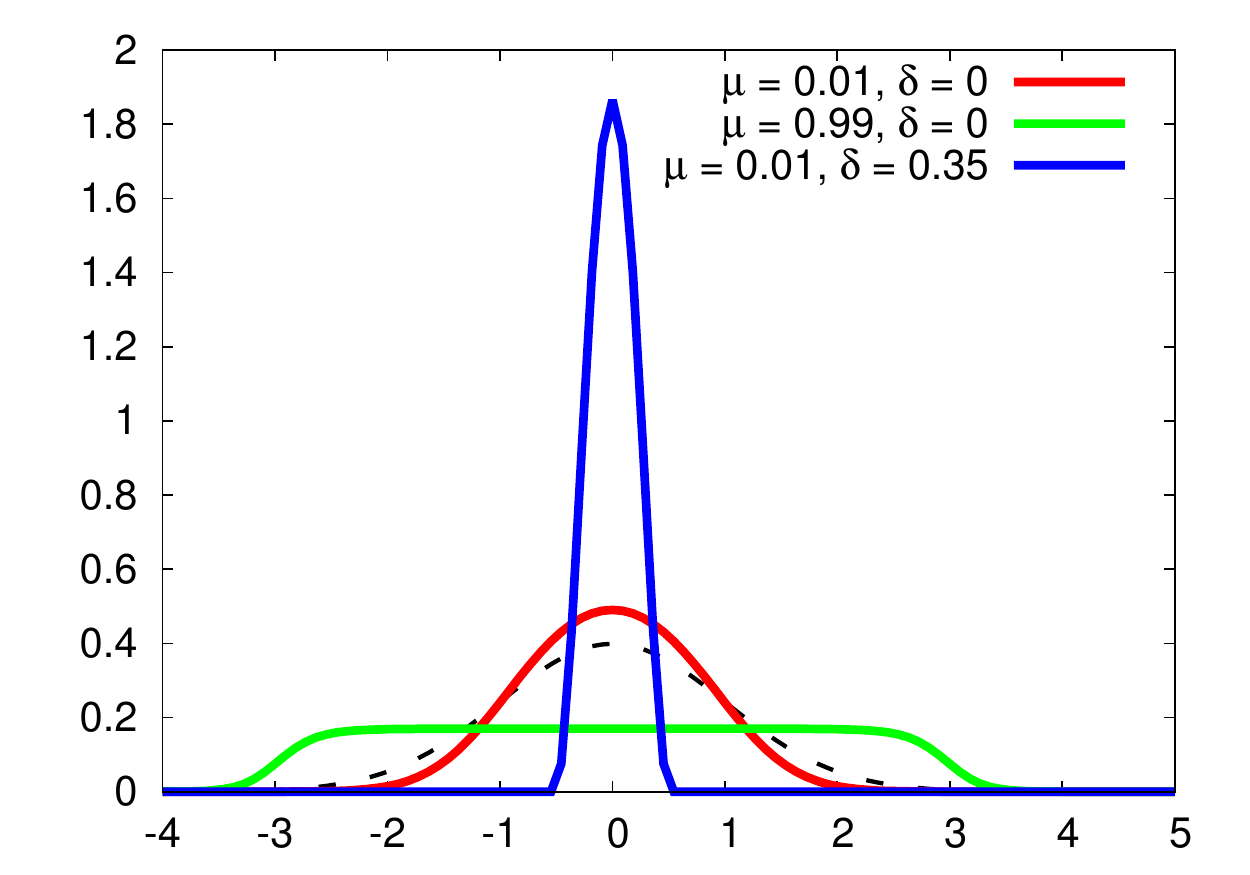}
\end{tabular} 
\end{center}
\caption{Escorts of a standard Gaussian (dashed), for a leaky-$\chi_{\delta,\epsilon}$ (see text).}
\label{f-gene2}
\end{figure}
Since Lemma
\ref{lemACT} is very general, we can engineer very specific
$\chi$s for this objective: inspired from the leaky-ReLU activation,
the example of
Figure \ref{f-gene2} uses such \textit{leaky}-$\chi$ escorts when
$\chi$ is that of the $\mu$-ReLU (Table \ref{t-neu-synt}, $\delta > 0$, small
$\epsilon>0$): 
\begin{eqnarray}
\chi_{\delta,\epsilon}(z) & \defeq & 1_{z<\delta}\cdot (\epsilon z) +
  1_{z\geq \delta} \cdot (\epsilon \delta + \chi(z-\delta))\:\:.
\end{eqnarray}

\subsection{How to play the proper-GAN game}

In \cite{pasaIG}, the density ratio connection was used to modify the GAN training procedure as follows:
first, one trains the discriminator to solve the inner maximisation in eq. (\ref{defVARGAN}) for convex $f$;
next, one estimates the density ratio $r( \ve{x} ) = P( \ve{x} )/Q(
\ve{x} )$ by 
\begin{eqnarray}
r( \ve{x} ) & = & (f')^{-1}( T^*( \ve{x} ))\:\:,
\end{eqnarray}
finally, one trains the generator to minimise the $f$-divergence $I_\varphi( P \| \dist{Q} ) = \expect_{\X' \sim \dist{Q}}{ \varphi( r( \X' ) ) }$ for convex $\varphi$.
In terms of proper composite losses, the first two steps can be generalised as follows:
first, one trains the discriminator to solve the inner maximisation in eq. (\ref{eqn:proper-gan}) for convex $f$ and link function $\Psi$;
next, one estimates the density ratio $r( \ve{x} ) = P( \ve{x} )/Q( \ve{x} )$ by $r( \ve{x} ) = \Psi^{-1}( T^*( \ve{x} ) )/(1 - \Psi^{-1}( T^*( \ve{x} ) ))$.
Note that this allows us \textit{e.g.}\ to use the logistic loss, for
which $\Psi( z ) = \log (z/(1 - z))$ and $f( z ) = z \cdot \log z
- (z+1) \cdot \log (z+1) + 2\log 2 = f_{\gan}(z)$ (eq. (\ref{defFGAN})).

\subsection{A more complete picture of geometric optimization in
  GANs} 
Any Bregman divergence is locally 
Mahalanobis', \textit{i.e.} a squared distance with a particular metric
\citep[Section 3]{anMO}. For eq. (\ref{eqFUND1}), is means when $C$ is
strictly convex that $\forall \ve{\theta}_P, \ve{\vartheta}_Q$, there exists
Symmetric Positive Definite (SPD) matrix $\matrice{m}$ such that
\begin{eqnarray}
D_C(\ve{\theta}_P\|\ve{\vartheta}_Q) =
D_{C^\star}(\ve{\mu}_Q\|\ve{\mu}_P) = \|\ve{\mu}_Q -
\ve{\mu}_P\|^2_{\matrice{m}}\:\:, \label{eqAPPROX}
\end{eqnarray} 
where $\ve{\mu}_. \defeq \nabla C
(\ve{\theta}_.) = \expect_.[\ve{\phi}]$ \cite[Section
4]{bmdgCWj}. Inner layers in the generator's deep net are sufficient
statistics ($\ve{\phi}$, Theorem \ref{factorDEEP} and Subsection \ref{subDEEPFACT}). We see that the parameterization chosen for the
geometric optimization of \citep[Section 3]{sgzcrcIT} looks like such
a divergence, with $\matrice{m} = \matrice{i}$. The only difference
with the vig-$f$-GAN identity is that the optimization occurs on the
statistics $\ve{\mu}_Q, \ve{\mu}_P$ of the discriminator and not the
generator, but it turns out that the $f$-divergences involved in the
supervised game (Section \ref{sec-sup} and \cite{rwID}) also admit a
formulation in terms of Bregman divergences \cite{nnBD} and therefore
can be approximated using eq. (\ref{eqAPPROX}). Hence, our results
support the feature matching technique of Salimans \textit{et al.} \citep[Section 3.1]{sgzcrcIT}.

\subsection{The generator can accomodate complex multimodal densities}

This is currently a hot topic in GAN architectures, with some concerns
raised about the capacity of the networks to capture multimodal
densities \cite{aglmzGA,azDG}. More specifically, whenever the
discriminator is too "small", then the generator may be trapped in
densities with very small \textit{support}, thereby preventing it
to capture the many modes of highly multi-modal densities. This is the so-called
"mode collapse" problem, and it is crucial since the modes of a density being its local maxima, they locally represent
the most natural objects to model. Because GAN
applications are complex, one works with the objective to capture numerous
modes \cite{cljblMR}. We consider the problem from the generator's
side and ask, at first hand, whether it is amenable to model such
complex densities --- if it were not, then GAN architectures would
be doomed beyond the training concerns raised by \cite{aglmzGA,azDG}.

Such a question can be answered in the
affirmative via Theorem \ref{factorDEEP} (See \SI, Section
\ref{sec_modes}), yet it requires specific signatures tailor made for
the generator's density to
capture all modes. It is therefore more a theoretical result than a
proof of validity for current architectures, yet using such signatures
can accomodate as many as $\Omega(d\cdot L)$ modes.

\subsection{Playing the (vig-)$f$-GAN game in the expected utility
  theory}

To play the GAN game at its fullest extent, we need to understand it
\textit{in extenso}.
Most of the game-theoretic focus on GANs has been focused on the
convergence and/or its Nash equilibrium \cite{aglmzGA,gGA}, around the
idea that the generator tries to "fool" the discriminator. The expected utility theory
allows to better qualify the quotes directly in the context of the
vig-$f$-GAN game. This requires some background which we now briefly
state \cite{cRA}. 

In an insurance market, a portfolio is a function $\Upsilon
:\mathcal{X}\to\mathbb{R}$ such that $\Upsilon(\ve{x})$ is the amount
of cash $\Upsilon$ pays to whomever holds it under the state of the
world $\ve{x} \in \mathcal{X}$ (negative payoffs are interpreted as
costs to the asset holder). Portfolio management for a Decision Maker
(DM) works in two steps: first, DM purchases the portfolio $\Upsilon$ with
market prices $P$, for a cost $\kappa \defeq \expect_{\X\sim
  P}[\Upsilon(\X)]$. Then DM receives a payoff $\Upsilon(\ve{x})$ upon
the revelation of the state of the world
$\ve{x}\in\mathcal{X}$. In the expected utility theory \cite{cRA}, assuming DM (i) has a quasilinear utility
function and (ii)
maximises expected utility according to subjective beliefs $\mathcal{Q}$. Then there
exists utility $\utility:\mathbb{R}\to\mathbb{R}$ increasing and concave such that DM
achieves maximal \textit{utility} $U(\mathcal{Q})$:
\begin{eqnarray}
U(\mathcal{Q}) & \defeq & \sup_{\Upsilon:
  \mathcal{X}\rightarrow \mathbb{R}} \left\{\expect_{\X\sim
  \mathcal{Q}}\bigl[\utility (\Upsilon(\X)) - \kappa\bigr]\right\} = \sup_{\Upsilon:
  \mathcal{X}\rightarrow \mathbb{R}} \{\expect_{\X\sim
  \dist{P}}[-\Upsilon(\X)] + \expect_{\X\sim
  \mathcal{Q}}[\utility (\Upsilon(\X))]\}\label{defEU}\:\:.
\end{eqnarray}
Suppose now that subjective beliefs $\mathcal{Q}$ are in the hand of
another player, G, distinct
from DM, and whose objective is to minimize $U(\mathcal{Q})$, the game
being the horizon of of min-max optimization iterations. The following
Lemma sheds light on the key parameters of the game.
\begin{lemma}\label{lemGAME}
The DM vs G game is equivalent to the (vig-)$f$-GAN game
(eq. (\ref{eqFUND1})) in which DM = discriminator, G = generator, the set of portfolios 
$\{\Upsilon\} = \{T\}$, the subjective beliefs $\mathcal{Q} =
\escort{\dist{Q}}$ and the utility 
\begin{eqnarray}
\utility (z) & = &
\log_{(\chi^\bullet)_{\frac{1}{\escort{Q}}}}(z)\:\:.
\end{eqnarray}
\end{lemma}
(Proof in \SI, Section \ref{proof_lemGAME}) Hence, G tampers with the utility function of DM in this game ---
which, we note, amounts for G to learn the true market prices
$P$. There is more to drill from the game in terms of risk aversion,
as shown below.
\begin{lemma}\label{lemRISK}
Let the Arrow-Pratt coefficient of absolute risk aversion \cite{pRA} $a_{u\ve{x}}(z) \defeq -\frac{u''(z)}{u'(z)}$,
and the Arrow-Pratt coefficient of \textit{relative}
risk aversion, $r_{u}(z) \defeq z \cdot a_{u}(z)$. Suppose $\chi$ differentiable. Then, in
the DM vs G game, (i) DM is always risk averse. Furthermore, (ii) $r_{u}(z)$ is also indexed by $\X \sim \mathcal{Q}$ and we have
\begin{eqnarray}
r_{u\ve{x}}(z) & = & g\left(
  \frac{z}{\mathcal{Q}(\ve{x})}\right)\:\:,\\
g(z) & \defeq & z \cdot \frac{(\chi^{-1})'(z)}{\chi^{-1}(z)}\:\:.
\end{eqnarray} 
Finally,
(iii) at the optimum
$\Upsilon^*$, we have 
\begin{eqnarray}
r_{u\ve{x}}(\Upsilon^*(\ve{x})) & = &
g\left(\frac{1}{\chi(P(\ve{x}))}\right)\:\:.
\end{eqnarray}
\end{lemma}
(Proof in \SI, Section \ref{proof_lemRISK}) Hence, DM is always risk
averse and his relative risk aversion depends on subjective
beliefs with the notable exception
of the optimum $T^*$ for which
it depends on market prices only. Everything is like if DM was getting
rid of G's influenced subjective beliefs to come up with the optimal solution.

\section{Experiments}\label{sec-expes}

\begin{figure}[t]
\begin{center}
\begin{tabular}{c}
\includegraphics[trim=30bp 330bp 450bp
20bp,clip,width=0.60\columnwidth]{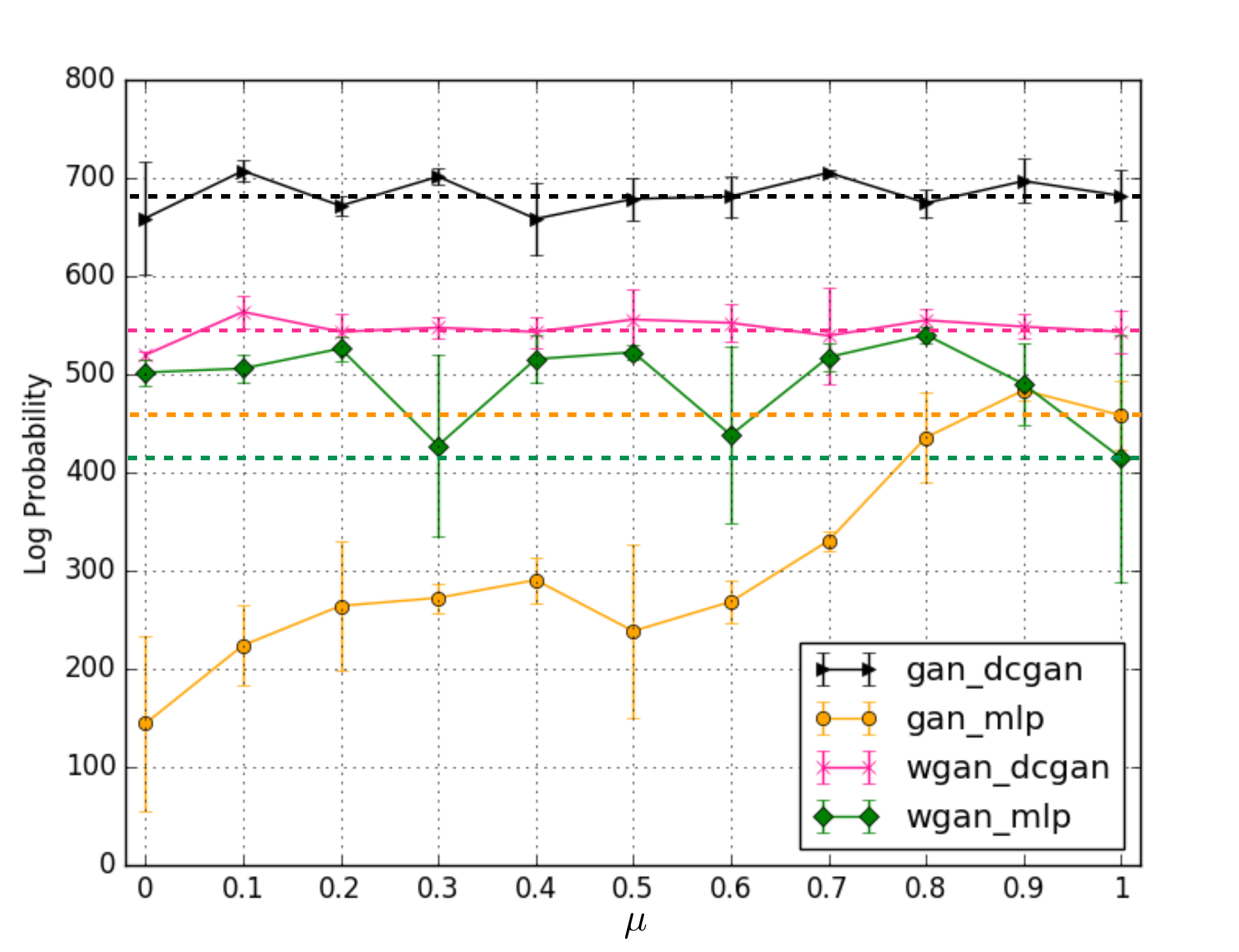} 
\end{tabular} 
\end{center}
\caption{Summary of our results on MNIST, on experiment A, comparing different values of $\mu$ for the $\mu$-ReLU activation in the generator (ReLU = 1-ReLU, see text). Thicker horizontal dashed lines present the ReLU average baseline: for each color, points above the baselines represent values of $\mu$ for which ReLU is beaten on average.}
\label{fig:expesA1}
\end{figure}

\begin{figure}[t]
\begin{center}
\begin{tabular}{c}
\includegraphics[trim=30bp 330bp 450bp
20bp,clip,width=0.60\columnwidth]{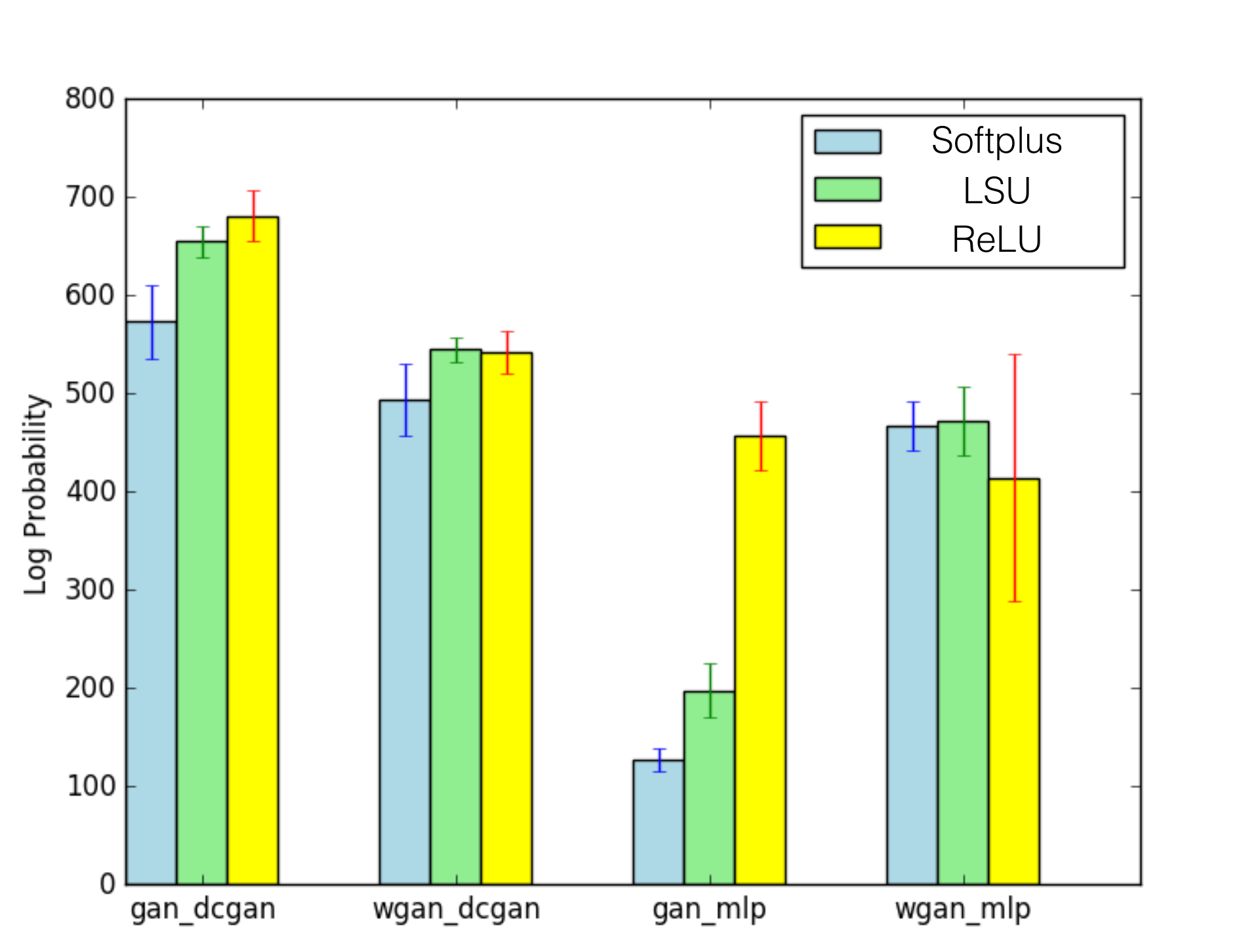} 
\end{tabular} 
\end{center}
\caption{Summary of our results on MNIST, on experiment A, comparing different activations in the generator, for the same architectures as in Figure \ref{fig:expesA1}.}
\label{fig:expesA2}
\end{figure}

Two of our theoretical contributions are:
\begin{itemize}
\item [(A)] the fact that on the \textit{generator}'s side, there exists numerous activation functions $v$ that comply with the design of its density as factoring escorts (Lemma \ref{lemACT}), and 
\item [(B)] the fact that on the \textit{discriminator}'s side, the so-called output activation function $g_f$ of \cite{nctFG} aggregates in fact two components of proper composite losses, one of which, the link function $\Psi$, should be a fine knob to operate (Theorem \ref{thmSUP}). 
\end{itemize}
We have tested these two possibilities with the idea that an experimental validation should provide substantial ground to be competitive with mainstream approaches, leaving space for a finer tuning in specific applications. Also, in order not to mix their effects, we have treated (A) and (B) separately. \\

\noindent \textbf{Architectures and datasets} --- We provide in \SI~(Section \ref{exp_expes}) the detail of all experiments. To summarize, we consider two architectures in our experiments: DCGAN~\cite{rmcUR} and the multilayer feedforward network (MLP) used in~\cite{nctFG}. Our datasets are MNIST~\cite{lecun1998gradient} and LSUN tower category~\cite{yu15lsun}.\\

\noindent \textbf{Comparison of varying activations in the generator (A)} --- We have compared $\mu$-ReLUs with varying $\mu$ in $[0, 0.1, ... , 1]$ (hence, we include ReLU as a baseline for $\mu=1$), the Softplus and the LSU activation (Figure \ref{t-neu-synt}). For each choice of the activation function, all inner layers of the generator use the same activation function. We evaluate the activation functions by using both DCGAN and the MLP used in~\cite{nctFG} as the architectures. As training divergence, we adopt both GAN \cite{gpmxwocbGA} and Wasserstein GAN (WGAN, \cite{acbWG}). Results are shown in Figure \ref{fig:expesA1}. Three behaviours emerge when varying $\mu$: either it is globally equivalent to ReLU (GAN DCGAN) but with local variations that can be better ($\mu = 0.7$) or worse ($\mu = 0$), or it is almost consistently better than ReLU (WGAN MLP) or worse (GAN MLP). The best results were obtained for GAN DCGAN, and we note that the ReLU baseline was essentially beaten for values of $\mu$ yielding smaller variance, and hence yielding smaller uncertainty in the results.

The comparison between different activation functions (Figure \ref{fig:expesA2}) reveals that ($\mu$-)ReLU performs overall the best, yet with some variations among architectures. We note in particular that, in the same way as for the comparisons intra $\mu$-ReLU (Figure \ref{fig:expesA1}), ReLU performs relatively worse than the other criteria for WGAN MLP, indicating that there may be different best fit activations for different architectures, which is good news. Visual results on LSUN (\SI, Table \ref{tab:expes_visu_lsun}) also display the quality of results when changing the $\mu$-ReLU activation.\\

\begin{figure}
\begin{center}
\begin{tabular}{c}
\hspace{-0.3cm} \includegraphics[trim=5bp 5bp 5bp
5bp,clip,width=0.50\columnwidth]{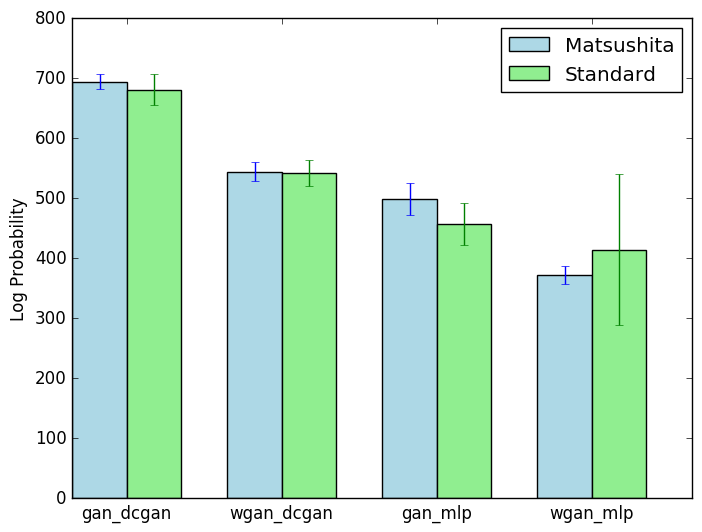} 
\end{tabular}
\end{center}
\caption{Summary of our results on MNIST, on experiment B, varying the link function in the discriminator (see text).}\label{fig:expesB}
\end{figure}

\noindent \textbf{Comparison of varying link functions in the discriminator (B)} --- We have compared the replacement of the sigmoid function by a link which corresponds to the entropy which is theoretically optimal in boosting algorithms, Matsushita entropy \cite{kmOTj,nnOT}, for which $\Psi_{\textsc{mat}}(z) \defeq (1/2)\cdot (1+z/\sqrt{1+z^2})$ and the entropy (Table \ref{t-neu-synt}) is $-\tau_{\textsc{mat}}(z) = 2\sqrt{z(1-z)}$. Figure \ref{fig:expesB} displays the comparison Matsushita vs "standard" (more specifically, we use sigmoid in the case of GAN \cite{nctFG}, and none in the case of WGAN to follow current implementations \cite{acbWG}). We evaluate with both DCGAN and MLP on MNIST (same hyperparameters as for generators, ReLU activation for all hidden layer activation of generators). Experiments tend to display that tuning the link may indeed bring additional uplift: for GANs, Matsushita is indeed better than the sigmoid link for both DCGAN and MLP, while it remains very competitive with the no-link (or equivalently an identity link) of WGAN, at least for DCGAN.

\section{Conclusion}

It is hard to exaggerate the success of GAN approaches in modelling
complex domains, and with their success comes an increasing need for a
rigorous theoretical understanding \cite{sgzcrcIT}. In this paper, we
complete the supervised understanding of the generalization of GANs
introduced in \cite{nctFG}, and provide a theoretical background to
understand its unsupervised part. We  show in particular how deep
architectures can be powerful at tackling the generative
part of the game, and can factor densities known to be far more
general than exponential families, both in terms of the available
densities (\textit{e.g.} Cauchy, Student) or physical phenomena that
can be modeled \cite{aomGO,nGE,nGT}.
Our contribution therefore improves
the understanding of both players in the GAN game. Experiments display that the tools we develop may help to
improve further the state of the art. Among the most prominent avenues
for future work relies the integration of penalty $J(\dist{Q})$ directly in
the GAN game. It turns out that a recent paper has precisely displayed
that the introduction of a mutual information regularizer in the GAN
game improves results and helps in disentangling representations 
\cite{cdhssaII}.

\section{Acknowledgments}

The authors wish to thank Shun-ichi Amari, Giorgio Patrini and Frank
Nielsen for numerous comments.

\bibliographystyle{plain}
\bibliography{bibgen,nips17-gan-references-aditya}

\begin{thebibliography}{10}

\bibitem{asAG}
S.-M. Ali and S.-D.-S. Silvey.
\newblock A general class of coefficients of divergence of one distribution
  from another.
\newblock {\em Journal of the Royal Statistical Society B}, 28:131--142, 1966.

\bibitem{aDM}
S.-I. Amari.
\newblock {\em Differential-Geometrical Methods in Statistics}.
\newblock Springer-Verlag, Berlin, 1985.

\bibitem{aIG}
S.-I. Amari.
\newblock {\em Information Geometry and Its Applications}.
\newblock Springer-Verlag, Berlin, 2016.

\bibitem{aPCOMM}
S.-I. Amari.
\newblock Personnal communication, 2017.

\bibitem{anMO}
S.-I. Amari and H.~Nagaoka.
\newblock {\em Methods of Information Geometry}.
\newblock Oxford University Press, 2000.

\bibitem{aomGO}
S.-I. Amari, A.~Ohara, and H.~Matsuzoe.
\newblock Geometry of deformed exponential families: Invariant, dually-flat and
  conformal geometries.
\newblock {\em Physica A: Statistical Mechanics and its Applications},
  391:4308--4319, 2012.

\bibitem{acbWG}
M.~Arjovsky, S.~Chintala, and L.~Bottou.
\newblock Wasserstein {GAN}.
\newblock {\em CoRR}, abs/1701.07875, 2017.

\bibitem{aglmzGA}
S.~Arora, R.~Ge, Y.~Liang, T.~Ma, and Y.~Zhang.
\newblock Generalization and equilibrium in generative adversarial nets
  ({GANs}).
\newblock {\em CoRR}, abs/1703.00573, 2017.

\bibitem{azDG}
S.~Arora and Y.~Zhang.
\newblock Do {GANs} actually learn the distribution? an empirical study.
\newblock {\em CoRR}, abs/1706.08224, 2017.

\bibitem{awRLBEF}
K.~S. Azoury and M.~K. Warmuth.
\newblock Relative loss bounds for on-line density estimation with the
  exponential family of distributions.
\newblock {\em MLJ}, 43(3):211--246, 2001.

\bibitem{bgwOT}
A.~Banerjee, X.~Guo, and H.~Wang.
\newblock On the optimality of conditional expectation as a bregman predictor.
\newblock {\em IEEE Trans. IT}, 51:2664--2669, 2005.

\bibitem{bmdgCWj}
A.~Banerjee, S.~Merugu, I.~Dhillon, and J.~Ghosh.
\newblock Clustering with {B}regman divergences.
\newblock {\em JMLR}, 6:1705--1749, 2005.

\bibitem{bbtCE}
A.~Ben-Tal, A.~Ben-Israel, and M.~Teboulle.
\newblock Certainty equivalents and information measures: Duality and extremal
  principles.
\newblock {\em J. of Math. Anal. Appl.}, pages 211--236, 1991.

\bibitem{bnnBV}
J.-D. Boissonnat, F.~Nielsen, and R.~Nock.
\newblock Bregman voronoi diagrams.
\newblock {\em DCG}, 44(2):281--307, 2010.

\bibitem{bvCO}
S.~Boyd and L.~Vandenberghe.
\newblock {\em Convex optimization}.
\newblock Cambridge University Press, 2004.

\bibitem{cRA}
J.-P. Chavas.
\newblock {\em Risk analysis in theory and practice}.
\newblock Academic press advanced finance, 2004.

\bibitem{cljblMR}
T.~Che, Y.~Li, A.-P. Jacob, Y.~Bengio, and W.~Li.
\newblock Mode regularized generative adversarial networks.
\newblock In {\em 5$^{th}$ ICLR}, 2017.

\bibitem{cdhssaII}
X.~Chen, Y.~Duan, R.~Houthooft, J.~Schulman, I.~Sutskever, and P.~Abbeel.
\newblock {InfoGAN}: Interpretable representation learning by information
  maximizing generative adversarial nets.
\newblock In {\em NIPS*29}, pages 2172--2180, 2016.

\bibitem{cuhFA}
D.-A. Clevert, T.~Unterthiner, and S.~Hochreiter.
\newblock Fast and accurate deep network learning by exponential linear units
  ({ELUs}).
\newblock In {\em 4$^{th}$ ICLR}, 2016.

\bibitem{cIT}
I.~Csisz{\'a}r.
\newblock Information-type measures of difference of probability distributions
  and indirect observation.
\newblock {\em Studia Scientiarum Mathematicarum Hungarica}, 2:299--318, 1967.

\bibitem{dsbDE}
L.~Dinh, J.~Sohl-Dickstein, and S.~Bengio.
\newblock Density estimation using real {NVP}.
\newblock In {\em 5$^{th}$ ICLR}, 2017.

\bibitem{dbbngIS}
C.~Dugas, Y.~Bengio, F.~B{\'{e}}lisle, C.~Nadeau, and R.~Garcia.
\newblock Incorporating second-order functional knowledge for better option
  pricing.
\newblock In {\em Advances in Neural Information Processing Systems*13}, pages
  472--478, 2000.

\bibitem{frCF}
R.-M. Frongillo and M.-D. Reid.
\newblock Convex foundations for generalized maxent models.
\newblock In {\em 33$^{rd}$ MaxEnt}, pages 11--16, 2014.

\bibitem{gpcSA}
A.~Genevay, G.~Peyr{\'e}, and M.~Cuturi.
\newblock Sinkhorn-autodiff: Tractable {Wasserstein} learning of generative
  models.
\newblock {\em CoRR}, abs/1706.00292, 2017.

\bibitem{gbbDS}
X.~Glorot, A.~Bordes, and Y.~Bengio.
\newblock Deep sparse rectifier neural networks.
\newblock In {\em 14$^{th}$ AISTATS}, pages 315--323, 2011.

\bibitem{gGA}
I.~Goodfellow.
\newblock Generative adversarial networks, 2016.
\newblock NIPS'16 tutorials.

\bibitem{gpmxwocbGA}
I.~Goodfellow, J.~Pouget-Abadie, M.~Mirza, B.~Xu, D.~Warde-Farley, S.~Ozair,
  A.~Courville, and Y.~Bengio.
\newblock Generative adversarial nets.
\newblock In {\em NIPS*27}, pages 2672--2680, 2014.

\bibitem{gaadcIT}
I.~Gulrajani, F.~Ahmed, M.~Arjovsky, V.~Dumoulin, and A.-C. Courville.
\newblock Improved training of wasserstein {GANs}.
\newblock {\em CoRR}, abs/1704.00028, 2017.

\bibitem{hsmdsDS}
R.-H.-R. Hahnloser, R.~Sarpeshkar, M.-A. Mahowald, R.-J. Douglas, and H.-S.
  Seung.
\newblock Digital selection and analogue amplification coexist in a
  cortex-inspired silicon circuit.
\newblock {\em Nature}, 405:947--951, 2000.

\bibitem{jcnvwID}
J.~Jiao, T.~Courtade, A.~No, K.~Venkat, and T.~Weissman.
\newblock Information divergences and the curious case of the binary alphabet.
\newblock In {\em ISIT'14}, pages 351--355, 2014.

\bibitem{kmOTj}
M.J. Kearns and Y.~Mansour.
\newblock On the boosting ability of top-down decision tree learning
  algorithms.
\newblock {\em J. Comp. Syst. Sc.}, 58:109--128, 1999.

\bibitem{corr/KingmaB14}
D.-P. Kingma and J.~Ba.
\newblock Adam: {A} method for stochastic optimization.
\newblock {\em CoRR}, abs/1412.6980, 2014.

\bibitem{lecun1998gradient}
Y.~LeCun, L.~Bottou, Y.~Bengio, and P.~Haffner.
\newblock Gradient-based learning applied to document recognition.
\newblock {\em Proceedings of the IEEE}, 86(11):2278--2324, 1998.

\bibitem{lgmraOT}
H.~Lee, R.~Ge, T.~Ma, A.~Risteski, and S.~Arora.
\newblock On the ability of neural nets to express distributions.
\newblock {\em CoRR}, abs/1702.07028, 2017.

\bibitem{lszGM}
Y.~Li, K.~Swersky, and R.-S. Zemel.
\newblock Generative moment matching networks.
\newblock In {\em 32$^{nd}$ ICML}, pages 1718--1727, 2015.

\bibitem{lbcAA}
S.~Liu, O.~Bousquet, and K.~Chaudhuri.
\newblock Approximation and convergence properties of generative adversarial
  learning.
\newblock {\em CoRR}, abs/1705.08991, 2017.

\bibitem{mhnRN}
A.-L. Maas, A.-Y. Hannun, and A.-Y. Ng.
\newblock Rectifier nonlinearities improve neural network acoustic models.
\newblock In {\em 30$^{th}$ ICML}, 2013.

\bibitem{mwDA}
H.~Matsuzoe and T.~Wada.
\newblock Deformed algebras and generalizations of independence on deformed
  exponential families.
\newblock {\em Entropy}, 17:5729--5751, 2015.

\bibitem{nhRL}
V.~Nair and G.~Hinton.
\newblock Rectified linear units improve restricted {B}oltzmann machines.
\newblock In {\em 27$^{th}$ ICML}, pages 807--814, 2010.

\bibitem{nGE}
J.~Naudts.
\newblock Generalized exponential families and associated entropy functions.
\newblock {\em Entropy}, 10:131--149, 2008.

\bibitem{nGT}
J.~Naudts.
\newblock {\em Generalized thermostatistics}.
\newblock Springer, 2011.

\bibitem{Nguyen:2010}
X.~Nguyen, M.~J. Wainwright, and M.~I. Jordan.
\newblock Estimating divergence functionals and the likelihood ratio by convex
  risk minimization.
\newblock {\em IEEE Transactions on Information Theory}, 56(11):5847--5861, Nov
  2010.

\bibitem{npCF}
C.~Niculescu and L.-E. Persson.
\newblock {\em Convex Functions and their Applications, A Contemporary
  Approach}.
\newblock Springer, 2006.

\bibitem{nnOT}
R.~Nock and F.~Nielsen.
\newblock On the efficient minimization of classification-calibrated
  surrogates.
\newblock In {\em NIPS*21}, pages 1201--1208, 2008.

\bibitem{nnBD}
R.~Nock and F.~Nielsen.
\newblock Bregman divergences and surrogates for learning.
\newblock {\em IEEE Trans.PAMI}, 31:2048--2059, 2009.

\bibitem{nnaOCD}
R.~Nock, F.~Nielsen, and S.-I. Amari.
\newblock On conformal divergences and their population minimizers.
\newblock {\em IEEE Trans. IT}, 62:1--12, 2016.

\bibitem{nctFG}
S.~Nowozin, B.~Cseke, and R.~Tomioka.
\newblock {$f$-GAN}: training generative neural samplers using variational
  divergence minimization.
\newblock In {\em NIPS*29}, pages 271--279, 2016.

\bibitem{pvAD}
M.-C. Pardo and I.~Vajda.
\newblock About distances of discrete distributions satisfying the data
  processing {T}heorem of {I}nformation {T}heory.
\newblock {\em IEEE Trans. IT}, 43:1288--1293, 1997.

\bibitem{pmbOT}
R.~Pascanu, T.~Mikolov, and Y.~Bengio.
\newblock On the difficulty of training recurrent neural networks.
\newblock In {\em 30$^{th}$ ICML}, pages 1310--1318, 2013.

\bibitem{pasaIG}
B.~Poole, A.-A. Alemi, J.~Sohl{-}Dickstein, and A.~Angelova.
\newblock Improved generator objectives for gans.
\newblock {\em CoRR}, abs/1612.02780, 2016.

\bibitem{pRA}
J.W. Pratt.
\newblock Risk aversion in the small and in the large.
\newblock {\em Econometrica}, 32:122--136, 1964.

\bibitem{rmcUR}
A.~Radford, L.~Metz, and S.~Chintala.
\newblock unsupervised representation learning with deep convolutional
  generative adversarial networks.
\newblock In {\em 4$^{th}$ ICLR}, 2016.

\bibitem{rwCB}
M.-D. Reid and R.-C. Williamson.
\newblock Composite binary losses.
\newblock {\em JMLR}, 11, 2010.

\bibitem{rwID}
M.-D. Reid and R.-C. Williamson.
\newblock Information, divergence and risk for binary experiments.
\newblock {\em JMLR}, 12:731--817, 2011.

\bibitem{sgzcrcIT}
T.~Salimans, I.-J. Goodfellow, W.~Zaremba, V.~Cheung, A.~Radford, and X.~Chen.
\newblock Improved techniques for training gans.
\newblock In {\em NIPS*29}, pages 2226--2234, 2016.

\bibitem{tdAB}
M.~Telgarsky and S.~Dasgupta.
\newblock Agglomerative {Bregman} clustering.
\newblock In {\em 29$^{~th}$ ICML}, 2012.

\bibitem{tieleman2012lecture}
T.~Tieleman and G.~Hinton.
\newblock Lecture 6.5-rmsprop: Divide the gradient by a running average of its
  recent magnitude.
\newblock {\em COURSERA: Neural networks for machine learning}, 4(2), 2012.

\bibitem{vehRD}
T.~van Erven and P.~Harremo{\"e}s.
\newblock R{\'e}nyi divergence and {K}ullback-{L}eibler divergence.
\newblock {\em IEEE Trans. IT}, 60:3797--3820, 2014.

\bibitem{vcOP}
R.-F. Vigelis and C.-C. Cavalcante.
\newblock On $\varphi$-families of probability distributions.
\newblock {\em J. Theor. Probab.}, 21:1--15, 2011.

\bibitem{wtpUC}
L.~Wolf, Y.~Taigman, and A.~Polyak.
\newblock Unsupervised creation of parameterized avatars.
\newblock {\em CoRR}, abs/1704.05693, 2017.

\bibitem{yu15lsun}
F.~Yu, Y.~Zhang, S.~Song, A.~Seff, and J.~Xiao.
\newblock Lsun: Construction of a large-scale image dataset using deep learning
  with humans in the loop.
\newblock {\em arXiv preprint arXiv:1506.03365}, 2015.

\bibitem{zmlEB}
J.~Zhao, M.~Mathieu, and Y.~{LeCun}.
\newblock Energy-based generative adversarial networks.
\newblock In {\em 5$^{th}$ ICLR}, 2017.

\end{thebibliography}

\section*{\supplement: table of contents}

\noindent Summary of the paper's notations\hrulefill Pg
  \pageref{sec-sum}\\

\noindent \textbf{Appendix on proofs and formal results}
\hrulefill Pg \pageref{the_theo}\\ 
\noindent Generalization of Theorem \ref{thmSTAND}\hrulefill Pg
  \pageref{SECGEN}\\
\noindent Proof of Theorem \ref{thVIGGEN1}\hrulefill Pg
  \pageref{proof_thVIGGEN1}\\
\noindent Proof of Theorem \ref{thVIGGEN2}\hrulefill Pg
  \pageref{proof_thVIGGEN2}\\
\noindent Proof of Theorem \ref{thVIGGEN3}\hrulefill Pg
  \pageref{proof_thVIGGEN3}\\
\noindent Proof of Theorem \ref{thVIGGEN4}\hrulefill Pg
  \pageref{proof_thVIGGEN4}\\
\noindent Proof of Theorem \ref{thVIGGEN5}\hrulefill Pg
  \pageref{proof_thVIGGEN5}\\
\noindent Proof of Theorem \ref{thmSUP}\hrulefill Pg
  \pageref{proof_thmSUP}\\
\noindent Proof of Theorem \ref{factorDEEP}\hrulefill Pg
  \pageref{proof_factorDEEP}\\
\noindent Proof of Lemma \ref{lemACT}\hrulefill Pg
  \pageref{proof_lemACT}\\
\noindent Proof of Theorem \ref{thmKLQQ}\hrulefill Pg
  \pageref{proof_thmKLQQ}\\
\noindent Many modes for GAN architectures\hrulefill Pg
  \pageref{sec_modes}\\
\noindent Proof of Lemma \ref{lemGAME}\hrulefill Pg
  \pageref{proof_lemGAME}\\
\noindent Proof of Lemma \ref{lemRISK}\hrulefill Pg
  \pageref{proof_lemRISK}\\

\noindent \textbf{Appendix on experiments} \hrulefill
Pg \pageref{exp_expes}\\
\noindent Architectures\hrulefill Pg \pageref{exp_archis}\\
\noindent Experimental setup for varying the activation function in the
generator\hrulefill Pg
  \pageref{exp_gen}\\
\noindent Visual results\hrulefill Pg
  \pageref{exp_mnist}\\
$\hookrightarrow$ MNIST results for GAN$\_$DCGAN at varying $\mu$ ($\mu=1$ is ReLU)\hrulefill Pg
  \pageref{tab:expes_visu_mnist_gan_dcgan}\\
$\hookrightarrow$ MNIST results for WGAN$\_$DCGAN at varying $\mu$ ($\mu=1$ is ReLU)\hrulefill Pg
  \pageref{tab:expes_visu_mnist_wgan_dcgan}\\
$\hookrightarrow$ MNIST results for WGAN$\_$MLP at varying $\mu$ ($\mu=1$ is ReLU)\hrulefill Pg
  \pageref{tab:expes_visu_mnist_wgan_mlp}\\
$\hookrightarrow$ MNIST results for GAN$\_$MLP at varying $\mu$ ($\mu=1$ is ReLU)\hrulefill Pg
  \pageref{tab:expes_visu_mnist_gan_mlp}\\
$\hookrightarrow$ LSUN results for GAN$\_$DCGAN at varying $\mu$ ($\mu=1$ is ReLU)\hrulefill Pg
  \pageref{tab:expes_visu_lsun}

\newpage

\section*{--- Summary of the paper's notations}\label{sec-sum}

\begin{figure}[t]
\begin{center}
\begin{tabular}{c}
\includegraphics[trim=30bp 270bp 420bp
20bp,clip,width=0.80\columnwidth]{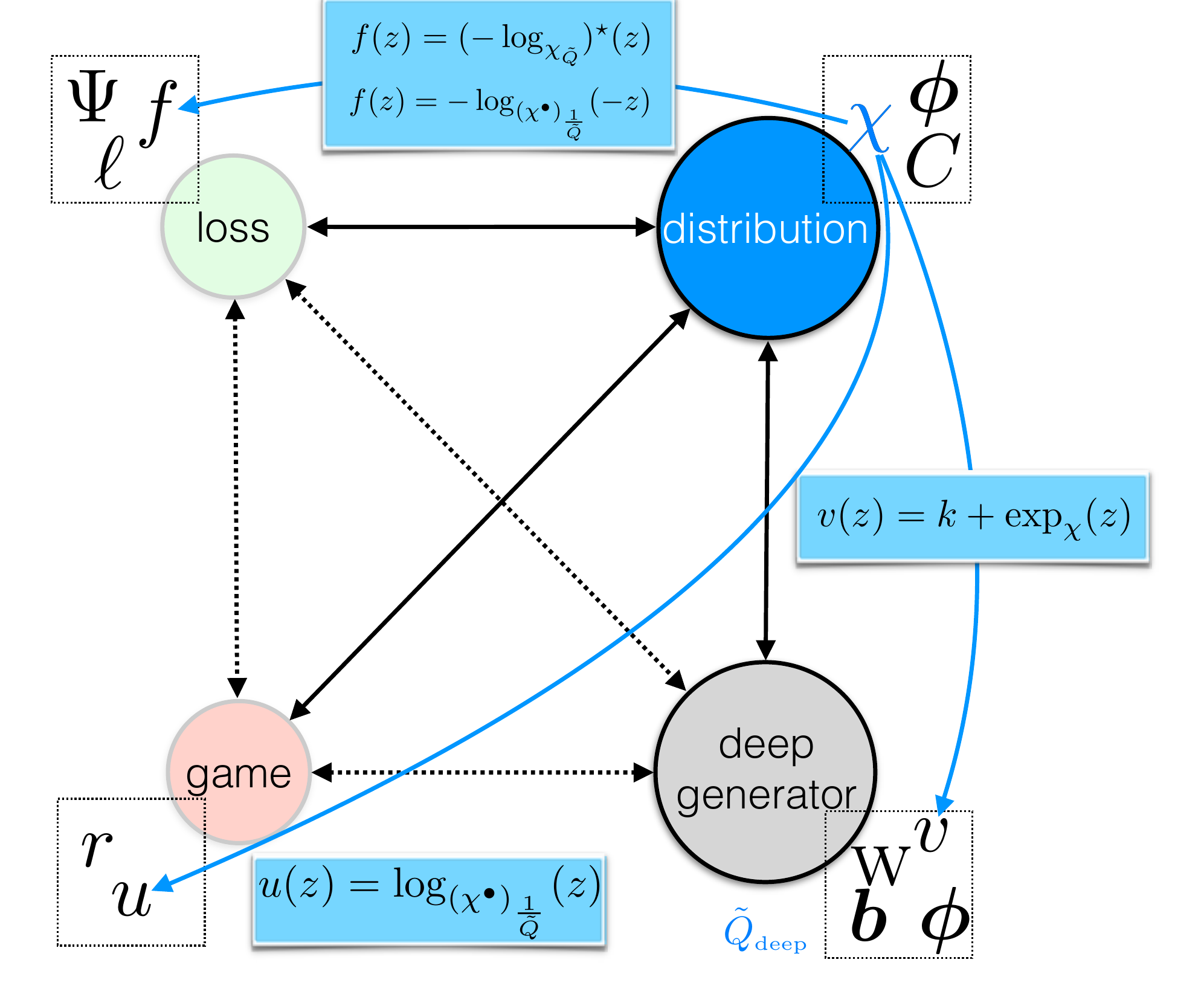} 
\end{tabular} 
\end{center}
\caption{Summary of the main parameters notations with respect to the GAN
  game, according to the four main components of the game (loss,
  distribution, game, model = deep generator). Plain (black / blue) arcs
  denote formal relationships between parameters that we show. 
The distribution learned by a deep generator
  decomposes in three parts, one which depends on the simple input
  distribution, one which depends on the very last layer and one which
  incorporates all the deep architecture components,
  $\tilde{Q}_{\mbox{\tiny{deep}}}$ (Theorem
  \ref{factorDEEP}). $\tilde{Q}_{\mbox{\tiny{deep}}}$ (shown) precisely
  factors escorts of deformed exponential families.}
\label{f-summa}
\end{figure}

Figure \ref{f-summa} summarizes the main notations with respect to our contributions on
the four components of a GAN "quadrangle": loss, distribution, game and
architecture = model ( = deep generator). Blue arcs identify some
  key parameters as a function of the signature of the deformed
  exponential family, $\chi$, to match several quantities of interest: 
\begin{itemize}
\item the arc $\chi \rightarrow
  f$ identifies the $f$ from $\chi$ which allows to prove the identity
  between vig-$f$-GAN and the variational $f$-GAN identity in
  \citep[Eq. 4]{nctFG} (Theorems \ref{thVIGGEN4}, \ref{thVIGGEN5});
\item the arc $\chi \rightarrow v$ identifies the activation function
  $v$ from $\chi$
  for which the inner deep part of the generator in Theorem
  \ref{factorDEEP} factors with $\chi$-escorts ($\tilde{Q}_{\mbox{\tiny{deep}}}$);
\item the arc $\chi \rightarrow u$ identifies the utility function of
  the discriminator / decision maker such that the decision maker's
  utility $U(\mathcal{Q})$ maximization (eq. (\ref{defEU})) matches
  vig-$f$-GAN (Lemma \ref{lemGAME}).
\end{itemize}
Name are as follows:
\begin{itemize}
\item [\textsc{loss}] $f$ = generator of the $f$-divergence; $\ell$ =
  loss function(s) for the supervised game; $\Psi$ = link function for
  the supervised loss;
 \item [\textsc{distribution}] $\chi$ = signature of the deformed
   exponential family; $\ve{\phi}$ = sufficient statistics; $C$ =
   cumulant;
\item [\textsc{game}] $u$ = utility function; $r$ = Arrow-Pratt
  coefficient of relative risk
  aversion;
\item [\textsc{model}] $\matrice{w}$ = inner layer matrices; $\ve{b}$
  = inner layers bias vectors; $v$ = inner layers activation function;
  $\ve{\phi}$ = inner layers vectors / "deep" sufficient statistics;
\end{itemize}

\newpage

\section*{--- Appendix on proofs and formal results}\label{the_theo}

\section{Generalization of Theorem \ref{thmSTAND}}\label{SECGEN}

In this Section, we adopt notations of \cite{frCF}. When dealing with
exponential families, it will be convenient to rewrite $\ve{\phi}(\ve{x})^\top
\ve{\theta}$ as the output of a function $\ve{\phi}^\top \ve{\theta} :
\mathcal{X} \rightarrow \mathbb{R}$ with $\ve{\phi}^\top \ve{\theta}
(\ve{x}) \defeq \langle \ve{\phi}(\ve{x}), \ve{\theta}
\rangle$ --- remark that $\ve{\theta}$ is implicitly fixed. Hence, the definition of the density of a (regular) exponential
family with cumulant $C : \Theta \rightarrow \mathbb{R}$ and
sufficient statistics $\ve{\phi}: \mathcal{X} \rightarrow
\mathbb{R}^d$ now becomes equivalently:
\begin{eqnarray}
\dens{P}_C(\ve{x} | \ve{\theta}, \ve{\phi}) & \defeq & \exp((\ve{\phi}^\top
\ve{\theta}) (\ve{x}) - C(\ve{\theta}))\:\:.\label{eqEXPFAM2B}
\end{eqnarray}
If we \textit{fix} $\ve{\theta}$, then
the sufficient statistics uniquely determines the cumulant (and
therefore the exponential family) and
\textit{vice-versa}. Let us fix such a vector $\ve{\theta}$ and adopt the concise formulation of Generalized
exponential families of \cite{frCF}, which we now introduce. Let $\bigtriangleup$ denote a set of
probability measures over $\mathcal{X}$ \cite{frCF}, and $\star$ denotes the Legendre transform \cite{bvCO}.
\begin{definition}\cite{frCF}\label{defGEFFRCF}
Let $F : \bigtriangleup \rightarrow \mathbb{R}$ be convex, lower
semi-continuous and proper. The $F$-Generalized exponential family
(GEF) of
distributions is the set $\{\dens{P}_{F}(\ve{x}|\ve{\theta}, \ve{\phi})
\in \partial F^\star(\ve{\phi}^\top
\ve{\theta}) : \ve{\theta} \in \Theta\}$, where $\ve{\phi}: \mathcal{X} \rightarrow
\mathbb{R}^d$ is called the statistic.
\end{definition}
Notice that $\ve{\phi}$ does not necessarily bear the properties of
sufficient statistics, and we can also define a cumulant, $C(\ve{\theta}) \defeq F^\star(\ve{\phi}^\top
\ve{\theta})$ \cite{frCF}\footnote{Notice the slight abuse of
  notation: this definition makes in fact the cumulant to be a function
$C : \mathbb{R}^{\mathcal{X}} \rightarrow \mathbb{R}$, but it does not
affect our results.}, and we have $\Theta = \mathrm{dom}(C)$. Deformed and generalized exponential
families emerged from two different grounds, thermostatistics and
information geometry for the former, convex optimization for the
latter. So, they are known for very different properties, yet regular
exponential families belong to both sets ($F$ is negative Shannon
entropy for regular exponential families in generalized exponential
families). For the sake of readability we now assume that the
cumulant is differentiable, so that the density $\dens{P}_{F}(\ve{x}|\ve{\theta}, \ve{\phi})
= \nabla F^\star(\ve{\phi}^\top
\ve{\theta})$ in Definition \ref{defGEFFRCF}. For any pairs of cumulants
statistics $\ve{\phi}_a, \ve{\phi}_b$, we define the Bregman divergence with
generator $F^\star$,
\begin{eqnarray}
D_{\ve{\theta}}(C_a\| C_b) & \defeq & F^\star(\ve{\phi}_a^\top
\ve{\theta}) - F^\star(\ve{\phi}_b^\top
\ve{\theta}) - \langle \ve{\phi}_a^\top
\ve{\theta}-\ve{\phi}_b^\top
\ve{\theta},\nabla F^\star(\ve{\phi}_b^\top
\ve{\theta})) \rangle\:\:.\label{defDIVCUM}
\end{eqnarray}
A key point of the bilinear form $\langle.,.\rangle$ is that it has
the fundamental property to transfer inner products from/to supports
to/from distribution parameters \citep[Section 2]{frCF}:
\begin{eqnarray}
\langle \ve{\phi}^\top
\ve{\theta}, \dens{P} \rangle & = & \langle \expect_P [\ve{\phi}], \ve{\theta}\rangle\:\:,\label{propINNER}
\end{eqnarray}
and in fact the inner product appearing in eq. (\ref{defDIVCUM}) is
also an inner product on parameters in disguise, a fact that will be
key to our result. We note that $D_{\ve{\theta}}(C_a\| C_b)$ is indeed a
Bregman divergence \citep[Theorem 3]{frCF}, which we can unambiguously
formulate over sufficient statistics or generators. Being a Bregman
divergence, it satisfies the identity of the indiscernibles: $C_a
= C_b$ iff $D_{\ve{\theta}}(C_a\| C_b) = 0$. Notice also that the
definition makes implicitly that the dimension of the sufficient
statistics is the same for both families defined by cumulants $C_a, C_b$.

With this notion of divergence between cumulants, we can now formulate
and prove our generalization of Theorem \ref{thmSTAND}: if we alleviate the membership constraint, then the KL
divergence is equal to the \textit{sum} of two divergences, one between
parameters (indexed by cumulants), and one between cumulants (indexed
by parameters).
\begin{theorem}\label{lemKLGEN}
Consider any two GEF distributions $\dist{P}$ and $\dist{Q}$ having
respective natural parameters $\ve{\theta}_p$ and $\ve{\theta}_q$,
cumulants $C_p$ and $C_q$ and densities $P$ and $Q$ absolutely
continuous with respect to base measure $\mu$. Then
\begin{eqnarray}
KL(\dist{P}\|\dist{Q}) & = & D_{C_p}
(\ve{\theta}_q\|\ve{\theta}_p) + D_{\ve{\theta}_q}(C_q\| C_p) \:\:.
\end{eqnarray}
\end{theorem}
\begin{proof}
We have:
\begin{eqnarray}
\lefteqn{KL(\dist{P}\|\dist{Q})}\nonumber\\
 & = & \int_{\ve{x}}
P(\ve{x})\log \frac{\dens{P}(\ve{x})}{\dens{Q}(\ve{x})}\mathrm{d} \mu(\ve{x})\nonumber\\
 &= & \int_{\ve{x}} P(\ve{x})\cdot \left( C_q(\ve{\theta}_q) - C_p(\ve{\theta}_p)  +
   \ve{\theta}_p ^\top \ve{\phi}_p(\ve{x}) - \ve{\theta}_q^\top
   \ve{\phi}_q(\ve{x}) \right)\mathrm{d} \mu(\ve{x})\nonumber\\
 & = & C_q(\ve{\theta}_q) - C_p(\ve{\theta}_p) - (\ve{\theta}_q^\top
 \expect_P[\ve{\phi}_q(\ve{x})]-\ve{\theta}_p^\top
 \nabla C_p(\ve{\theta}_p)) \nonumber\\
 & = & C_q(\ve{\theta}_q) - C_p(\ve{\theta}_p) - (\ve{\theta}_q -\ve{\theta}_p)^\top
 \nabla C_p(\ve{\theta}_p) - (\expect_P[\ve{\phi}_q(\ve{x})]-\expect_P[\ve{\phi}_p(\ve{x})])^\top \ve{\theta}_q \label{eqDEF001}\\
 & = & C_p(\ve{\theta}_q) - C_p(\ve{\theta}_p) - (\ve{\theta}_q -\ve{\theta}_p)^\top
 \nabla C_p(\ve{\theta}_p) \nonumber\\
 & & + \underbrace{C_q(\ve{\theta}_q) - C_p(\ve{\theta}_q)
 - (\expect_P[\ve{\phi}_q(\ve{x})]-\expect_P[\ve{\phi}_p(\ve{x})])^\top
 \ve{\theta}_q}_{\defeq A} \nonumber\\
 & = & D_{C_p}
(\ve{\theta}_q\|\ve{\theta}_p) + A \label{eqDEF000}\:\:.
\end{eqnarray}
In eq. (\ref{eqDEF001}), we use the fact that $\nabla
C_p(\ve{\theta}_p) = \expect_P[\ve{\phi}_p(\ve{x})]$. Now, we remark that $C_a(\ve{\theta}_q) = F^\star(\ve{\phi}_a^\top
\ve{\theta}_q)$ \citep[Definition 2, Lemma 2]{frCF}, and
\begin{eqnarray}
(\expect_P[\ve{\phi}_p(\ve{x})]-\expect_P[\ve{\phi}_q(\ve{x})])^\top
 \ve{\theta}_q & = & \langle P_{\ve{\theta}_p},
 \ve{\phi}_p^\top\ve{\theta}_q \rangle-\langle P_{\ve{\theta}_p},
 \ve{\phi}_q^\top\ve{\theta}_q \rangle\nonumber\\
 & = & \langle (\ve{\phi}_p - \ve{\phi}_q)^\top\ve{\theta}_q , P_{\ve{\theta}_p}
 \rangle\nonumber\\
 & = & \langle  (\ve{\phi}_p - \ve{\phi}_q)^\top\ve{\theta}_q, \nabla F^\star(\ve{\phi}_p^\top
\ve{\theta}))\rangle\label{leq001}\:\:,
\end{eqnarray}
using definitions of $\langle ., .\rangle$ and $F^\star$ in
\cite{frCF} (see also eq. (\ref{propINNER})). There remains to identify $A$ in eq. (\ref{eqDEF000}) and
$D_{\ve{\theta}_q}(C_q\| C_p)$ from eq. (\ref{defDIVCUM}). This ends
the proof of Theorem \ref{lemKLGEN}.
\end{proof}

\section{Proof of Lemma \ref{lemINV1}}\label{proof_lemINV1}
We have by definition of $KL_\chi$ divergences and
properties of the integration,
\begin{eqnarray}
KL_{\frac{\chi}{1+k\chi}}(\dist{P}\|\dist{Q}) & \defeq & \expect_{\X\sim
  \dist{P}}\left[-\log_{\frac{\chi}{1+k\chi}}\left(\frac{\dens{Q}(\X)}{\dens{P}(\X)}\right)\right]\nonumber\\
 & = & \expect_{\X\sim
  \dist{P}}\left[-\int_1^{\frac{\dens{Q}(\X)}{\dens{P}(\X)}} \frac{1+k\chi(t)}{\chi(t)}
  \mathrm{d}t \right]\nonumber\\
 & = & \expect_{\X\sim
  \dist{P}}\left[-\int_1^{\frac{\dens{Q}(\X)}{\dens{P}(\X)}} \left(\frac{1}{\chi(t)} + k\right)
  \mathrm{d}t \right]\nonumber\\
 & = & \expect_{\X\sim
  \dist{P}}\left[-\int_1^{\frac{\dens{Q}(\X)}{\dens{P}(\X)}} \frac{1}{\chi(t)} 
  \mathrm{d}t -\int_1^{\frac{\dens{Q}(\X)}{\dens{P}(\X)}} k
  \mathrm{d}t \right]\nonumber\\
 & = & \expect_{\X\sim
  \dist{P}}\left[-\log_{\chi}\left(\frac{\dens{Q}(\X)}{\dens{P}(\X)}\right) -\int_1^{\frac{\dens{Q}(\X)}{\dens{P}(\X)}} k
  \mathrm{d}t \right]\nonumber\\
 & = & \expect_{\X\sim
  \dist{P}}\left[-\log_{\chi}\left(\frac{\dens{Q}(\X)}{\dens{P}(\X)}\right) - k\cdot \left[z\right]_1^{\frac{\dens{Q}(\X)}{\dens{P}(\X)}} \right]\nonumber\\
 & = & \expect_{\X\sim
  \dist{P}}\left[-\log_{\chi}\left(\frac{\dens{Q}(\X)}{\dens{P}(\X)}\right) \right] - k\cdot \expect_{\X\sim
  \dist{P}}\left[\frac{\dens{Q}(\X)}{\dens{P}(\X)} - 1\right]\nonumber\\
 & = & \expect_{\X\sim
  \dist{P}}\left[-\log_{\chi}\left(\frac{\dens{Q}(\X)}{\dens{P}(\X)}\right) \right] - k\cdot \left(\expect_{\X\sim
  Q}\left[1\right]-\expect_{\X\sim
  \dist{P}}\left[1\right]\right) \nonumber\\
 & = & \expect_{\X\sim
  \dist{P}}\left[-\log_{\chi}\left(\frac{\dens{Q}(\X)}{\dens{P}(\X)}\right) \right] \nonumber\\
 & = & KL_{\chi}(\dist{P}\|\dist{Q})  \:\:,
\end{eqnarray}
as claimed. We finally check that $z \mapsto z/(1+k z)$ is increasing
and so is $t\mapsto \chi(t)/(1+k\chi(t))$ because $\chi$ is increasing, which is also non negative
and defined over $\mathbb{R}_+$ since $k\geq 0$, and so defines a
signature and a valid $\chi$-logarithm.

\section{Proof of Theorem \ref{thVIGGEN1}}\label{proof_thVIGGEN1}

Our basis for the proof of the Theorem is the following Lemma.
\begin{lemma}\label{propCSUBD}\citep[Proposition 1.6.1]{npCF}
Let $f: I \rightarrow \mathbb{R}$ be continuous convex and let $\xi :
I \rightarrow \mathbb{R}$ such that $\xi(z) \in \partial f(z), \forall
z\in \mathrm{int} I$. Then for any $a<b$ in $I$, it holds that:
\begin{eqnarray}
f(b) & = & f(a) + \int_a^b \xi(t) \mathrm{d}t\:\:.
\end{eqnarray}
\end{lemma}
Suppose that $b<a$. Then Lemma \ref{propCSUBD} says that we have $f(a) = f(b) + \int_b^a
\xi(t) \mathrm{d}t$, that is, after reordering, $f(b) = f(a) - \int_b^a
\xi(t) \mathrm{d}t = f(a) + \int_a^b
\xi(t) \mathrm{d}t$, so in fact the requested ordering between the
integral's bounds can be removed. Also, we can suppose that the
integral may not be proper, in which case we compute it as a limit of
a proper integral for which Lemma \ref{propCSUBD} therefore holds.\\

We now prove Theorem \ref{thVIGGEN1}. Suppose there exists $M \in \mathbb{R}$ such that
$\sup \xi(\mathbb{I}_{P,Q}) \leq M$, for some $\partial f \ni \xi : \mbox{ int }\mathrm{dom}(f) \rightarrow
\mathbb{R}$.
For any constants $k$, letting $f_{k}(z) \defeq
f(z) - k(z-1)$, which is convex since $f$ is, we
note that
\begin{eqnarray}
\expect_{\X\sim
  \dist{Q}}\left[f_{k}
  \left(\frac{P(\X)}{Q(\X)}\right)\right] & = & \expect_{\X\sim
  \dist{Q}}\left[f
  \left(\frac{P(\X)}{Q(\X)}\right)\right] - k\cdot \expect_{\X\sim
  Q}\left[\frac{P(\X)}{Q(\X)} - 1\right]\nonumber\\
 & = &  \expect_{\X\sim
  \dist{Q}}\left[f
  \left(\frac{P(\X)}{Q(\X)}\right)\right] - k\cdot \left(\int P(\X)
\mathrm{d}\mu(\X) - \int Q(\X)
\mathrm{d}\mu(\X)\right)\nonumber\\
 & = & \expect_{\X\sim
  \dist{Q}}\left[f
  \left(\frac{P(\X)}{Q(\X)}\right)\right]\:\:.\label{inv1}
\end{eqnarray}
Let $\xi_k \defeq \xi - k \in \partial f_{k}$.
Since $f_k$ is convex continuous, it follows from \cite[Proposition
1.6.1]{npCF} (Lemma \ref{propCSUBD}) that:
\begin{eqnarray}
f_k\left(\frac{P(\ve{x})}{Q(\ve{x})}\right) & = & f_{k}(1)
+ \lim_{\rho \rightarrow \frac{P(\ve{x})}{Q(\ve{x})}} \int_{1}^{\rho}
\xi_{k}(t) \mathrm{d}t\nonumber\\
& = &  -\lim_{\rho \rightarrow \frac{P(\ve{x})}{Q(\ve{x})}} \int_{1}^{\rho}
(-\xi(t)+k) \mathrm{d}t\:\:. \label{inv2}
\end{eqnarray}
The second identity comes from the assumption that $f(1) = 0 = f_{k}(1)$.
The limit appears to cope with a subdifferential that would diverge
around a density ratio. Fix some constant $\epsilon > 0$ and let
\begin{eqnarray}
\chi (t) & = & \left\{
\begin{array}{rcl}
\frac{1}{-\xi(t) + M +\epsilon} & \mbox{ if } & t < \sup \mathbb{I}_{P,Q}\\
\frac{1}{\epsilon} & \mbox{ if } & t \geq \sup \mathbb{I}_{P,Q}
\end{array}
\right.\:\:,\label{defchi11}
\end{eqnarray}
which, since $\sup \xi (\mathbb{I}_{P,Q}) \leq M$, guarantees $\chi \geq 0$ and $\chi$ is also increasing
since $\xi$ is increasing ($f$ is convex). We then check, using
eqs. (\ref{inv1}) and (\ref{defchi11}) that:
\begin{eqnarray}
KL_\chi(\dist{Q}\|\dist{P}) & = & \expect_{\X\sim
  \dist{Q}}\left[-\log_\chi
  \left(\frac{P(\X)}{Q(\X)}\right)\right] \nonumber\\
 & = & \expect_{\X\sim
  \dist{Q}}\left[-\lim_{\rho \rightarrow \frac{P(\X)}{Q(\X)}} \int_{1}^{\rho}
\frac{1}{\chi(t)} \mathrm{d}t \right] \nonumber\\
 & = & \expect_{\X\sim
  \dist{Q}}\left[-\lim_{\rho \rightarrow \frac{P(\X)}{Q(\X)}} \int_{1}^{\rho}
(-\xi(t) + M +\epsilon) \mathrm{d}t \right] \nonumber\\
 & = & \expect_{\X\sim
  \dist{Q}}\left[f_{M+\epsilon}
  \left(\frac{P(\X)}{Q(\X)}\right) \right] \nonumber\\
 & = & \expect_{\X\sim
  \dist{Q}}\left[f
  \left(\frac{P(\X)}{Q(\X)}\right) \right] = I_f(\dist{P}\|\dist{Q}) \:\:.
\end{eqnarray}
This ends the proof of Theorem \ref{thVIGGEN1}.

\section{Proof of Theorem \ref{thVIGGEN2}}\label{proof_thVIGGEN2}

Without loss of generality we can assume that $\sup \mathbb{I}_{P,Q} <
+\infty$. Otherwise, when $\sup \mathbb{I}_{P,Q} = + \infty$,
requesting $\sup \xi(\mathbb{I}_{P,Q}) = +\infty$ ($\forall \xi \in \partial
f$) implies, because $f$
is convex, that $\lim_{\sup
  \mathbb{I}_{P,Q}}f(z) = +\infty$, and so the constraint $I_f(\dist{P}\|\dist{Q}) <
+\infty$ essentially enforces zero measure over all infinite density ratios.

We make use of \cite[Proposition 1.6.1]{npCF} (Lemma \ref{propCSUBD}), now with a
subdifferential which is not Riemann integrable in $M \defeq \sup
\mathbb{I}_{P,Q}$. Notice that we can assume without loss of
generality that $M > 1$ since otherwise, since it is convex, $f$ would not be defined for $z> 1$ and $I_f(\dist{P}\|\dist{Q})$
would essentially be infinite unless $Q \geq P$ almost everywhere
(\textit{i.e.} $P$ dominates $Q$ only on sets of zero measure).\\

For any constants $\epsilon$ and
$t^* < M$ such that $\xi(t^*) < +\infty$, let
\begin{eqnarray}
g_{t^*,\epsilon}(z) & \defeq & \int_{1}^{z}
(-\xi(t)+\xi(t^*)+\epsilon)\mathrm{d}t\:\:, 
\end{eqnarray}
where $z \in \mathbb{R}_+$ is any real such that the integral in
$g_{t^*,\epsilon}$ is not improper
(therefore, $z< M$). Let
\begin{eqnarray}
\chi_{t^*,\epsilon}(t) & = & \left\{
\begin{array}{rcl}
\frac{1}{-\xi(t) +\xi(t^*)+\epsilon} & \mbox{ if } & t < t^*\\
\frac{1}{\epsilon} & \mbox{ if } & t \geq t^*
\end{array}
\right.\:\:,
\end{eqnarray}
which, if $\epsilon > 0$, is non negative and also increasing
since $\xi $ is increasing. Consider any fixed $z^* \in \mathbb{I}_{P,Q}
\cap (1,\infty)$ with $0<\xi(z^*) < \infty$ and let $t^* \defeq \sup
\{z : \xi(z) \leq \xi(z^*)\}$. We have:
\begin{eqnarray}
g_{t^*,\epsilon}(z) & = & \int_{1}^{z}
(-\xi(t)+\xi(t^*)+\epsilon)\mathrm{d}t\nonumber\\
& = & \int_{1}^{z}
\frac{1}{\chi_{t^*,\epsilon}(t)} \mathrm{d}t + 1_{[z \geq t^*]} \cdot \int_{t^*}^{z}
(-\xi(t)+ \xi(t^*))\mathrm{d}t\nonumber\\
& =  & \int_{1}^{z}
\frac{1}{\chi_{t^*,\epsilon}(t)} \mathrm{d}t  -
1_{[z \geq t^*]} \cdot \int_{t^*}^{z}
(\xi(t)- \xi(t^*))\mathrm{d}t\nonumber\\
 & = & \int_{1}^{z}
\frac{1}{\chi_{t^*,\epsilon}(t)} \mathrm{d}t  -
1_{[z\geq t^*]} \cdot D_{f,
  \xi}\left(\left. z\right\| t^*\right) \label{inv3}\:\:.
\end{eqnarray}
The last identity comes from \cite[Proposition 1.6.1]{npCF} (Lemma \ref{propCSUBD}) and the
fact that $\xi(t)- \xi(t^*)$ belongs to the subdifferential of the
Bregman divergence whose generator is $f$ \cite{frCF} (being convex in
its left parameter we can apply Lemma \ref{propCSUBD}). We extend
hereafter the definition of Bregman divergences to non-differentiable
functions, and let $D_{f, \xi}$ denote the Bregman divergence with
generator the (convex) $f$ in which we replace the gradient by $\xi
\in \partial f$. We obtain:
\begin{eqnarray}
\lefteqn{\expect_{\X\sim
  \dist{Q}}\left[f
  \left(\frac{P(\X)}{Q(\X)}\right)\right]}\nonumber\\
 & = & \expect_{\X\sim
  \dist{Q}}\left[f_{\xi(t^*)+\epsilon}
  \left(\frac{P(\X)}{Q(\X)}\right)\right]\label{cons11}\\
 & = & \expect_{\X\sim
  \dist{Q}}\left[\lim_{\rho \rightarrow \frac{P(\X)}{Q(\X)}} -\int_{1}^{\rho}
(-\xi(t)+\xi(t^*)+\epsilon) \mathrm{d}t\right]\label{cons12}\\
 & = & \expect_{\X\sim
  Q}\left[\lim_{\rho \rightarrow \frac{P(\X)}{Q(\X)}} -g_{t^*,\epsilon}(\rho)\right]\nonumber\\
 & = & \expect_{\X\sim
  \dist{Q}}\left[\lim_{\rho \rightarrow \frac{P(\X)}{Q(\X)}}
  \left\{-\int_{1}^{\rho}
\frac{1}{\chi_{t^*,\epsilon}(t)} \mathrm{d}t +
1_{\left[\rho\geq t^*\right]} \cdot D_{f,
  \xi}\left(\left. \rho \right\| t^*\right) \right\}\right]\label{cons13}\\
 & = & \expect_{\X\sim
  \dist{Q}}\left[-\lim_{\rho \rightarrow \frac{P(\X)}{Q(\X)}}
  \int_{1}^{\rho}
\frac{1}{\chi_{t^*,\epsilon}(t)} \mathrm{d}t\right] + \underbrace{\expect_{\X\sim
  Q}\left[\lim_{\rho \rightarrow \frac{P(\X)}{Q(\X)}}
1_{\left[\rho\geq t^*\right]} \cdot D_{f,
  \xi}\left(\left. \rho \right\| t^*\right) \right]}_{\defeq R(t^*)}\label{cons15}\\
 & = & \expect_{\X\sim
  \dist{Q}}\left[-\log_{\chi_{t^*,\epsilon}}\left(\frac{P(\X)}{
      Q(\X)}\right)\right] + R(t^*)\label{cons16}\\
 & = & KL_{\chi_{t^*,\epsilon}}(\dist{Q}\|\dist{P}) + R(t^*)\label{cons17}\:\:.
\end{eqnarray}
\begin{figure}[t]
\begin{center}
\begin{tabular}{c}
\includegraphics[trim=20bp 430bp 610bp
60bp,clip,width=0.80\columnwidth]{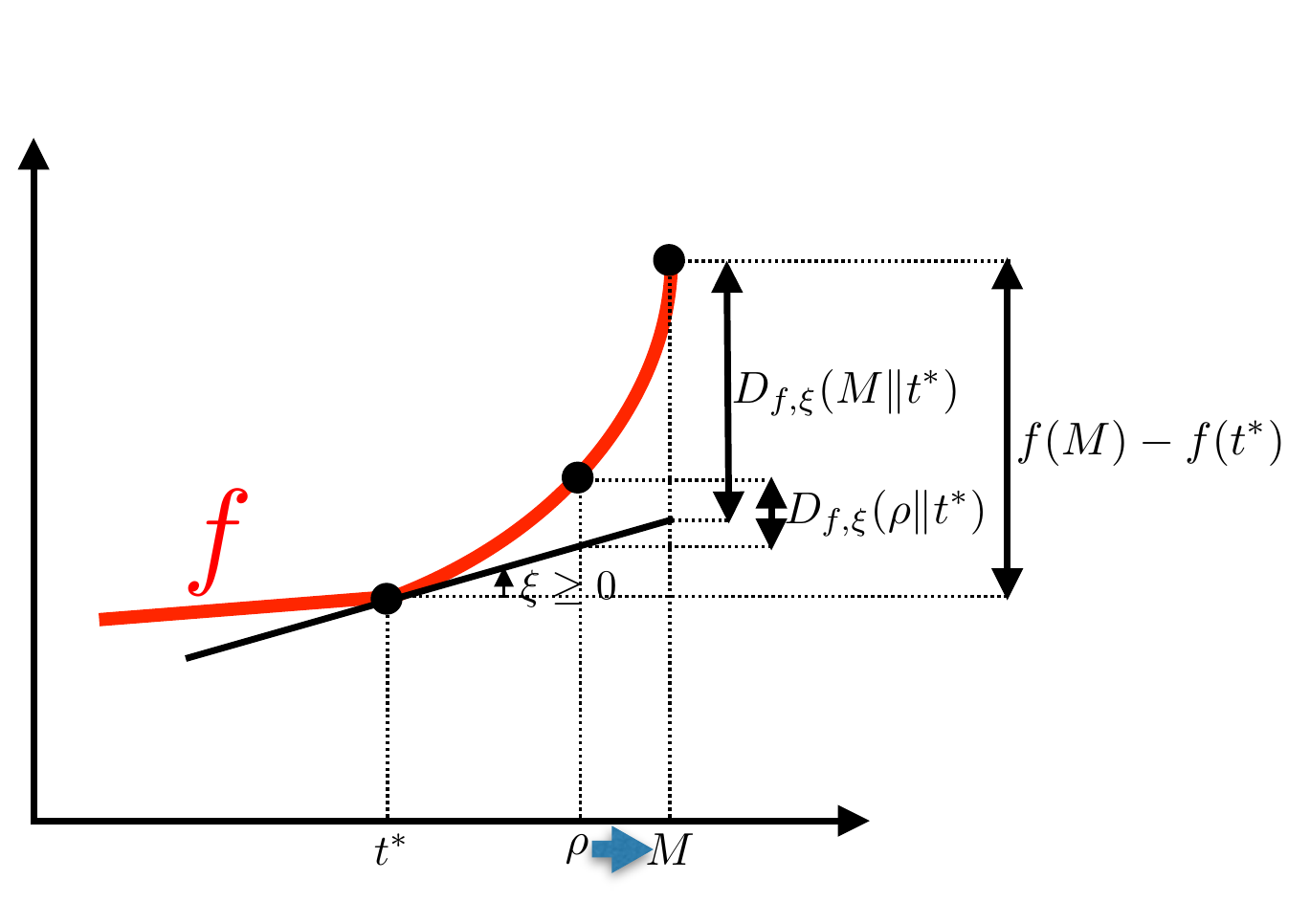} 
\end{tabular} 
\end{center}
\caption{Illustration of ineq. (\ref{ineqDEFDIV}).}
\label{f-bound}
\end{figure}
Eq. (\ref{cons11}) follows from Eq. (\ref{inv1}).
Eq. (\ref{cons12}) follows from Eq. (\ref{inv2}).
Eq. (\ref{cons13}) follows from Eq. (\ref{inv3}). We can split the
limits in eq. (\ref{cons15}) because each term in the expectation of
$R(t^*)$ is finite. To see it, since $I_f(\dist{P}\|\dist{Q}) < \infty$ and $\sup
\xi (\mathbb{I}_{P,Q}) = +\infty$, we can assume that $f(M) <
+\infty$. Since $\xi(t^*) \geq
0$, then $f$ is non decreasing for $x\geq t^*$ and
\begin{eqnarray}
 t^* \leq \rho\leq M & \Rightarrow & D_{f,
  \xi}\left(\left. \rho \right\| t^*\right) \leq D_{f,
  \xi}\left(\left. M \right\| t^*\right)  \leq f(M) - f(t^*) \:\:,\label{ineqDEFDIV}
\end{eqnarray}
which is indeed finite. Figure \ref{f-bound} provides an
illustration of this bound.
It then comes 
\begin{eqnarray}
R(t^*) = \expect_{\X\sim
  \dist{Q}}\left[\lim_{\rho \rightarrow \frac{P(\X)}{Q(\X)}}
1_{\left[\rho\geq t^*\right]} \cdot D_{f,
  \xi}\left(\left. \rho \right\| t^*\right) \right] & = & \expect_{\X\sim
  \dist{Q}}\left[
1_{\left[\frac{P(\X)}{Q(\X)}\geq t^*\right]} \cdot D_{f,
  \xi}\left(\left. \frac{P(\X)}{Q(\X)} \right\| t^*\right)
\right]\nonumber\\
 & \leq & \expect_{\X\sim
  \dist{Q}}\left[D_{f,
  \xi}(M\|t^*)
\right] \nonumber\\
 & \leq & f(M) - f(t^*)\:\:,
\end{eqnarray}
where we have used ineq. (\ref{ineqDEFDIV}) in the last inequality. Since $f$ is continuous, we get the upperbound on $I_f(\dist{P}\|\dist{Q})$ by choosing $t^*<M$ as
close as desired to $M$. We get the lowerbound by remarking that
$R(t^*) \geq 0$ (a Bregman divergence cannot be negative).

\section{Proof of Theorem \ref{thVIGGEN3}}\label{proof_thVIGGEN3}

We have
\begin{eqnarray*}
\lefteqn{\expect_{\X\sim \escort{{\dist{Q}}}}[-(\log_\chi({P}(\X))-\log_\chi({Q}(\X)))]}\nonumber\\
& = &
\expect_{\X\sim
  \escort{{\dist{Q}}}}[-(\log_\chi({P}(\X))-\log_\chi(\escort{{Q}}(\X)))]  + \expect_{\X\sim \escort{{\dist{Q}}}}[-(\log_\chi(\escort{{Q}}(\X))-\log_\chi({Q}(\X)))]\nonumber\\
& = &
\expect_{\X\sim
  \escort{{\dist{Q}}}}[-(\log_\chi({P}(\X))-\log_\chi(\escort{{Q}}(\X)))]  - \expect_{\X\sim \escort{{\dist{Q}}}}[-(\log_\chi({Q}(\X))-\log_\chi(\escort{{Q}}(\X)))]\:\:.
\end{eqnarray*}
Consider some fixed $\ve{x} \in \mathcal{X}$. We have
\begin{eqnarray}
\log_\chi({P}(\ve{x}))-\log_\chi({\escort{{Q}}}(\ve{x})) & = & \int_{{{1}}}^{P (\ve{x})}
\frac{1}{\chi(t)} \cdot \mathrm{d}t - \int_1^{\escort{{Q}}(\ve{x})}
\frac{1}{\chi(t)} \cdot \mathrm{d}t\nonumber\\
& = & \int_{\escort{{Q}}(\ve{x})}^{P (\ve{x})}
\frac{1}{\chi(t)} \cdot \mathrm{d}t\nonumber\\
 & = & \int_{1}^{\frac{P (\ve{x})}{\escort{{Q}}(\ve{x})}}
\frac{\escort{{Q}}(\ve{x})}{\chi(t \escort{{Q}}(\ve{x}))} \cdot \mathrm{d}t\nonumber\\
 & = & \int_{1}^{\frac{P (\ve{x})}{\escort{{Q}}(\ve{x})}}
\frac{1}{\chi_{\escort{{Q}}(\ve{x})}(t)} \cdot \mathrm{d}t\nonumber\\
 & = & \log_{\chi_{\escort{{Q}}(\ve{x})}}\left(\frac{P (\ve{x})}{\escort{{Q}}(\ve{x})}\right)\:\:,
\end{eqnarray}
with 
\begin{eqnarray}
\chi_{\escort{{Q}}(\ve{x})}(t) & \defeq & \frac{1}{\escort{{Q}} (\ve{x})}\cdot \chi
(t\escort{{Q}}(\ve{x}))\label{defCHITILDE} \:\:.
\end{eqnarray}
To cope with the case where any of the integrals is improper, we
derive the limit expression:
\begin{eqnarray}
(\log_\chi({P}(\ve{x}))-\log_\chi({\escort{{Q}}}(\ve{x}))) & = &
\lim_{(p, q)\rightarrow (P (\ve{x}), \escort{Q} (\ve{x}))} \log_{\chi_{q}}\left(\frac{p}{q}\right)\:\:,
\end{eqnarray}
so we get in all cases,
\begin{eqnarray}
\expect_{\X\sim
  \escort{{\dist{Q}}}}[-(\log_\chi({P}(\X))-\log_\chi(\escort{{Q}}(\X)))] & = &
KL_{\chi_{\escort{{Q}}}}(\escort{Q}\|P)\:\:.
\end{eqnarray}
We also note that 
\begin{eqnarray}
\log_\chi({Q}(\X))-\log_\chi(\escort{{Q}}(\X)) & = &
\lim_{(q, q')\rightarrow (Q (\ve{x}), \escort{Q} (\ve{x}))} \log_{\chi_{q'}}\left(\frac{q}{q'}\right)\nonumber\\
 & = & \log_{\chi_{\escort{{Q}}(\ve{x})}}\left(\frac{Q (\ve{x})}{\escort{{Q}}(\ve{x})}\right)
\end{eqnarray}
(if the limit exists) so we get
\begin{eqnarray}
\expect_{\X\sim
  \escort{{\dist{Q}}}}[-(\log_\chi({P}(\X))-\log_\chi(\escort{{Q}}(\X)))] & = &
KL_{\chi_{\escort{{Q}}}}(\escort{Q}\|P)\:\:,\\
\expect_{\X\sim \escort{{Q}}}[-(\log_\chi({Q}(\X))-\log_\chi(\escort{{Q}}(\X)))] & = &
KL_{\chi_{\escort{{Q}}}}(\escort{Q}\|Q)\:\:,
\end{eqnarray}
and
\begin{eqnarray}
\expect_{\X\sim \escort{\dist{Q}}}[\log_\chi(Q(\X)) - \log_\chi(P(\X))] & = &
KL_{\chi_{\escort{{Q}}}}(\escort{\dist{Q}}\|\dist{P}) - KL_{\chi_{\escort{{Q}}}}(\escort{\dist{Q}}\|\dist{Q})\:\:,
\end{eqnarray} 
as claimed.

\section{Proof of Theorem \ref{thVIGGEN4}}\label{proof_thVIGGEN4}

Let us denote $\mathcal{F}_{\escort{Q}} \subseteq
\mathbb{R}^{\mathcal{X}}$ denote the subset of functions
$:\mathcal{X}\rightarrow \mathbb{R}$ whose values are constrained as follows: 
\begin{eqnarray}
\mathcal{F}_{\escort{Q}} & \defeq & \left\{T\in
  \mathbb{R}^{\mathcal{X}} : T(\ve{x}) \in \mathrm{dom} \left(-\log_{\chi_{\escort{{Q}}(\ve{x})}}\right)^\star\right\}\:\:.
\end{eqnarray}
Since $-\log_{\chi_{\escort{{Q}}(\ve{x})}}$ is convex for any
$\ve{x}$, it follows from Legendre duality,
\begin{eqnarray}
KL_{\chi_{\escort{{Q}}}}(\escort{\dist{Q}}\|\dist{P}) & = & \expect_{\X\sim
  \escort{{\dist{Q}}}}\left[-\log_{\chi_{\escort{{Q}}(\X)}}\left(\frac{P
      (\X)}{\escort{{Q}}(\X)}\right)\right]\nonumber\\
& = & \expect_{\X\sim
  \escort{{\dist{Q}}}}\left[\sup_{T(\X) \in \mathrm{dom}
    \left(\log_{\chi_{\escort{{Q}}(\X)}}\right)^\star} \left\{T(\X)\cdot \frac{P
      (\X)}{\escort{{Q}}(\X)} -  (-\log_{\chi_{\escort{Q}(\X)}})^\star(T(\X))\right\}\right]\nonumber\\
& = & \sup_{T \in \mathcal{F}_{\escort{Q}}} \left\{\expect_{\X\sim
  \escort{{\dist{Q}}}}\left[T(\X)\cdot \frac{P
      (\X)}{\escort{{Q}}(\X)} -  (-\log_{\chi_{\escort{Q}(\X)}})^\star(T(\X))\right]\right\}\nonumber\\
 & = & \sup_{T \in \mathcal{F}_{\escort{Q}}} \left\{\expect_{\X \sim \dist{P}} [T(\X)]  -
\expect_{\X \sim \escort{\dist{Q}}}
[(-\log_{\chi_{\escort{Q}(\X)}})^\star(T(\X))]\right\}\label{varprob}\:\:.
 \end{eqnarray}
Now, we know that $-\log_{\chi_{\escort{{Q}}(\ve{x})}}(z)$ is proper
lower-semicontinuous and therefore
$(-\log_{\chi_{\escort{{Q}}(\ve{x})}})^{\star\star} =
-\log_{\chi_{\escort{{Q}}(\ve{x})}}$. Being closed, the domain of the
derivative of $(-\log_{\chi_{\escort{{Q}}(\ve{x})}})^\star$ is the
image of the derivative of $-\log_{\chi_{\escort{{Q}}(\ve{x})}}$, given
by $-\escort{{Q}}(\ve{x})/\chi(\escort{{Q}}(\ve{x}) t)$. If
$\chi : \mathbb{R}_+ \rightarrow \mathbb{R}_+$, then
$-\escort{{Q}}(\ve{x})/\chi(\escort{{Q}}(\ve{x}) t) \in
\overline{\mathbb{R}_{++}}, \forall \escort{{Q}}(\ve{x})$ and so $\mathcal{F}_{\escort{Q}} = \left\{T\in
  \overline{\mathbb{R}_{++}}^{\mathcal{X}} \right\}$.\\

A pointwise differentiation of eq. (\ref{varprob})
yields that at the optimum, we have 
\begin{eqnarray}
P(\ve{x}) - \escort{Q}(\ve{x})\cdot
{(-\log_{\chi_{\escort{Q}(\ve{x})}})^\star}'(T(\ve{x})) & = & P(\ve{x}) - \escort{Q}(\ve{x})\cdot 
{(-\log_{\chi_{\escort{Q}(\ve{x})}})'}^{-1}(T(\ve{x}))\nonumber\\
 & = & 0\:\:,
\end{eqnarray}
that is, exploiting the fact that $(-\log_{\chi_{\escort{Q}(\ve{x})}})' = -\escort{{Q}}(\ve{x})/\chi(\escort{{Q}}(\ve{x}) t)$,
\begin{eqnarray}
T^*(\ve{x}) & = & (-\log_{\chi_{\escort{Q}}})' \left(\frac{P(\ve{x})}{\escort{Q}(\ve{x})}\right) \nonumber\\
 & = &
 -\frac{\escort{Q}(\ve{x})}{\chi
   \left(\frac{P(\ve{x})}{\escort{Q}(\ve{x})}\cdot \escort{Q}(\ve{x})\right)} \label{eqCHIBULLET}\\
 & = &
 -\frac{\escort{Q}(\ve{x})}{\chi
   (P(\ve{x}))}\nonumber\\
 & = & - \frac{1}{Z} \cdot \frac{\chi(Q(\ve{x}))}{\chi
   (P(\ve{x}))}\:\:.\label{eqCHIBUL2}
\end{eqnarray}

\section{Proof of Theorem \ref{thVIGGEN5}}\label{proof_thVIGGEN5}

We now elicitate 
$(-\log_{\chi_{q}})^\star$ for $q\in \mathbb{R}_+$, under the
conditions of Theorem \ref{thVIGGEN4}. By definition,
\begin{eqnarray}
(-\log_{\chi_{q}})^\star(z) & \defeq & \sup_{z'\in \mathbb{R}_{+}} \left\{zz' -
(-\log_{\chi_q}(z'))\right\}\nonumber\\
 & = & \sup_{z'\in \mathbb{R}_{+}} \left\{zz' +
\int_1^{z'} \frac{q}{\chi(qt)} \mathrm{d}t\right\}\nonumber\\
 & = & \sup_{z'\in \mathbb{R}_{+}} \left\{z +
\int_1^{z'} \left(z + \frac{q}{\chi(qt)} \right)\mathrm{d}t\right\}\nonumber\\
 & = & z + \sup_{z'\in \mathbb{R}_{+}} \left\{\int_1^{z'} \left(z + \frac{q}{\chi(qt)} \right)\mathrm{d}t\right\}
\end{eqnarray}
Because $\mathrm{dom}
\log_{\chi_q} \subseteq \mathbb{R}_+$ and $q/\chi(qt)\geq 0$, the $\sup$
is unbounded if $z> 0$. If $z=0$, it is bounded iff
\begin{eqnarray}
\sup_z \log_{\chi_{q}}(z) & < & \infty\:\:.
\end{eqnarray}
Otherwise, when $z<0$, it reaches its maximum when $z'$
belongs to the integrand's zeroes, $\{t : z +
q/\chi(qt) = 0\}$, or equivalently, when $z'$ satisfies:
\begin{eqnarray}
\chi(qz') & = & -\frac{q}{z}\:\:.\label{defFF}
\end{eqnarray}
Let us denote 
\begin{eqnarray}
h(t) & \defeq & \frac{q}{\chi(qt)}
\end{eqnarray} 
for short ($q\geq 0$), noting that $h$ is non increasing. The set of reals for
which eq. (\ref{defFF}) holds is $\mathcal{Z} = h^{-1}(-z) \defeq \{z' : q/\chi(qz') = -z\}$, which may not
be a singleton if $\chi$ is not invertible. For any $z^*\in
\mathcal{Z}$, letting $h(t) \defeq q/\chi(qt)$ for short, we get:
\begin{eqnarray}
(-\log_{\chi_{q}})^\star(z) & = & z z^* + \int_1^{z^*} \frac{q}{\chi(qt)}
\mathrm{d}t\nonumber\\
 & = & zz^* + \int_{-z}^{h(1)} h^{-1}(t)
\mathrm{d}t - 1\cdot(h(1) - (-z)) + (-z)\cdot(z^* - 1) \label{eqsum0}\\
 & = & -h(1) + \int_{-z}^{h(1)} h^{-1}(t)
\mathrm{d}t \nonumber\\
& = & -h(1) - \int_{h(1)}^{-z} h^{-1}(t)
\mathrm{d}t \nonumber\\
& = & \underbrace{-h(1) + \int_{1}^{h(1)} h^{-1}(t)
\mathrm{d}t}_{\defeq k(q)} + \underbrace{\int_{1}^{-z} -h^{-1}(t)
\mathrm{d}t}_{\defeq B(z)} \label{sum3}\:\:.
\end{eqnarray}
\begin{figure}[t]
\begin{center}
\begin{tabular}{c}
\includegraphics[trim=5bp 420bp 420bp
50bp,clip,width=0.70\columnwidth]{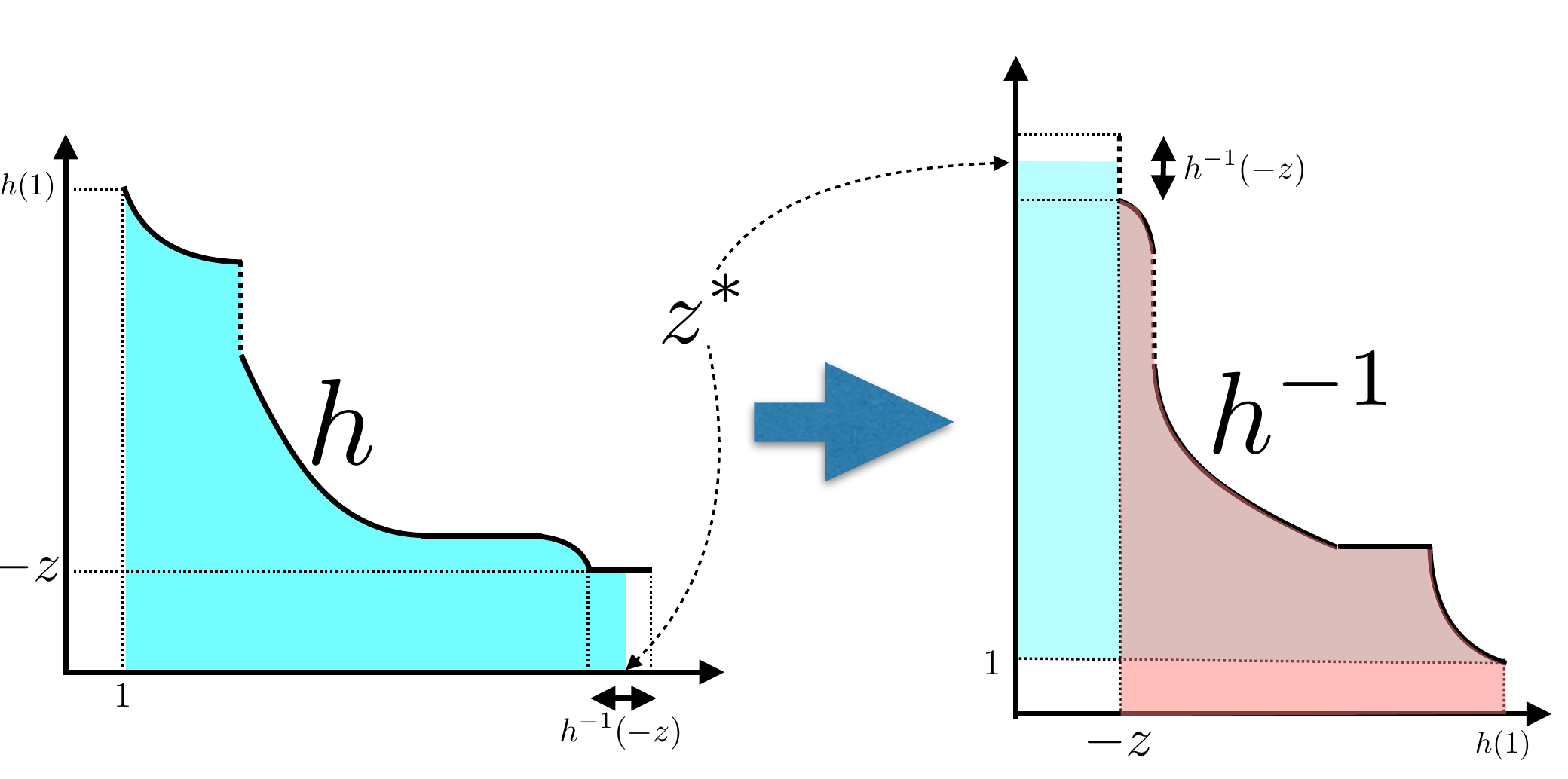} 
\end{tabular} 
\end{center}
\caption{Explanation of eq. (\ref{eqsum0}).}
\label{f-inte}
\end{figure}
The derivation in eq. (\ref{eqsum0}) is explained in Figure
\ref{f-inte}. We remark that $k$ depends only on $q$, so it is not affected
by the choice of $T$. Concerning $B(z)$, we have
\begin{eqnarray}
B(z)  & = & \int_{1}^{-z}
 -\left(\frac{q}{\chi(qt)}\right)^{-1} \mathrm{d}t \nonumber\\
  & = & \int_{1}^{-z}
 -\left(\frac{1}{\chi_q(t)}\right)^{-1} \mathrm{d}t \nonumber\\
  & = & -\int_{1}^{-z}
 (\chi_q)^{-1}\left(\frac{1}{t}\right)\mathrm{d}t \:\:,
\end{eqnarray}
and finally, letting
\begin{eqnarray}
\chi^\bullet(t) & \defeq & \frac{1}{\chi^{-1}\left(\frac{1}{t}\right)}\:\:,
\end{eqnarray}
we remark that
\begin{eqnarray}
(\chi^\bullet)_{q}(t) & \defeq & \frac{1}{q}\chi^\bullet(qt)\nonumber\\
 & = & \frac{1}{q\chi^{-1}\left(\frac{1}{qt}\right)} \nonumber\\
 & = & \frac{1}{(\chi^{-1})_{\frac{1}{q}}\left(\frac{1}{t}\right)} \nonumber\\
 & = & \frac{1}{\left(\chi_{\frac{1}{q}}\right)^{-1}\left(\frac{1}{t}\right)}\:\:,
\end{eqnarray}
so 
\begin{eqnarray}
(\chi^\bullet)_{\frac{1}{q}}(t) &  = & \frac{1}{\left(\chi_{q}\right)^{-1}\left(\frac{1}{t}\right)}\:\:,\label{eqCHIBULLET0}
\end{eqnarray}
and finally
\begin{eqnarray}
B(z) & = & -\log_{(\chi^\bullet)_{\frac{1}{q}}}(-z)\:\:.
\end{eqnarray}
We can check that whenever $\chi^\bullet$ is differentiable, 
\begin{eqnarray}
(\chi^\bullet)'(t) & \defeq & \frac{1}{t^2 \cdot \chi'
  \left(\chi^{-1}\left(\frac{1}{t}\right)\right)\cdot\left(\chi^{-1}\left(\frac{1}{t}\right)\right)^2}
\geq 0\:\:,
\end{eqnarray}
so that $\chi^\bullet$ is non decreasing and since it is positive, it defines a $\chi^\bullet$-logarithm.
We end up with
\begin{eqnarray}
\lefteqn{\sup_{T \in \overline{\mathbb{R}_{++}}^{\mathcal{X}}} \left\{\expect_{\X \sim \dist{P}} [T(\X)]  -
\expect_{\X \sim \escort{\dist{Q}}}
[(-\log_{\chi_{\escort{Q}(\X)}})^\star(T(\X))]\right\}}\nonumber\\
 & = & \sup_{T \in \overline{\mathbb{R}_{++}}^{\mathcal{X}}} \left\{\expect_{\X \sim \dist{P}} [T(\X)]  -
\expect_{\X \sim \escort{\dist{Q}}}
\left[k(\escort{Q}(\X)) - \log_{(\chi^\bullet)_{\frac{1}{\escort{Q}(\X)}}}(-T(\X))\right]\right\}\nonumber\\
 & = & \sup_{T \in \overline{\mathbb{R}_{++}}^{\mathcal{X}}} \left\{\expect_{\X \sim \dist{P}} [T(\X)]  -
\expect_{\X \sim \escort{\dist{Q}}}
\left[-
  \log_{(\chi^\bullet)_{\frac{1}{\escort{Q}(\X)}}}(-T(\X))\right]\right\} - \expect_{\X \sim \escort{\dist{Q}}}
\left[k(\escort{Q}(\X))\right]\nonumber\\
 & = & \sup_{T \in \overline{\mathbb{R}_{++}}^{\mathcal{X}}} \left\{\expect_{\X \sim \dist{P}} [T(\X)]  -
\expect_{\X \sim \escort{\dist{Q}}}
\left[-
  \log_{(\chi^\bullet)_{\frac{1}{\escort{Q}(\X)}}}(-T(\X))\right]\right\} - K(\escort{\dist{Q}})\:\:,
\end{eqnarray}
as claimed. We finally remark that it is clear from Figure
\ref{f-inte} that $\chi$ being used to compute integrals,
it does not need to be strictly monotonic for this to be possible: we
just have to break the continuity in $\chi^{-1}(y)$ whenever the set
$\mathbb{I}$ defined by
$\chi(\mathbb{I}) = y$ is of non-zero Lebesgue measure taking
care that $\chi^{-1}(y)$ be still defined in $y$. This does
not change the integral values.

\section{Proof of Theorem \ref{thmSUP}}\label{proof_thmSUP}

The proof of the Theorem mainly follows from
identifying the parameters of eq. (\ref{eqn:proper-gan}) with the
variational part of eq. (\ref{eqFUND1}).
Recall from eq. (\ref{eqCHIBULLET0}) that
\begin{eqnarray}
(\chi^\bullet)_{\frac{1}{q}}(t) &  = & \frac{1}{\left(\chi_{q}\right)^{-1}\left(\frac{1}{t}\right)}\:\:,
\end{eqnarray}
so, exploiting eq. (\ref{eqSIMPLCONJ}) (Theorem \ref{thVIGGEN5}) and the fact that $K(\dist{Q})$ does
not depend on $T$, we get:
\begin{eqnarray}
\ell_{\ve{x}}'( -1, z ) & = & \frac{\mathrm{d}}{\mathrm{d}z} 
  (-\log_{\chi_{\escort{Q}(\ve{x})}})^\star(-z) \nonumber\\
 & = & \frac{\mathrm{d}}{\mathrm{d}z} -
  \log_{(\chi^\bullet)_{\frac{1}{\escort{Q}(\ve{x})}}}(-z) \nonumber\\
 & = & \left(\chi_{\escort{Q}(\ve{x})}\right)^{-1}\left(-\frac{1}{z}\right)\:\:.
\end{eqnarray}
Since $\ell'( +1, z ) = -1$, we deduce that the loss is proper composite with inverse link function
\citep[Corollary 12]{rwCB} given by:
\begin{eqnarray}
	\Psi_{\ve{x}}^{-1}( z ) &= & \frac{\ell'( -1, z )}{\ell'( -1, z ) - \ell'( +1, z )} \nonumber\\
	&= & \frac{\left(\chi_{\escort{Q}(\ve{x})}\right)^{-1}\left(-\frac{1}{z}\right)}{\left(\chi_{\escort{Q}(\ve{x})}\right)^{-1}\left(-\frac{1}{z}\right) + 1},
\end{eqnarray}
so that the link is
\begin{eqnarray}
\Psi_{\ve{x}}( z ) & = & -\frac{1}{\chi_{\escort{Q}(\ve{x})} \left(\frac{z}{1 -
      z}\right)}\:\:.
\end{eqnarray}
\begin{remark}
We easily retrieve the optimal discriminator (Theorem \ref{thVIGGEN4}) but this
time from the proper composite loss, since (the first line is a
general property of $\Psi_{\ve{x}}$, see Section \ref{sec-sup}):
\begin{eqnarray}
T^*(\ve{\ve{x}}) & = & \Psi_{\ve{x}}\left(\frac{P(\ve{x})}{P( \ve{x} ) + \escort{Q}( \ve{x} )}\right) \nonumber\\
& = & -\frac{1}{{\chi_{\escort{Q}(\ve{x})}}\left( \frac{\frac{P(\ve{x})}{P( \ve{x} ) + \escort{Q}( \ve{x} )}}{1 - \frac{P(\ve{x})}{P( \ve{x} ) + \escort{Q}( \ve{x} )}} \right)}. \nonumber\\
 & = &
 -\frac{1}{\chi_{\escort{Q}(\ve{x})} \left(\frac{P(\ve{x})}{\escort{Q}(\ve{x})}\right)} \nonumber\\
 & = &  - \frac{1}{Z} \cdot \frac{\chi(Q(\ve{x}))}{\chi
   (P(\ve{x}))}\:\:.\nonumber
\end{eqnarray}
The last identity follows from eqs. (\ref{eqCHIBULLET}) ---
(\ref{eqCHIBUL2}).
\end{remark}

\section{Proof of Theorem \ref{factorDEEP}}\label{proof_factorDEEP}

In the context of the proof, we simplify notations and replace
signature $\chinet$ by $\chi$ and output activation $\vout$ by $v_2$.
Let us call $\ve{z} \in \mathbb{R}^d$ the output of $g$. We revert the
transformation and check:
\begin{eqnarray}
\ve{\phi}_{\newell-1}(\ve{z}) & \defeq &
\matrice{w}_{\newell}^{-1}
(\ve{v}^{-1}(\ve{\phi}_{\newell}(\ve{z})) -
\ve{b}_{\newell}) \:\:, \forall \newell \in \{1,
2, ..., L\}\:\:,\\
\ve{\phi}_{L}(\ve{z}) & = &
\Gamma^{-1}\left(\ve{v}_2^{-1}(\ve{z}) -
  \ve{\beta}\right) \:\:.
\end{eqnarray}
For the sake of readability, we shall sometimes remove the dependence
in $\ve{z}$. Letting $a_i$ denote coordinate $i$ in vector $\ve{a}$, $(\matrice{a})_{ij}$ the coordinate in row $i$ and column $j$ of
matrix $\matrice{a}$,
for any $i, j \in [d]$, and $a_{l,i}$ coordinate $i$ in vector $\ve{a}_l$, we have
\begin{eqnarray}
\frac{\partial \phi_{\newell-1, i}}{\partial \phi_{\newell, j}} & = &
(\matrice{w}^{-1}_{\newell})_{ij} \cdot
\frac{1}{v'_i(\ve{v}^{-1}(\ve{\phi}_{\newell}))}
\label{defUU}\:\:,
\end{eqnarray}
and furthermore
\begin{eqnarray}
\frac{\partial \phi_{L, i}}{\partial z_{j}} & = & (\Gamma^{-1})_{ij} \cdot
\frac{1}{{v_2'}_i(\ve{v}_2^{-1}(\ve{z}))} \:\:. \label{defVV}
\end{eqnarray}
Let us denote vector $\tilde{\ve{a}}$ as the vector whose coordinates are
the inverses of those of $\ve{a}$, namely $\tilde{a}_i \defeq
1/a_i$. From eqs. (\ref{defUU}) and (\ref{defVV}), the layerwise
Jacobians are:
\begin{eqnarray}
\frac{\partial \ve{\phi}_{\newell-1}}{\partial \ve{\phi}^\top_{\newell}} & =
& \matrice{w}^{-1}_\newell \odot \tilde{\ve{v}'}(\ve{v}^{-1}(\ve{\phi}_{\newell})) \ve{1}^\top\:\:, \forall \newell \in \{1, 2, ...,
 L\}\:\:,\\
\frac {\partial \ve{\phi}_L}{\partial \ve{z}^\top} & =
& \Gamma^{-1} \odot \tilde{\ve{v}_2'}(\ve{v}_2^{-1}(\ve{z})) \ve{1}^\top\:\:,
\end{eqnarray}
where $\odot$ is Hadamard (coordinate-wise) product. These Jacobians
have a very convenient form, since:
\begin{eqnarray}
\mathrm{det} \left(\frac{\partial \ve{\phi}_{\newell-1}}{\partial
       \ve{\phi}_{\newell}^\top}\right) & = &
  \sum_{\ve{\sigma}\in S_d} \mathrm{sign}(\ve{\sigma})\cdot \prod_{i=1}^{d} \left( \matrice{w}^{-1}_\newell \odot \tilde{\ve{v}'}(\ve{v}^{-1}(\ve{\phi}_{\newell})) \ve{1}^\top\right)_{i,\sigma_i} \nonumber\\
 & = & \sum_{\ve{\sigma}\in S_d} \mathrm{sign}(\ve{\sigma})\cdot \prod_{i=1}^{d} (\matrice{w}^{-1})_{\newell, i, \sigma_i}
    \left(\tilde{\ve{v}'}(\ve{v}^{-1}(\ve{\phi}_{\newell})) \ve{1}^\top\right)_{i,\sigma_i}\nonumber\\
 & = & \sum_{\ve{\sigma}\in S_d} \left(\prod_{i=1}^{d}
   \tilde{v'}_i (\ve{v}^{-1}(\ve{\phi}_{\newell})) \right)\cdot \mathrm{sign}(\ve{\sigma})\cdot \prod_{i=1}^{d} (\matrice{w}^{-1})_{\newell, i, \sigma_i}\nonumber\\
 & = & \left(\prod_{i=1}^{d}
   \tilde{v'}_i (\ve{v}^{-1}(\ve{\phi}_{\newell})) \right)\cdot \sum_{\ve{\sigma}\in S_d} \mathrm{sign}(\ve{\sigma})\cdot \prod_{i=1}^{d} (\matrice{w}^{-1})_{\newell, i, \sigma_i}\nonumber\\
 & = & \left(\prod_{i=1}^{d}
   \tilde{v'}_i (\ve{v}^{-1}(\ve{\phi}_{\newell})) \right)\cdot
 \mathrm{det} \left(\matrice{w}^{-1}_\newell\right)\nonumber\\
 & = & \left(\prod_{i=1}^{d}
   \tilde{v'}_i (\ve{v}^{-1}(\ve{\phi}_{\newell})) \right)\cdot
 \left(\mathrm{det} \left(\matrice{w}_\newell\right)\right)^{-1}\:\:, \forall \newell \in \{1, 2, ...,
 L\}\:\:,\nonumber
\end{eqnarray}
and, using the same derivations,
\begin{eqnarray}
\mathrm{det} \left(\frac{\partial \ve{\phi}_L}{\partial
    \ve{z}^\top}\right) & = & \left(\prod_{i=1}^{d}
   \tilde{v_2'}_i (\ve{v}_2^{-1}(\ve{z})) \right)\cdot
 \left(\mathrm{det} \left(\Gamma\right)\right)^{-1}\:\:.
\end{eqnarray}
The change of variable formula \citep{dsbDE} yields:
\begin{eqnarray}
Q_{g}(\ve{z})  & = & Q_{\mbox{\tiny{in}}} (\ve{g}^{-1}(\ve{z})) \cdot \left|
  \mathrm{det} \left(\frac{\partial \ve{g}^{-1}}{\partial
      \ve{z}^\top}\right)\right|\nonumber\\
& = & Q_{\mbox{\tiny{in}}} (\ve{g}^{-1}(\ve{z})) \cdot \left|
  \mathrm{det} \left(\frac{\partial \ve{\phi}_0}{\partial
      \ve{z}^\top}\right)\right|\nonumber\\
& = & Q_{\mbox{\tiny{in}}} (\ve{g}^{-1}(\ve{z})) \cdot \left|
  \mathrm{det} \left(\prod_{\newell=1}^{L} \frac{\partial \ve{\phi}_{\newell-1}}{\partial
       \ve{\phi}_{\newell}^\top} \cdot \frac{\partial \ve{\phi}_L}{\partial
      \ve{z}^\top}\right)\right|\nonumber\\
& = & Q_{\mbox{\tiny{in}}} (\ve{g}^{-1}(\ve{z})) \cdot \left|
  \prod_{\newell=1}^{L} \mathrm{det} \left(\frac{\partial \ve{\phi}_{\newell-1}}{\partial
       \ve{\phi}_{\newell}^\top} \right) \cdot \mathrm{det} \left(\frac{\partial \ve{\phi}_L}{\partial
      \ve{z}^\top}\right)\right|\nonumber\\
 & = & Q_{\mbox{\tiny{in}}} (\ve{g}^{-1}(\ve{z})) \cdot \prod_{\newell=1}^{L}\prod_{i=1}^{d}
   |\tilde{v'}_i (\ve{v}^{-1}(\ve{\phi}_{\newell}))| \cdot \prod_{i=1}^{d}
   |\tilde{v_2'}_i (\ve{v}_2^{-1}(\ve{z}))| \cdot \left| \mathrm{det}
     \left(\Gamma \cdot \prod_{\newell=1}^{L} \matrice{w}_\newell\right)\right|^{-1}\nonumber\\
 & = & Q_{\mbox{\tiny{in}}} (\ve{g}^{-1}(\ve{z})) \cdot \frac{1}{\prod_{\newell=1}^{L}\prod_{i=1}^{d}
   |v'_i (\ve{v}^{-1}(\ve{\phi}_{\newell}))| \cdot \prod_{i=1}^{d}
   |{v_2'}_i (\ve{v}_2^{-1}(\ve{z}))}| \cdot \left| \mathrm{det}
     \left(\Gamma \cdot \prod_{\newell=1}^{L} \matrice{w}_\newell\right)\right|^{-1}\nonumber\\
 & = & Q_{\mbox{\tiny{in}}} (\ve{g}^{-1}(\ve{z})) \cdot \frac{1}{\prod_{\newell=1}^{L}\prod_{i=1}^{d}
   |v' (v^{-1}(\phi_{\newell,i}))| \cdot \prod_{i=1}^{d}
   |v_2' (v_2^{-1}(z_i))|} \cdot \left| \mathrm{det}
     \left(\Gamma \cdot \prod_{\newell=1}^{L}
       \matrice{w}_\newell\right)\right|^{-1}\nonumber\\
 & = & \frac{Q_{\mbox{\tiny{in}}} (\ve{g}^{-1}(\ve{z}))}{\prod_{\newell=1}^{L}\prod_{i=1}^{d}
   |v' (v^{-1}(\phi_{\newell,i}))|} \cdot \frac{1}{\prod_{i=1}^{d}
   |v_2' (v_2^{-1}(z_i))| \cdot \left| \mathrm{det}
     \left(\matrice{n}\right)\right|} \nonumber\:\:,
\end{eqnarray}
because $v$ and $v_2$ are coordinatewise. We have let
\begin{eqnarray}
\matrice{n} & = & \Gamma \cdot \prod_{\newell=1}^{L}
       \matrice{w}_\newell\:\:,
\end{eqnarray}
and also $\phi_{\newell,i} \defeq
v(\ve{w}^\top_{\newell, i} \ve{\phi}_{\newell-1} + b_{\newell,i})$, where
$\ve{w}_{\newell, i} \defeq \matrice{w}^\top_\newell \ve{1}_i$ is the
(column) vector built from row $i$ in $\matrice{w}_\newell$ and
similarly $z_i \defeq v_2(\ve{\gamma}^\top_{i} \ve{\phi}_{L} + \beta_{i})$
with $\ve{\gamma}_{i} \defeq \Gamma^\top \ve{1}_i$. Notice that we can
also write 
\begin{eqnarray}
\prod_{\newell=1}^{L}\prod_{i=1}^{d}
   |v' (v^{-1}(\phi_{\newell,i}))| & = & \prod_{\newell=1}^{L}\prod_{i=1}^{d}
   |v' (\ve{w}^\top_{\newell, i} \ve{\phi}_{\newell-1} + b_{\newell,i})|\:\:.\label{eqQDEEP}
\end{eqnarray}
So, letting $\tilde{Q}^*_{\mbox{\tiny{deep}}} \defeq \prod_{\newell=1}^{L}\prod_{i=1}^{d}
   |v' (\ve{w}^\top_{\newell, i} \ve{\phi}_{\newell-1} + b_{\newell,i})|$, $H_{\mbox{\tiny{out}}} \defeq
  \prod_{i=1}^{d}
   |\vpout (\ve{\gamma}^\top_{i} \ve{\phi}_{L}(\ve{x}) + \beta_{i})|$
   (with $\ve{x} \defeq \ve{g}^{-1}(\ve{z})$),
and dropping the determinant which does not depend on $\ve{z}$, we get:
\begin{eqnarray}
Q_{g}(\ve{z})  & \propto & \frac{Q_{\mbox{\tiny{in}}} (\ve{g}^{-1}(\ve{z}))}{\tilde{Q}^*_{\mbox{\tiny{deep}}}} \cdot \frac{1}{H_{\mbox{\tiny{out}}}} \:\:.\label{deflike}
\end{eqnarray}
To finish up the proof, we are going to identify
$\tilde{Q}^*_{\mbox{\tiny{deep}}}$ to (a constant times) the product of escorts in
eq. (\ref{defUUstate}). To do so, we are first going to design the
general activation function $v$ as a function of $\chi$, and choose:
\begin{eqnarray}
v (z) & \defeq & k + k' \cdot \exp_\chi(z)\:\:,\label{propDEFV1}
\end{eqnarray} 
for $k \in \mathbb{R}, k' > 0$ constants, which can be chosen \textit{e.g.} to ensure that
zero signal implies zero activation ($v(0)=0$). Our choice for $v$ has the
following key properties.
\begin{lemma}\label{lemmaVprop}
$v$ is $C^1$, invertible and we have $v'(z) = k' \cdot \chi(\exp_\chi(z))$.
\end{lemma}
\begin{proof}
The derivative comes from \citep[Eq. 84]{aomGO}. Notice that
$\exp_\chi$ is continuous as an integral, $\chi$ is continuous by
assumption and so $v'$ is continuous, implying $v$ is $C^1$. We prove
the invertibility. Because of the expression of $v'$, $v$ is
increasing, and in fact strictly
increasing with the sole exception when $\exp_\chi(z) \in
\chi^{-1}(0)$. Hovever, note that $\chi^{-1}(0) \not\subset
\mathrm{dom}(\log_\chi)$ because of the definition of
$\log_\chi$. Since $\exp_\chi$ is the inverse of $\log_\chi$
\citep[Section 10.1]{nGT}, it follows that $\chi^{-1}(0) \not\subset
\mathrm{im}(\exp_\chi)$ and so $\exp_\chi(z) \not\in
\chi^{-1}(0), \forall z \in \mathrm{dom}(\exp_\chi)$, which implies
$v$ invertible. 
\end{proof}
What the Lemma shows is that we can plug $v$ as in
eq. (\ref{propDEFV1}) directly in
$\tilde{Q}^*_{\mbox{\tiny{deep}}}$. To do so, let us now define strictly positive constants $Z_{li}$ that shall be fixed
later. We directly get from eq. (\ref{eqQDEEP})
\begin{eqnarray}
\tilde{Q}^*_{\mbox{\tiny{deep}}} & = &
\left(\prod_{\newell=1}^{L}\prod_{i=1}^{d} Z_{li}\right)\cdot  \prod_{\newell=1}^{L}\prod_{i=1}^{d}
   \frac{1}{Z_{li}} \cdot |v' (\ve{w}^\top_{\newell, i}
   \ve{\phi}_{\newell-1} + b_{\newell,i})|\nonumber\\
& = & \left(k'^{Ld} \cdot \prod_{\newell=1}^{L}\prod_{i=1}^{d} Z_{li}\right)\cdot  \underbrace{\prod_{\newell=1}^{L}\prod_{i=1}^{d}
   \frac{1}{Z_{li}} \cdot\chi(\exp_\chi(\ve{w}^\top_{\newell, i}
   \ve{\phi}_{\newell-1} + b_{\newell,i}))}_{\defeq \tilde{Q}_{\mbox{\tiny{deep}}}}\label{eqGENQDEEP}
\end{eqnarray}
(we can remove the absolute values since $\chi$ is non-negative). We
 now ensure that $\tilde{Q}_{\mbox{\tiny{deep}}} $ is indeed a product
of escorts: to do so, we just need to ensure that (i) $b_{\newell,i}$ normalizes the
deformed exponential family, \textit{i.e.} defines (negative) its
cumulant (Definition \ref{defDEF}), and (ii) $Z_{li}$ normalizes its
escort as in eq. (\ref{defNORMAL}). To be more explicit, we pick
$b_{\newell,i}$ the solution of 
\begin{eqnarray}
\int_{\ve{\phi}} \exp_\chi(\ve{w}^\top_{\newell, i} \ve{\phi} -
   b_{\newell,i}) \mathrm{d}\nu_{\newell-1} (\ve{\phi}) & = & 1 \:\:, \label{propU2}
\end{eqnarray}
where $\mathrm{d}\nu_{\newell-1} (\ve{\phi}) \defeq
\int_{\ve{\phi}_{\newell-1}(\ve{x}) = \ve{\phi}}  \mathrm{d}\mu(\ve{x}) $
is the pushforward measure, and 
\begin{eqnarray}
Z_{li} & = & \int_{\ve{x}}
\chi(P_{\chi, b_{l,i}}(\ve{x}|\ve{w}_{\newell, i} , \ve{\phi}_{l-1}))
\mathrm{d}\mu(\ve{x})\:\:.
\end{eqnarray}
We get
\begin{eqnarray}
\tilde{Q}_{\mbox{\tiny{deep}}}  = & \prod_{\newell=1}^{L}\prod_{i=1}^{d}
   \tilde{P}_{\chi, b_{\newell, i}}(\ve{x}|\ve{w}_{\newell, i},
   \ve{\phi}_{\newell-1})  \:\:,\label{defUU33}
\end{eqnarray}
and finally,
\begin{eqnarray}
Q_{g}(\ve{z})  & = &
\frac{Q_{\mbox{\tiny{in}}}(\ve{x})}{\tilde{Q}^*_{\mbox{\tiny{deep}}}(\ve{x})} \cdot
\frac{1}{H_{\mbox{\tiny{out}}}(\ve{x}) \cdot \left| \mathrm{det}
     \left(\matrice{n}\right)\right|} \nonumber\\
& = &
\frac{Q_{\mbox{\tiny{in}}}(\ve{x})}{\tilde{Q}_{\mbox{\tiny{deep}}}(\ve{x})} \cdot
\frac{1}{H_{\mbox{\tiny{out}}}(\ve{x})\cdot Z_\net}\:\:,\label{deflikeTH2}
\end{eqnarray}
with 
\begin{eqnarray}
Z_{\net} & \defeq & \left(k'^{Ld} \cdot \prod_{\newell=1}^{L}\prod_{i=1}^{d} Z_{li}\right)\cdot \left| \mathrm{det}
     \left(\matrice{n}\right)\right|
\end{eqnarray}
a constant. We get the statement of Theorem \ref{factorDEEP}.

\begin{remark}
(unnormalized densities) since in practice all $\ve{b}_\newell$s are
learned, we in fact work with deformed exponential families with
unspecified normalization. We may also consider that the normalization
of escorts is unspecified and therefore drop all $Z_{li}$s, which
simplifies $Z_\net$ to $Z_{\net} = \left| \mathrm{det}
     \left(\matrice{n}\right)\right|$.
\end{remark}

\begin{remark}
(completely factoring $Q_g$ as an escort)
Denote for short
$\ve{z}_p \defeq \ve{\phi}_L(\ve{x})$ the penultimate layer of
$\ve{g}$, and $\ve{g}_p$ the net obtain from eliminating the last
layer of $\ve{g}$, which allows us to drop $H_{\mbox{\tiny{out}}}(.)$
from $Q_{g_p}(\ve{z})$ and we have $Q_{g}(\ve{z}) \propto Q_r \defeq Q_{\mbox{\tiny{in}}}(\ve{g}_p^{-1}(\ve{z}_p)) /
\tilde{Q}_{\mbox{\tiny{deep}}}(\ve{g}_p^{-1}(\ve{z}_p))$. One can
factor $Q_r$ as a proper likelihood
over escorts of $\chi_{\mbox{{\tiny net}}}$-exponential families: for
this, replace all $Ld$ inner nodes of $\ve{g}_p$ in
Figure \ref{f-gene1} by random variables, say $\Phi_{\newell, i}$ (for
$\newell \in \{0,1, ..., L-1\}, i \in \{1, 2, ..., d\}$), treat
the deep net $\ve{g}_p$ as a directed graphical model whose connections
are the dashed arcs. Now, if we let, say, $Q_{\mbox{\tiny{in}}}(\ve{g}_p^{-1}(\ve{z}_p)) \defeq \tilde{Q}_a(\cap_{\newell,
  i} \Phi_{\newell, i})$ and
$\tilde{Q}_{\mbox{\tiny{deep}}}(\ve{g}_p^{-1}(\ve{z}_p)) \defeq \tilde{Q}_b(\cap_{\newell>0,
  i} \Phi_{\newell, i})$, and if we use as $Q_{\mbox{\tiny{in}}}$ an
uninformed escort (\textit{i.e.} with constant coordinate, say for example $\ve{\theta} = \ve{1}$,
Definition \ref{defDEF}), then assuming correct factorization one may obtain $Q_r = \tilde{Q}_{c}(\ve{g}_p^{-1}(\ve{z}_p) | \cap_{\newell>0,
  i} \Phi_{\newell, i})$ for some escort $\tilde{Q}_{c}$ that we can plug
directly in 
eq. (\ref{eqFUND1}). To properly understand the relationships between
$\chi, Q_a, Q_b$ and how the escorts factor in $Q_c$ requires a push
of the state of the art: conjugacy in deformed
exponential families is less understood than for exponential
families; it is also unknown how product of
deformed exponential families factor within the same deformed exponential
families \amari; some factorizations are known but only on
subsets of deformed exponential families and rely on
particular notions of independence \cite{mwDA};
\end{remark}

\begin{remark}
(twist introduced by the last layer) We return to the twist introduced by the last layer of $\ve{g}$:
\begin{eqnarray}
H_{\mbox{\tiny{out}}}(\ve{x})  & = & \prod_{i=1}^{d}
   |v_2' (\ve{\gamma}^\top_{i} \ve{\phi}_{L}(\ve{x}) + \beta_{L})|\:\:.\label{defVV22-2}
\end{eqnarray}
It is clear that when $v_2$ is the identity,
$H_{\mbox{\tiny{out}}}(\ve{x})$ is constant; so deep architectures, as
experimentally carried out \textit{e.g.} in Wasserstein GANs
\cite{acbWG} or analyzed theoretically \textit{e.g.} in \cite{lgmraOT}
exactly fit to the escort factoring --- notice that one can choose as
input density one from some particular deformed exponential family, as
\textit{e.g.} done experimentally for \cite[Section 2.5]{nctFG}
(standard Gaussian), so that in this case $Q_{g}(\ve{z})$ factors
completely as escorts. 

Suppose now that $v_2$ is not the identity but chosen so that, for
some couple $(\chi, g)$ where $\chi$ is differentiable and $g : \mathbb{R}_+ \rightarrow
\mathbb{R}$ is invertible,
\begin{eqnarray}
(v'_2 \circ g)(z) & = & \frac{\mathrm{d}}{\mathrm{d} z} (\log_\chi
\circ \chi) (z) = \frac{\chi'(z)}{\chi(z)}\label{eqAA}\:\:,
\end{eqnarray}
which is equivalent, after a variable change, to having $v_2$ satisfy
\begin{eqnarray}
v'_2(t) & = & \frac{\chi'\circ g^{-1}}{\chi\circ g^{-1}} (t)
\:\:.
\end{eqnarray}
In addition, suppose that $g$ is chosen so that $\sum_i g^{-1}(\ve{\gamma}^\top_{i} \ve{\phi}_{L}(\ve{x}) +
     \beta_{i}) = 1$. Call $D \defeq \{p_1, p_2, ..., p_d\}$
     this discrete distribution, removing reference to $\ve{x}$. We then have:
\begin{eqnarray}
H_{\mbox{\tiny{out}}}(\ve{x}) & = & \prod_{i=1}^{d}
   \frac{\chi' ( g^{-1}(\ve{\gamma}^\top_{i} \ve{\phi}_{L}(\ve{x}) +
     \beta_{i}))}{\chi( g^{-1}(\ve{\gamma}^\top_{i} \ve{\phi}_{L}(\ve{x}) +
     \beta_{i}))} \nonumber\\
 & = & \prod_{i=1}^{d}
   \frac{\chi' (p_i)}{\chi (p_i)}\nonumber\\
 & = & \prod_{i=1}^{d}
   \chi (p_i) \cdot \frac{\chi'
     (p_i)}{\chi (p_i)}\nonumber\\
 & = & \left|\prod_{i=1}^{d} ((\exp_\chi)' \circ \log_\chi) (p_i) \cdot
 (\log_\chi)''(p_i)\right|\nonumber\\
 & \propto & |\mathrm{det}(H)|\:\:.
\end{eqnarray}
Here, $H$ is the $\chi$-Fisher information metric of $D$ \citep[Theorem 12,
eqs 119, 120]{aomGO}. In other words, $H_{\mbox{\tiny{out}}}(\ve{x})$
can be absorbed in the volume element in eq. (\ref{deflikeTH}). 

As an example, pick a prop-$\tau$ activation (Table \ref{t-neu-synt}), for which $\log_\chi =
(\tau^\star)^{-1}(\tau^\star(0) z)$ and
\begin{eqnarray}
\chi(t) & = & \frac{(\tau^\star)'\circ (\tau^\star)^{-1} (\tau^\star(0) z)}{\tau^\star(0)}\:\:.
\end{eqnarray}
Now, pick $g(z) = \log_\chi(K \cdot z)$, where $K \defeq
\sum_i \exp_\chi (\ve{\gamma}^\top_{i} \ve{\phi}_{L}(\ve{x}) +
     \beta_{i}) $ guarantees:
\begin{eqnarray}
\sum_i g^{-1}(\ve{\gamma}^\top_{i} \ve{\phi}_{L}(\ve{x}) +
     \beta_{i}) & = & \frac{1}{K} \cdot \sum_i \exp_\chi (\ve{\gamma}^\top_{i} \ve{\phi}_{L}(\ve{x}) +
     \beta_{i})  = 1\:\:.
\end{eqnarray}
Condition in eq. (\ref{eqAA}) becomes
\begin{eqnarray}
(v'_2 \circ (\tau^\star)^{-1}) (\tau^\star(0) K z) & = &
\frac{\chi'\circ g^{-1}}{\chi\circ g^{-1}} (t)\nonumber\\
 & = & \tau^\star(0) \cdot \frac{(\tau^\star)''\circ (\tau^\star)^{-1} (\tau^\star(0) Kz)}{((\tau^\star)'\circ (\tau^\star)^{-1} (\tau^\star(0) Kz))^2}\:\:,
\end{eqnarray}
and we obtain after a variable change,
\begin{eqnarray}
v_2 & = & \tau^\star(0) \cdot \int_t
\frac{(\tau^\star)''(t)}{((\tau^\star)')^2(t)} \mathrm{d}t \:\:,\label{eqTAU1}
\end{eqnarray}
which does not depend on $K$ and, if $\tau^\star$ is \textit{strictly}
convex, is strictly increasing. Notice that we can carry out the
integration, $v_2(z) = K' - (\tau^\star(0)/(\tau^\star)'(z))$ for some
constant $K'$. To make a parallel with a popular
activation for the last layer, consider the sigmoid, $v_2 \defeq v_s(z) \defeq 1/(1+\exp(-z))$, for which
\begin{eqnarray}
v'_s(z) & = & \frac{\exp(z)}{(1+\exp(z))^2}\:\:.
\end{eqnarray}
Fitting it to eq. (\ref{eqTAU1}),
\begin{eqnarray}
\frac{\exp(z)}{(1+\exp(z))^2} & = & \tau_s^\star(0) \cdot \frac{(\tau_s^\star)''(t)}{((\tau_s^\star)')^2(t)}
\end{eqnarray}
reveals that we can pick $\tau_s^\star(z) = z + \exp(z)$ (we control
that $\tau_s^\star(0) = 1$). Such a $\tau^\star$
analytically fits to the prop-$\tau$ definition and in fact corresponds to a
$\chi$-exponential family,
but it does not
correspond to an entropy $\tau$. This would be also true for affine
scalings (argument and function) of the sigmoid of the type $v_2 = a + bv_s(c+d z)$.
\end{remark}

\section{Proof of Lemma \ref{lemACT}}\label{proof_lemACT}
Define function
\begin{eqnarray}
h(z) & \defeq & \frac{v(z) - \inf v(z)}{v(0) - \inf v(z)}\:\:,
\end{eqnarray}
and let $g(z) \defeq h^{-1}(z)$. Since $\mathrm{dom}(v) \cap
\overline{\mathbb{R}_+} \neq \emptyset$, $v(0) - \inf v(z) > 0$, so
$h(z)$ bears the same properties as $v$. We first show that $g$ is a
valid $\chi$-logarithm. Since $v$ is convex increasing, $g(z)$ is concave increasing and $-g$ is convex
decreasing. Therefore, since $g$ is $C^1$ as well, letting $\xi \defeq
g'$, we get:
\begin{eqnarray}
g(z) & = & \int_{1}^z \frac{1}{\left(\frac{1}{\xi(t)}\right)} \mathrm{d}t\:\:.
\end{eqnarray}
We also check that $g(1) = 0$ since $h(0) = 1$. 
If we let $\chi \defeq 1/\xi$, then because $\xi(z) \geq 0$, $\chi(z)
\geq 0$ and also because $\xi$ is decreasing, $\chi$ is
increasing. Finally, $\chi : \mathbb{R}_+ \rightarrow \mathbb{R}_+$.
Summarizing, we have shown that $\chi$ defines a valid signature and
$g(z) = \log_\chi(z)$. Therefore, $h(z) = \exp_\chi(z)$ and it comes
that 
\begin{eqnarray}
v(z) & = &  k + k' \cdot \exp_\chi(z)\:\:,\label{ppp33}
\end{eqnarray}
for $k \defeq \inf v(z) \in \mathbb{R}$ and $k' \defeq v(0) - \inf
v(z) > 0$, so $v$ matches the analytic expression in
eq. (\ref{propDEFV1}), which allows to complete the proof of the Lemma.

\section{Proof of Lemma \ref{lemACT_ReLU}}\label{proof_lemACT_ReLU}
We use a scaled perspective transform of the Softplus function and let:
\begin{eqnarray}
v_{\mu}(z) & \defeq & (1-\mu) \cdot \log\left(1+\exp\left(\frac{z}{1-\mu}\right)\right)\:\:,
\end{eqnarray}
with $\mu \in [0,1]$. It is clear that $v_{\mu}$ is strongly
admissible for any $\mu \in [0,1)$.
\begin{lemma}\label{lemDIFF1}
For any $z\geq 0, \mu \in [0,1]$,
\begin{eqnarray}
(1-\mu)\cdot \log \left( \frac{1+\exp\left(\frac{z}{1-\mu}\right)}{1+\exp(z)} \right) & \leq & \mu z\:\:.
\end{eqnarray}
\end{lemma}
\begin{proof}
Equivalently, we want
\begin{eqnarray}
\frac{1+\exp\left(\frac{z}{1-\mu}\right)}{1+\exp(z)} & \leq & \exp\left(\frac{\mu z}{1-\mu}\right)\:\:,
\end{eqnarray}
or, equivalently,
\begin{eqnarray}
1+\exp\left(\frac{z}{1-\mu}\right) & \leq & \exp\left(\frac{\mu
    z}{1-\mu}\right) + \exp(z)\cdot \exp\left(\frac{\mu
    z}{1-\mu}\right)\nonumber\\
 & & = \exp\left(\frac{\mu
    z}{1-\mu}\right) + \exp\left(\frac{
    z}{1-\mu}\right) \:\:,
\end{eqnarray}
which, after simplification, is equivalent to $\mu z/ (1-\mu) \geq 0$,
which indeed holds when $z\geq 0, \mu \in [0,1]$.
\end{proof}
We now have $v_\mu(z) \geq \max\{0, z\}, \forall \mu \in [0,1]$, and
we can also check that Lemma \ref{lemDIFF1} implies
\begin{eqnarray}
\lefteqn{(1-\mu)\cdot \log \left( 1+\exp\left(\frac{z}{1-\mu} \right) \right) -
z}\nonumber\\
 & \leq & (1-\mu)\cdot \left( \log \left( 1+\exp (z)\right) -
z\right)\:\:, \forall z\geq 0, \mu \in [0,1]\:\:.\label{ineqF1}
\end{eqnarray}
Let us denote, for any $z\geq 0$,
\begin{eqnarray}
I_\mu(z) & \defeq & \int_0^z |v_{\mu}(t) - \max\{0, t\}|
\mathrm{d}t\nonumber\\
 & = & \int_0^z |v_{\mu}(t) - t|
\mathrm{d}t\nonumber\\
 & = & \int_0^z (v_{\mu}(t) - t)
\mathrm{d}t\:\:.
\end{eqnarray}
Since $\max\{0, -t\} = \max\{0, t\} - t$ and $v_{\mu}(-t) = v_{\mu}(t)
- t$, we have $\|v_{\mu} - \mathrm{ReLU}\|_{L1} = 2 \lim_{z\rightarrow
+\infty} I_\mu(z)$. It also comes from ineq. (\ref{ineqF1}) that
\begin{eqnarray}
I_\mu(z) & \leq & (1-\mu) I_0(z) \:\:, \forall z\leq 0\:\:,
\end{eqnarray}
furthermore, it can be shown by numerical integration that
$\lim_{+\infty}  I_0(z) = \pi^2/6$, so we get
\begin{eqnarray}
\|v_{\mu} - \mathrm{ReLU}\|_{L1} & \leq &
\frac{(1-\mu)\pi^2}{3}\:\:,\forall \mu \in [0,1]\:\:,
\end{eqnarray}
and to have the right hand side smaller than $\epsilon > 0$, it
suffices to take
\begin{eqnarray}
\mu & > & 1 - \frac{3\epsilon}{\pi^2}\:\:,
\end{eqnarray}
which yields the statement of the Lemma.

\section{Proof of Theorem \ref{thmKLQQ}}\label{proof_thmKLQQ}

We split the proof of the Theorem in several Lemmata.
\begin{lemma}\label{lemmINEQKL}
Suppose $f$ satisfies Corollary \ref{corGEN}, and let 
\begin{eqnarray}
\chi (t) & \defeq & \frac{1}{-\xi(t) + k} \:\:,\label{defchi1}
\end{eqnarray}
where $k \geq \sup_{\mathbb{R}_+} \xi$. Let $\mathbb{R}_+ \supseteq \mathcal{Q} \defeq \{Q :
f(\tilde{Q}) < f (Q)\}$ and $\mathbb{R}_+ \supseteq \mathcal{Q}' \defeq \{Q :
\tilde{Q} < Q\}$. Suppose the following property (A) holds: there exists $g : \mathbb{R}_+ \rightarrow
\mathbb{R}_+$ such that
\begin{eqnarray}
f(Q(\ve{x}))
 -f(\tilde{Q}(\ve{x})) & \leq &
 g\left(\frac{Q(\ve{x})}{\tilde{Q}(\ve{x})}\right) \:\:, \forall
 \ve{x}: Q(\ve{x}) \in \mathcal{Q}\:\:.\label{propG}
\end{eqnarray}
Then,
\begin{eqnarray}
KL_{\chi_{\tilde{{Q}}}}(\tilde{\dist{Q}}\|\dist{Q}) & \leq & (-k) \cdot \sup \mathcal{Q}'+ \int_{\ve{x}: Q(\ve{x}) \in \mathcal{Q}}
\tilde{Q}(\ve{x})
g\left(\frac{Q(\ve{x})}{\tilde{Q}(\ve{x})}\right)\mathrm{d}\mu(\ve{x})\label{INEQKL1}\:\:.
\end{eqnarray}
\end{lemma}
\begin{proof}
It follows from the definition of $KL_{\chi}$ that:
\begin{eqnarray}
KL_{\chi_{\tilde{{Q}}}}(\tilde{\dist{Q}}\|\dist{Q}) & \defeq & \expect_{\X\sim
  \tilde{\dist{Q}}}\left[-\log_{\chi_{\tilde{{Q}}}}\left(\frac{Q(\X)}{\tilde{Q}(\X)}\right)\right]\nonumber\\
 & = & \int_{\ve{x}} \tilde{Q}(\ve{x})
 \int_{Q(\ve{x})}^{\tilde{Q}(\ve{x})}
 \frac{1}{\chi
(t)} \mathrm{d}t\mathrm{d}\mu(\ve{x}) \label{propKLCHI}\\
 & = &  \int_{\ve{x}} \tilde{Q}(\ve{x})
 \left[-f(z) + kz\right]_{Q(\ve{x})}^{\tilde{Q}(\ve{x})}
 \mathrm{d}\mu(\ve{x}) \nonumber\\
 & = & \int_{\ve{x}} \tilde{Q}(\ve{x}) (f(Q(\ve{x}))
 -f(\tilde{Q}(\ve{x})) - k\cdot (Q(\ve{x}) -
 \tilde{Q}(\ve{x}))) \mathrm{d}\mu(\ve{x}) \nonumber\\
 & = & A(f(\tilde{Q}) \leq f (Q)) +
 A(f(\tilde{Q})> f(Q)) \nonumber\\
 & & + (-k)
 \cdot \int_{\ve{x}} \tilde{Q}(\ve{x}) (Q(\ve{x}) -
 \tilde{Q}(\ve{x})) \mathrm{d}\mu(\ve{x})\label{defEQ1}\:\:.
\end{eqnarray}
where, for any predicate $\pi : \mathcal{X} \rightarrow
\{\texttt{false}, \texttt{true}\}$,
\begin{eqnarray}
A(\pi) & \defeq & \int_{\ve{x} : \pi(\ve{x}) = \texttt{true}} \tilde{Q}(\ve{x}) (f(Q(\ve{x}))
 -f(\tilde{Q}(\ve{x})) \mathrm{d}\mu(\ve{x})\:\:.
\end{eqnarray}
Let $\mathbb{R}_+ \supseteq \mathcal{Q} \defeq \{Q :
f(\tilde{Q}) < f (Q)\}$ and $\mathbb{R}_+ \supseteq \mathcal{Q}' \defeq \{Q :
\tilde{Q} < Q\}$. Remark that $A(f(\tilde{Q})>
f(Q)) \leq 0$ and
\begin{eqnarray}
f(Q(\ve{x}))
 -f(\tilde{Q}(\ve{x})) & \leq &
 g\left(\frac{Q(\ve{x})}{\tilde{Q}(\ve{x})}\right) \:\:, \forall
 \ve{x}: Q(\ve{x}) \in \mathcal{Q}\:\:,
\end{eqnarray}
from property (A) and, densities being non-negative, 
\begin{eqnarray}
\int_{\ve{x}} \tilde{Q}(\ve{x}) (Q(\ve{x}) -
 \tilde{Q}(\ve{x})) \mathrm{d}\mu(\ve{x}) & \leq & \int_{\ve{x} :
   Q(\ve{x}) \in \mathcal{Q}'} \tilde{Q}(\ve{x}) (Q(\ve{x}) -
 \tilde{Q}(\ve{x})) \mathrm{d}\mu(\ve{x})\nonumber\\
 & \leq & \int_{\ve{x} :
   Q(\ve{x}) \in \mathcal{Q}'} Q^2(\ve{x}) \mathrm{d}\mu(\ve{x}) \label{eq1001}
 \\
 & \leq & (\sup \mathcal{Q}')\cdot  \int_{\ve{x} :
   Q(\ve{x}) \in \mathcal{Q}'} Q (\ve{x}) \mathrm{d}\mu(\ve{x}) \leq \sup \mathcal{Q}'\:\:,
\end{eqnarray}
where eq. (\ref{eq1001}) follows from the definition of $\mathcal{Q}'$. Putting
this altogether, we get
\begin{eqnarray}
KL_{\chi_{\tilde{{Q}}}}(\tilde{\dist{Q}}\|\dist{Q}) & \leq & (-k) \cdot \sup \mathcal{Q}'+ \int_{\ve{x}: Q(\ve{x}) \in \mathcal{Q}}
\tilde{Q}(\ve{x})
g\left(\frac{Q(\ve{x})}{\tilde{Q}(\ve{x})}\right)\mathrm{d}\mu(\ve{x})\:\:,\label{defEQ11}
\end{eqnarray}
as claimed.
\end{proof}

We now check that Lemma \ref{lemmINEQKL} is optimal in the sense that
we recover $KL_{\chi_{\tilde{{Q}}}}(\tilde{\dist{Q}}\|\dist{Q}) = 0$ for all
exponential families.
\begin{lemma}
Suppose $\dist{Q}$ is en exponential family. Then the bound in
eq. (\ref{INEQKL1}) is zero.
\end{lemma} 
\begin{proof}
The KL divergence admits the following form, for $h(z) \defeq z\log z$:
\begin{eqnarray}
I_{\textsc{kl}} (P\|\dist{Q}) & = & \expect_{\X\sim \dist{Q}}
\left[h\left(\frac{P(\X)}{Q(\X)}\right)\right]\nonumber\\
 & = & \expect_{\X\sim P}\left[-\log_{\chi}\left(\frac{Q(\X)}{P(\X)}\right)\right]\:\:,
\end{eqnarray}
with $\chi(z) = z$ (and $f(z) = -\log z$, yielding $k=0$ in Lemma \ref{lemmINEQKL}), $\mathcal{Q} = \mathbb{R}_+$, $\mathcal{Q}' = \{0\}$. We also have 
\begin{eqnarray}
f(u) - f(v) & = & - \log \frac{u}{v} \:\:,
\end{eqnarray}
and so we can pick $g(z) \defeq -\log(z)$ for property (A). After remarking
that $Z=1$, and using the convenient choice $k=0$, we get from ineq. (\ref{defEQ11}),
\begin{eqnarray}
KL_{\chi_{\tilde{{Q}}}}(\tilde{\dist{Q}}\|\dist{Q}) & \leq & 0\cdot 0 + \int_{\ve{x}: Q(\ve{x})\geq 0}
\tilde{Q}(\ve{x})
g\left(\frac{Q(\ve{x})}{\tilde{Q}(\ve{x})}\right)\mathrm{d}\mu(\ve{x})\nonumber\\
 & \leq & 0 + \sup_{Q(\ve{x})\geq 0} (-\log) \left( 1 \right)
 \nonumber\\
 & = & 0\:\:,\label{defEQ11EXP}
\end{eqnarray}
as claimed. 
\end{proof}
We now treat all cases of Theorem \ref{thmKLQQ}, starting with point (i).
\begin{lemma}
For the original GAN choice of $f$, $Z > 1$ and $J(\dist{Q}) \leq (1/Z) \cdot
\textsc{m}\left(Q(.) < 1/(Z-1)\right)$.
\end{lemma}
\begin{proof}
In
this case, we choose
\begin{eqnarray}
f(z) = f_{\mbox{\tiny{\textsc{gan}}}}(z) & \defeq & z\log z -
(1+z)\log (1+z) + 2\log 2\:\:.
\end{eqnarray}
We also remark that for any $0\leq u \leq v$,
\begin{eqnarray}
f_{\mbox{\tiny{\textsc{gan}}}}(u) - f_{\mbox{\tiny{\textsc{gan}}}}(v)
& \leq & - \log \frac{u}{v} \:\:,\label{eqbound1}
\end{eqnarray}
We can show this by analyzing function
$f_{\mbox{\tiny{\textsc{gan}}}}(\varepsilon z) -
f_{\mbox{\tiny{\textsc{gan}}}}(z)$ for any fixed $\varepsilon \in [0,1]$, which is increasing on $z \in \mathbb{R}_+$ and
converges to $- \log (\varepsilon)$. So we can pick $g(z) \defeq
-\log(z)$ for assumption (A) and since
$f_{\mbox{\tiny{\textsc{gan}}}}$ is strictly decreasing, $\mathcal{Q}
=\{Q : \tilde{Q} > Q\}$. Using $k=0$ in Lemma \ref{lemmINEQKL}, we obtain
\begin{eqnarray}
\chi(z) & = & -\frac{1}{f'_{\mbox{\tiny{\textsc{gan}}}}(z)} =
\frac{1}{\log\left( 1+ \frac{1}{z}\right)}\:\:.\label{eqchif222}
\end{eqnarray}
We have $\chi(z) > z, \forall z>0$, so $Z > 1$. We also obtain
\begin{eqnarray}
\sup_{\mathcal{Q}} g \left( Z \cdot \frac{Q(\ve{x})}{\chi(Q(\ve{x}))}
\right) & = & \sup_{\mathcal{Q}} \log \left( \frac{1}{Z} \cdot
  \frac{1}{r(z)} \right)\:\:,\\
r(z) &\defeq & z \cdot \log\left( 1+ \frac{1}{z}\right)\:\:.
\end{eqnarray}
Because $\chi(z)$ is strictly increasing and satisfies $\chi(z)
\in [z, z + 1/2)$ with $\lim_0 \chi(z) = 0$, we have $Z>1$ and 
\begin{eqnarray}
\int_{\ve{x}: Q(\ve{x}) \in \mathcal{Q}}
\tilde{Q}(\ve{x})
g\left(\frac{Q(\ve{x})}{\tilde{Q}(\ve{x})}\right)\mathrm{d}\mu(\ve{x})
& = & \int_{\ve{x}: Q(\ve{x}) \in \mathcal{Q}}
\tilde{Q}(\ve{x})
\log \left(\frac{\tilde{Q}(\ve{x})}{Q(\ve{x})}\right)\mathrm{d}\mu(\ve{x}) \nonumber\\
& = & \int_{\ve{x}: Q(\ve{x}) \in \mathcal{Q}}
\frac{\chi(Q(\ve{x}))}{Z}
\log \left(\frac {\chi(Q(\ve{x}))}{Z
    Q(\ve{x})}\right)\mathrm{d}\mu(\ve{x}) \nonumber\\
 & = & \int_{\ve{x}: Q(\ve{x}) \in \mathcal{Q}}
\frac{\chi(Q(\ve{x}))}{Z}
\log \left(\frac{1}{Z}\right)\mathrm{d}\mu(\ve{x}) \nonumber\\
 & & +\int_{\ve{x}: Q(\ve{x}) \in \mathcal{Q}}
\frac{\chi(Q(\ve{x}))}{Z}
\log \left(\frac {\chi(Q(\ve{x}))}{
    Q(\ve{x})}\right)\mathrm{d}\mu(\ve{x})  \nonumber\\
 & \leq & \log \left(\frac{1}{Z}\right)+ \frac{1}{Z} \cdot \int_{\ve{x}: Q(\ve{x}) \in \mathcal{Q}}
s(Q(\ve{x})) \mathrm{d}\mu(\ve{x})  \label{llEQ}\\
 & \leq & \frac{1}{Z} \cdot \int_{\ve{x}: Q(\ve{x}) \in \mathcal{Q}}
s(Q(\ve{x})) \mathrm{d}\mu(\ve{x}) \:\:,
\end{eqnarray}
with 
\begin{eqnarray}
s(z) & \defeq & \frac{1}{\log\left( 1+ \frac{1}{z}\right)} \cdot \log
\left(\frac{1}{z \log\left( 1+ \frac{1}{z}\right)}\right) \in
\left[\frac{1}{2}, 1\right]\:\:.
\end{eqnarray}
In eq. (\ref{llEQ}), we have exploited the choice of $\chi$ in eq. (\ref{eqchif222}). We remark that
\begin{eqnarray}
\frac{\chi(Q)}{Q} & = & \frac{1}{h(Q)}\:\:,
\end{eqnarray}
with 
\begin{eqnarray}
h(z) & \defeq & z \cdot \log \left(1+ \frac{1}{z}\right) \:\:,
\end{eqnarray}
which is strictly increasing on $\mathbb{R}_+$, satisfies $\mathrm{im}
h = [0,1)$, and so $\tilde{Q}/Q$ is a strictly decreasing function of
$Q$ and $\sup \mathcal{Q}$ is strictly smaller than the solution of $h(z) =
1/Z$ (equivalently, $Q = \tilde{Q}$), call it $q(Z)$. We get, since
$s(z) \leq 1$,
\begin{eqnarray}
\int_{\ve{x}: Q(\ve{x}) \in \mathcal{Q}}
s(Q(\ve{x})) \mathrm{d}\mu(\ve{x}) & \leq & \textsc{m}(Q < q(Z))\:\:,
\end{eqnarray}
where $\textsc{m}(Q < z) \defeq \int_{\ve{x}: Q(\ve{x}) \leq z}
\mathrm{d}\mu(\ve{x})$ is the total measure of the support with
"small" density (\textit{i.e.} upperbounded by $q(Z)$). We get
\begin{eqnarray}
KL_{\chi_{\tilde{{Q}}}}(\tilde{\dist{Q}}\|\dist{Q}) & \leq & \frac{\textsc{m}(Q < q(Z))}{Z}\:\:,
\end{eqnarray}
and to make this bound further readable, it can be shown that $h(z)
\geq z/(1+z)$ and so $h^{-1}(z) \leq z / (1-z)$ for $z \in [0,1)$. It
follows
\begin{eqnarray}
q(Z) & \defeq & h^{-1}\left(\frac{1}{Z}\right) \leq \frac{1}{Z-1}\:\:,
\end{eqnarray}
and so $\textsc{m}(Q < q(Z)) \leq \textsc{m}(Q < 1/(Z-1))$, and we obtain:
\begin{eqnarray}
J_{\mbox{\tiny{\textsc{gan}}}}(\dist{Q}) = KL_{\chi_{\tilde{{Q}}}}(\tilde{\dist{Q}}\|\dist{Q}) & \leq & \frac{1}{Z} \cdot
\textsc{m}\left(Q < \frac{1}{Z-1}\right)\:\:,
\end{eqnarray}
as claimed.
\end{proof}
We now treat point (ii) in Theorem \ref{thmKLQQ}.
\begin{lemma}\label{lemMUReLU}
Consider the $\mu$-ReLU choice for which
\begin{eqnarray}
\chi(z) & =& \frac{4z^2}{(1-\mu)^2+4z^2}
\end{eqnarray}
with $\mu \in [0,1)$. Then the associated normalization constant of
the escort, $Z$, satisfies
\begin{eqnarray}
Z & \leq & \frac{1}{1-\mu}\:\:,
\end{eqnarray}
and penalty $J(\dist{Q})$ satisfies:
\begin{eqnarray}
J(\dist{Q}) & \leq & \frac{1}{Z} \cdot \left( 1 + \frac{3\sqrt{3}}{8Z(1-\mu)} \right) \:\:.
\end{eqnarray}
\end{lemma}
\begin{proof}
We first remark that
\begin{eqnarray}
\max_{\mathbb{R}_+} \frac{4z}{(1-\mu)^2 + 4z^2} & = & \frac{1}{1-\mu}\:\:,
\end{eqnarray}
from which we derive
\begin{eqnarray}
Z & \defeq & \int_{\ve{x}}
\chi(Q(\ve{x}))
\mathrm{d}\mu(\ve{x}) \nonumber\\
 & = &  \int_{\ve{x}}
\frac{4 Q^2(\ve{x})}{(1-\mu)^2 + 4 Q^2(\ve{x})}
\mathrm{d}\mu(\ve{x}) \nonumber\\
 & = &  \int_{\ve{x}}
\frac{4 Q(\ve{x})}{(1-\mu)^2 + 4 Q^2(\ve{x})} \cdot Q(\ve{x})
\mathrm{d}\mu(\ve{x}) \nonumber\\
 & \leq &  \frac{1}{1-\mu} \cdot \int_{\ve{x}} Q(\ve{x})
\mathrm{d}\mu(\ve{x}) \nonumber\\
 & &  = \frac{1}{1-\mu} \:\:.
\end{eqnarray}
Then, we remark that 
\begin{eqnarray}
\log_\chi(z) & \defeq & \int_1^{z} \frac{\mathrm{d}t}{\chi(t)} = z -
\frac{(1-\mu)^2}{4z} - K\:\:,
\end{eqnarray}
with $K \defeq 1 - (1-\mu)^2/4$ and it comes from eq. (\ref{propKLCHI}) that
\begin{eqnarray}
KL_{\chi_{\tilde{{Q}}}}(\tilde{\dist{Q}}\|\dist{Q}) & = & \int_{\ve{x}} \tilde{Q}(\ve{x})
 \int_{Q(\ve{x})}^{\tilde{Q}(\ve{x})}
 \frac{1}{\chi
(t)} \mathrm{d}t\mathrm{d}\mu(\ve{x}) \nonumber\\
 & = &  \int_{\ve{x}} g(Q(\ve{x}))
 \mathrm{d}\mu(\ve{x}) \nonumber\:\:,
\end{eqnarray}
with
\begin{eqnarray}
g(z) & \defeq & \frac{1}{Z}\cdot \frac{4 z^2}{(1-\mu)^2
   + 4 z^2}\nonumber\\
 & &\cdot \left( 
\frac{1}{Z}\cdot \frac{4 z^2}{(1-\mu)^2
   + 4 z^2} - \frac{Z(1-\mu)^2\cdot ((1-\mu)^2
   + 4 z^2)}{16 z^2}
-z +
\frac{(1-\mu)^2}{4z}
\right)\nonumber\\
 & \leq & \frac{1}{Z}\cdot \frac{4 z^2}{(1-\mu)^2
   + 4 z^2} \cdot \left( 
\frac{1}{Z}\cdot \frac{4 z^2}{(1-\mu)^2
   + 4 z^2} +
\frac{(1-\mu)^2}{4z}
\right)\nonumber\\
 & & = z \cdot \left( \frac{1}{Z^2}\cdot \frac{16 z^3}{((1-\mu)^2
   + 4 z^2)^2} + \frac{1}{Z}\cdot \frac{(1-\mu)^2}{(1-\mu)^2
   + 4 z^2}\right)\:\:.
\end{eqnarray}
We then remark that
\begin{eqnarray}
\max_{\mathbb{R}_+} \frac{16 z^3}{((1-\mu)^2
   + 4 z^2)^2} & = & \frac{3\sqrt{3}}{8(1-\mu)}\:\:,
\end{eqnarray}
so that
\begin{eqnarray}
g(z) & \leq & z \cdot \left(\frac{3\sqrt{3}}{8Z^2(1-\mu)} +
  \frac{1}{Z}\cdot \frac{(1-\mu)^2}{(1-\mu)^2
   + 4 z^2}\right)\nonumber\\
 & \leq & z \cdot \left(\frac{3\sqrt{3}}{8Z^2(1-\mu)} + \frac{1}{Z}\right)\:\:,
\end{eqnarray}
and finally
\begin{eqnarray}
J(\dist{Q})  \defeq KL_{\chi_{\tilde{{Q}}}}(\tilde{\dist{Q}}\|\dist{Q}) & \leq &
\left(\frac{3\sqrt{3}}{8Z^2(1-\mu)} + \frac{1}{Z}\right) \cdot \int_{\ve{x}} Q(\ve{x})
 \mathrm{d}\mu(\ve{x}) = \frac{3\sqrt{3}}{8Z^2(1-\mu)} + \frac{1}{Z} \nonumber\:\:,
\end{eqnarray}
as claimed.
\end{proof}
We complete point (ii) by remarking that $3\sqrt{3}/8 \approx 0.65 <
1$. We now treat point (iii) in Theorem \ref{thmKLQQ}. We are going to
show a more complete statement. 
\begin{table}[t]
\begin{center}
{
\begin{tabular}{c|cccc}\hline\hline
$Q$ & $[0,Z)$ & $[Z,\gamma Z)$ & $[\gamma Z, \gamma)$ & $[\gamma,+\infty)$ \\ 
$\tilde{Q}$ & $\frac{Q}{Z}$ & $\frac{Q}{Z}$ & $\frac{Q}{Z}$ & $\frac{\gamma}{Z}$\\ \hdashline
$\tilde{Q}\log_\chi \tilde{Q}$ & $\frac{Q}{Z}\log \frac{Q}{Z}$&$\frac{Q}{Z}\log \frac{Q}{Z}$&
$\frac{Q}{Z}\left(\log\gamma + \frac{Q}{\gamma Z}-1\right)$ &
$\frac{\gamma}{Z}\left(\log\gamma + \frac{1}{Z}-1\right)$\\
$\tilde{Q}\log_\chi \tilde{Q} \leq $ &
\multicolumn{4}{c}{$\frac{Q}{Z}\left(\log\gamma + \frac{1}{Z}-1\right)$}\\\hdashline
$-\tilde{Q}\log_\chi Q$ & $-\frac{Q}{Z}\log Q$&$-\frac{Q}{Z}\log Q$&
$-\frac{Q}{Z}\log Q$ & $-\frac{\gamma}{Z}\left( \log\gamma +
  \frac{Q}{\gamma} - 1\right)$\\
$-\tilde{Q}\log_\chi \tilde{Q} \leq $ &
\multicolumn{4}{c}{$\max\left\{0, -\frac{Q}{Z}\log Q\right\}$}\\
\hline\hline
\end{tabular}
}
\end{center}
\caption{Bounds on $\tilde{Q}\log_\chi \tilde{Q}$ and
  $-\tilde{Q}\log_\chi \tilde{Q} $ as a function of $Q$, as used for
  the proof of Lemma \ref{lemELU}, using the fact that $\gamma\geq 1$.\label{t-proof-ELU}}
\end{table}

\begin{lemma}\label{lemELU}
Consider the $(\alpha, \beta)$-ELU choice for which
\begin{eqnarray}
\chi(z) & =& \left\{\hspace{-0.2cm} \begin{array}{rcl} 
\beta & \hspace{-0.3cm}\mbox{ if }\hspace{-0.3cm} & z > \alpha \\
z  & \hspace{-0.3cm}\mbox{ if }\hspace{-0.3cm} & z\leq \alpha \\
\end{array}\right.\:\:.
\end{eqnarray}
 Then the associated normalization constant of
the escort, $Z$, satisfies 
\begin{eqnarray}
Z & \leq & \frac{\beta}{\alpha}\:\:,
\end{eqnarray}
so that for the choice $\beta=\alpha\defeq \gamma$, we have $Z\leq 1$. Furthermore,
whenever $\gamma\geq 1$, penalty $J(\dist{Q})$ satisfies:
\begin{eqnarray}
J(\dist{Q}) & \leq & \frac{\log\gamma }{Z} + \frac{1-Z}{Z^2} + \frac{1}{Z}\cdot H_*(\dist{Q})\:\:,\label{ELUZ}
\end{eqnarray}
where $H_{*}(\dist{Q}) \defeq \expect_{\X \sim \dist{Q}}[\max\{0,-\log Q(\X)\}]$. This bound is tight.
\end{lemma}
\begin{proof}
We obtain directly $\chi(z) \leq (\beta/\alpha) \cdot z$, from which 
\begin{eqnarray}
Z & \defeq & \int_{\ve{x}}
\chi(Q(\ve{x}))
\mathrm{d}\mu(\ve{x}) \nonumber\\
 & \leq &  \frac{\beta}{\alpha} \cdot \int_{\ve{x}}
Q(\ve{x})
\mathrm{d}\mu(\ve{x}) = \frac{\beta}{\alpha} \:\:.
\end{eqnarray}
Then, we remark that if $\beta=\alpha\defeq \gamma \geq 1$, $Z\leq 1$ and
\begin{eqnarray}
\log_\chi(z) & \defeq & \int_1^{z} \frac{\mathrm{d}t}{\chi(t)}
=\left\{
\begin{array}{rcl}
\log \gamma + \frac{z}{\gamma} - 1 & \mbox{ if } & z > \gamma\\
\log z & \mbox{ if } & z \leq \gamma
\end{array}\right. \:\:,
\end{eqnarray}
and we finally obtain,
\begin{eqnarray}
KL_{\chi_{\tilde{{Q}}}}(\tilde{\dist{Q}}\|\dist{Q}) & = & \int_{\ve{x}} \tilde{Q}(\ve{x})
 \int_{Q(\ve{x})}^{\tilde{Q}(\ve{x})}
 \frac{1}{\chi
(t)} \mathrm{d}t\mathrm{d}\mu(\ve{x}) \nonumber\\
 & = &  \int_{\ve{x}} \tilde{Q}(\ve{x}) \log_\chi \tilde{Q}(\ve{x})
 \mathrm{d}\mu(\ve{x}) -\int_{\ve{x}} \tilde{Q}(\ve{x}) \log_\chi Q(\ve{x})
 \mathrm{d}\mu(\ve{x}) \nonumber\\
 & \leq & \left( \frac{\log\gamma }{Z} + \frac{1-Z}{Z^2}\right)\cdot \int_{\ve{x}} Q(\ve{x}) 
 \mathrm{d}\mu(\ve{x}) + \frac{1}{Z}\cdot H_*(\dist{Q}) \nonumber\\
 & & =  \frac{\log\gamma }{Z} + \frac{1-Z}{Z^2}+ \frac{1}{Z}\cdot H_{*}(\dist{Q})\:\:,
\end{eqnarray}
with $H_{*}(\dist{Q}) \defeq \expect_{\X \sim \dist{Q}}[\max\{0,-\log Q(\X)\}]$ is a
clipping of Shannon's
entropy (which prevents it from being negative). The inequality follows from bounding the two terms in
the integral depending on the value of $Q$ following Table
\ref{t-proof-ELU}.

For tightness, consider the "square" uniform distribution with support an
interval $[a, a+1]$ with $a \leq 0$, and fix $\gamma = 1$, which
brings $J(\dist{Q}) = 0$ and $\log\gamma = 0$, $Z=1$
and $H_*(\dist{Q}) = 0$, so both bounds in eq. (\ref{ELUZ}) match.
\end{proof}

We end up this Section with three additional results related to
Theorem \ref{thmKLQQ}:
\begin{itemize}
\item [(iv)] bounding $J(\dist{Q})$ when $\chi$ is the signature of $q$-exponential
  families, also displaying that $J(\dist{Q}) = O(1/Z)$;
\item [(v)] computing exactly $J(\dist{Q})$ for a particular $\chi$-family and a
  member of the $\chi$-family for $\dist{Q}$, displaying that $J(\dist{Q}) =
  \theta(1/Z)$;
\item [(vi)] showing how a particular choice for $\chi$ that blows up
  large density regions for some $\dist{Q}$ can yield $Z$
  arbitrarily large.
\end{itemize}
We focus now on (v) and pick $\chi$
as the signature of popular deformed exponential families, the
$q$-exponential families \cite{aIG}.
\begin{lemma}\label{lemCHIQ}
Consider $\chi(z) = \chi_q(z) \defeq z^q$ for $q>1$. Then for any $\dist{Q}$,
\begin{eqnarray}
J_q(\dist{Q})  & \leq & \frac{1}{(q-1)Z}\:\:.
\end{eqnarray}
\end{lemma}
\begin{proof}
We get directly
\begin{eqnarray}
J_q(\dist{Q}) \defeq KL_{\chi_{\tilde{{Q}}}}(\tilde{\dist{Q}}\|\dist{Q}) & = & \int_{\ve{x}} \tilde{Q}(\ve{x})
 \int_{Q(\ve{x})}^{\tilde{Q}(\ve{x})}
 t^{-q} \mathrm{d}t\mathrm{d}\mu(\ve{x}) \nonumber\\
 & = & \frac{1}{1-q} \cdot \int_{\ve{x}} (\tilde{Q}^{2-q}(\ve{x}) - \tilde{Q}(\ve{x}) Q^{1-q}(\ve{x}))
 \mathrm{d}\mu(\ve{x}) \nonumber\\
 & = & \frac{1}{1-q} \cdot \int_{\ve{x}} \left(\frac{1}{Z^{2-q}} \cdot Q^{q(2-q)}(\ve{x})
 - \frac{1}{Z} \cdot Q(\ve{x})\right)
 \mathrm{d}\mu(\ve{x}) \nonumber\\
 & = & \frac{1}{q-1} \cdot \left(\frac{1}{Z} - \frac{1}{Z^{2-q}} \cdot \int_{\ve{x}} Q^{q(2-q)}(\ve{x})
 \mathrm{d}\mu(\ve{x}) \right)\label{eq001}\\
 & \leq & \frac{1}{(q-1)Z}
\end{eqnarray}
since $q > 1$.
\end{proof}
We continue with (v) and pick a particular case for which $\dist{Q}$ belongs to the
$\chi$-family, with an exact computation of $J(\dist{Q})$. We choose the
$1/2$-Gaussian.
\begin{lemma}\label{lemKLQQGAUSS}
Consider the $1/2$-Gaussian on the real interval $[-\sigma, \sigma]$, for some
$\sigma > 0$, whose density is given by 
\begin{eqnarray}
Q(x) & \defeq & \frac{A}{\sigma}\cdot [1-(x^2/\sigma^2)]_+^2\:\:,
\end{eqnarray}
with $A \defeq 15\sqrt{2}/32$ and $[z]_+ \defeq \max\{0,z\}$. Then,
for $\chi$ being the one of the $1/2$-Gaussian, we have:
\begin{eqnarray}
J(\dist{Q}) & = & \frac{3^{\frac{3}{2}}}{2\sqrt{\sigma}}\cdot \left( \frac{
     3 \pi}{16}
   - \frac{1}{\sqrt{15\sqrt{2}}} \right)   = \theta\left(\frac{1}{Z}\right)  \:\:.
\end{eqnarray}
\end{lemma}
\begin{proof}
The $1/2$-Gaussian arises from the more general class of
$q$-exponential families, for which $\chi$ is given in Lemma \ref{lemCHIQ} \cite{aomGO}, \citep[Chapter
7]{nGT}. 
We start at eq. (\ref{eq001}):
\begin{eqnarray}
J(\dist{Q}) \defeq KL_{\chi_{\tilde{{Q}}}}(\tilde{\dist{Q}}\|\dist{Q})
 & = & \frac{1}{1-q} \cdot \left(-\frac{1}{Z} + \frac{1}{Z^{2-q}} \cdot \int_{\ve{x}} Q^{q(2-q)}(\ve{x})
 \mathrm{d}\mu(\ve{x}) \right)\:\:.
\end{eqnarray}
Now, consider more specifically the $q \defeq 1/2$-Gaussian defined on the real
line, for which $Q(x) \defeq (A/\sigma)[1-(x^2/\sigma^2)]_+^2$. In this
case one can obtain that $Z = B \sqrt{\sigma}$, $B \defeq
4\sqrt{A}/3$, and so
\begin{eqnarray}
\int_{\ve{x}} Q^{q(2-q)}(\ve{x})
 \mathrm{d}\mu(\ve{x}) & = & \frac{A^{\frac{3}{4}}}{\sigma^{\frac{3}{4}}} \cdot \int_{-\sigma}^{\sigma} \left[
 1 - \frac{x^2}{\sigma^2}\right]^{\frac{3}{2}}_+ \mathrm{d}
x\nonumber\\
& = & A^{\frac{3}{4}}\sigma^{\frac{1}{4}} \cdot \int_{-1}^{1} \left[
 1 - x^2\right]^{\frac{3}{2}}_+ \mathrm{d}
x\nonumber\\
& = & \frac{3\pi}{8} \cdot A^{\frac{3}{4}}\sigma^{\frac{1}{4}} \:\:,
\end{eqnarray}
since
\begin{eqnarray}
\int_{-1}^{1} \left[
 1 - x^2\right]^{\frac{3}{2}}_+ \mathrm{d}
x & = & \frac{3\pi}{8} \:\: (> 1)\:\:.
\end{eqnarray}
So we obtain, taking into account that $A \defeq 15\sqrt{2}/32$,
\begin{eqnarray}
J(\dist{Q}) & = & \frac{3\pi}{4B^{\frac{3}{2}} \cdot
  \sigma^{\frac{3}{4}}} \cdot  A^{\frac{3}{4}}\sigma^{\frac{1}{4}} -
\frac{3}{2\sqrt{A}\sqrt{\sigma}}\nonumber\\
 & = & \frac{1}{\sqrt{\sigma}}\cdot \left( \frac{3^{\frac{3}{2}}\cdot
     3 \pi}{2\cdot 16}
   - \frac{3^{\frac{3}{2}}}{2\sqrt{15\sqrt{2}}} \right)  \nonumber\\
 & = & \frac{3^{\frac{3}{2}}}{2\sqrt{\sigma}}\cdot \left( \frac{
     3 \pi}{16}
   - \frac{1}{\sqrt{15\sqrt{2}}} \right)  \approx \frac{0.9663}{\sqrt{\sigma}}\nonumber\\
 & = & \theta\left(\frac{1}{\sqrt{\sigma}}\right)  \:\:,
\end{eqnarray}
and we can conclude for the proof of Lemma \ref{lemKLQQGAUSS}.
\end{proof}

\begin{figure}[t]
\begin{center}
\begin{tabular}{c}
\includegraphics[trim=150bp 320bp 650bp
200bp,clip,width=0.50\columnwidth]{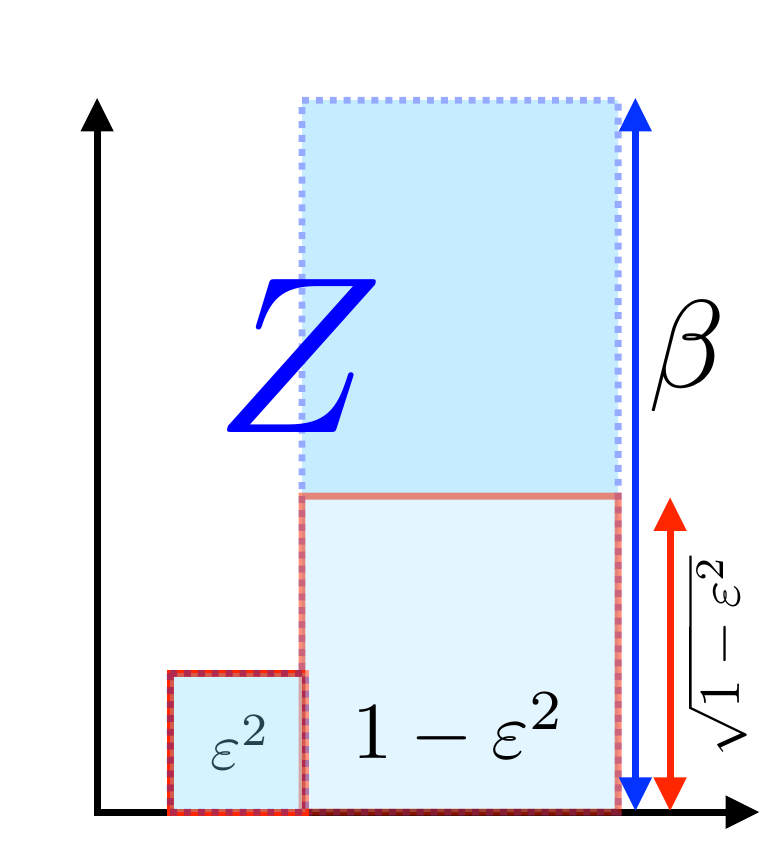} 
\end{tabular} 
\end{center}
\caption{Consider an initial density $Q$ given by two adjacent squares
  (in red). Using $\chi$ as defined in eq. (\ref{defCHIEX}), the
  largest values can be blown up $\beta$ as large as desired so that
  $Z$ (in blue) is in turn as large as desired (see text; Figure best
  seen in colour).}
\label{f-exsesc}
\end{figure}

We finish with (vi) and an example on how picking an escort that blows up large
values for a density can indeed make $Z$ very large. 
\begin{lemma}\label{lemKINFTY}
Fix $K>0$ and,
for some $0<\epsilon <1/4$, let
\begin{eqnarray}
\chi(z) & \defeq & \left\{
\begin{array}{rcl}
z & \mbox{ if } & z\leq \epsilon\\
\epsilon + \frac{1}{\log K}\cdot (K^{z-\epsilon}-1) & \mbox{ if } & z> \epsilon
\end{array}
\right. \:\:.\label{defCHIEX}
\end{eqnarray}
Consider the density $Q$ given in Figure \ref{f-exsesc}. Then, letting
$Z(Q)$ denote the normalization of the escort of $Q$, it holds that
$\lim_{K\rightarrow +\infty} Z(Q) = +\infty$.
\end{lemma}
\begin{proof}
It follows that
\begin{eqnarray}
Z & = & \epsilon^2 + \sqrt{1-\epsilon^2}\cdot \left(\epsilon +
  \frac{1}{\log K}\cdot (K^{\sqrt{1-\epsilon^2}-\epsilon}-1)
\right)\nonumber\\
 & = & \underbrace{\epsilon^2 + \sqrt{1-\epsilon^2}\cdot
   \left(\epsilon - \frac{1}{\log K}\right)}_{\defeq g_1(K,\epsilon)} + \underbrace{\frac{\sqrt{1-\epsilon^2}}{\log K}\cdot K^{\sqrt{1-\epsilon^2}-\epsilon}}_{\defeq g_2(K,\epsilon)}\:\:.
\end{eqnarray}
Provided $K\geq \exp(4) \geq \exp(1/\epsilon)$, $g_1(K,\epsilon)\geq
0$; since $\epsilon <1/4$, $\sqrt{1-\epsilon^2}-\epsilon \geq 1/2$ and
$\sqrt{1-\epsilon^2}\geq 1/2$. Hence, if $0<\epsilon <1/4$ and $K\geq
\exp(4)$, we have
\begin{eqnarray}
Z & \geq & \frac{\sqrt{K}}{2\log K}\:\:,
\end{eqnarray}
and we indeed have $\lim_{K\rightarrow +\infty} Z = +\infty$.
\end{proof}

\section{Many modes for GAN architectures}\label{sec_modes}

\begin{figure}{l}%
\begin{center}
 \includegraphics[trim=110bp 260bp 650bp
300bp,clip,width=0.50\columnwidth]{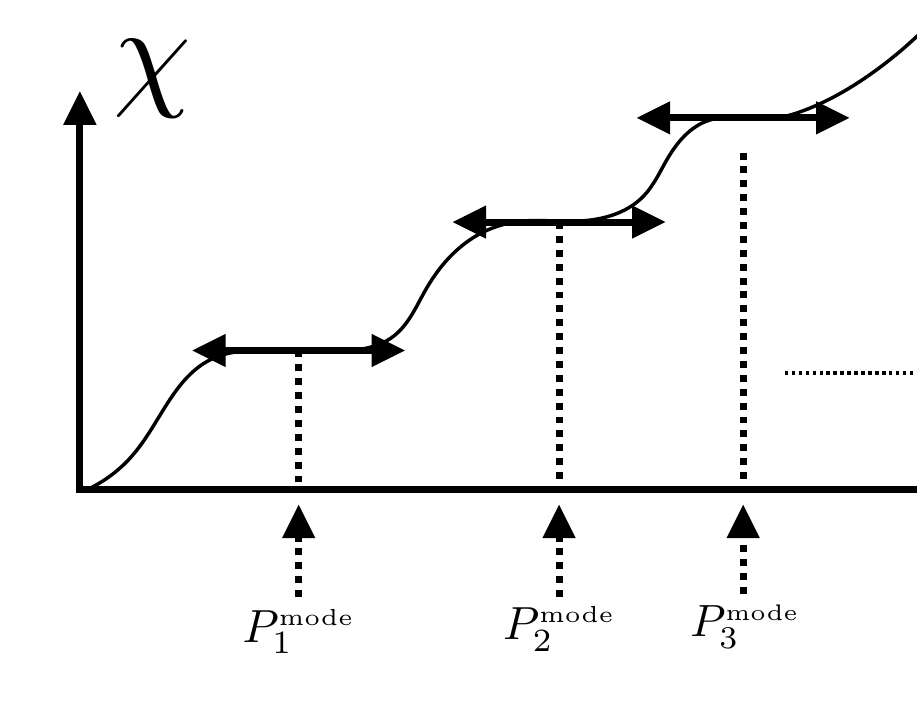}
\end{center}
  \caption{Taking a $\chi$ whose derivatives zeroes as many times as
    needed, and then putting the modes of each inner deformed
    exponential family where it zeroes makes it easy
    to accomodate as many modes as needed for a deep architecture.\label{f-modes}}\hspace{-2cm}
\end{figure}

In Section \ref{sec-deep-all}, we claim that a deep architecture working
under the general model specified in Section \ref{sec-deep-all} can
accomodate a number of modes of the order of the total dimension of
deep sufficient statistics, $\Omega(d\cdot L)$. To develop a simple
argument, assume that the last layer is the identity function so we do
hot have to care for $H_{\mbox{\tiny{out}}}(\ve{x})$ in Theorem \ref{factorDEEP}.
A simple argument for
this consists in three steps. First, we
  pick a $\chi$ like in Figure \ref{f-modes}, whose derivative is going
  to zero as many times as necessary. Then,
\begin{itemize}
\item ($\Omega(d\cdot L)$ critical points at the modes) first computing the critical points using the gradient
$\nabla Q_{g}(\ve{z})$ from Theorem \ref{factorDEEP}, which
yields:
\begin{eqnarray}
\nabla 
Q_{g}(\ve{z}) & \propto & \ve{\zeta}(\ve{z}) + \sum_{\newell=1}^{L}
\sum_{i=1}^d \zeta_{\newell, i} \cdot \chi_{\newell, i}'(P_{\chi, b_{\newell,
    i}}) \cdot \nabla_{\ve{z}}  (\ve{w}^\top_{\newell, i}
\ve{\phi}_{\newell-1})\:\:,
\end{eqnarray}
where the $\zeta$ functions are not important, since (i)
$\ve{\zeta}(\ve{z})$ is always the null vector for $Q_{\mbox{\tiny{in}}}$
  uniform, and (ii) $\zeta_{\newell, i}$ depends on $\ve{z}$ but we can assume
  it never zeroes (it factors the escort's density with a non zero
  constant). Then, using the $\chi$ as defined before, we choose the modes of the
  inner deformed exponential families (for which $\nabla_{\ve{z}}  (\ve{w}^\top_{\newell, i}
\ve{\phi}_{\newell-1}) = \ve{0}$) in such a way that they are
  located at the critical points of $\chi$, and a different one for each of them. We
  obtain a $Q_{g}(\ve{z}) $ which has up to (an order of) $d\cdot L$ critical points, exactly at all
  modes, as claimed;
\item (modes at all critical points) since each critical point
  is located at a mode for one of the densities, it is sufficient to
  ensure that the influence of all other densities in the
  curvature of the density is sufficiently small: for this, it is
  sufficient to then control the second derivative of $\chi$ in the neighborhood
  its critical points, making sure it does not exceed a small
  threshold in absolute value.
\end{itemize}
Notice that this property is independent from the
  one which allows to craft escorts that blow high density regions
  (and may yield large $Z$, Theorem \ref{thmKLQQ}, see also Lemma
  \ref{lemKINFTY} and Figure \ref{f-exsesc}), so we can combine both
  properties and obtain densities for the deep net with both high
  contrast around the modes and a large number of modes.

Of course, the $\chi$ we choose is very artificial and corresponds to
an activation which would be almost piecewise linear, a sort of
generalization of the ReLU activation with a large number of segments
or half lines instead of two. Yet, it gives some simple intuition as
to how fitting multimodal densities can indeed happen.

\section{Proof of Lemma \ref{lemGAME}}\label{proof_lemGAME}
The proof directly comes from Theorem \ref{thVIGGEN5}: the variational part in the
vig-$f$-GAN identity is:
\begin{eqnarray}
\lefteqn{\sup_{T \in \overline{\mathbb{R}_{++}}^{\mathcal{X}}} \left\{\expect_{\X\sim P} [T(\X)]  -
\expect_{\X\sim\tilde{Q}}
\left[-
  \log_{(\chi^\bullet)_{\frac{1}{\tilde{Q}}}}(-T(\X))\right]\right\}}\nonumber\\
 &
= & \sup_{T \in {\mathbb{R}_{++}}^{\mathcal{X}}} \left\{\expect_{\X\sim P} [-T(\X)]  +
\expect_{\X\sim\tilde{Q}}
\left[
  \log_{(\chi^\bullet)_{\frac{1}{\tilde{Q}}}}(T(\X))\right]\right\}\:\:,\label{eqFUND21}
\end{eqnarray}
and the DM vs G game in which DM's objective is to realize:
\begin{eqnarray}
\sup_{\Upsilon:
  \mathcal{X}\rightarrow \mathbb{R}} \{\expect_{\X\sim
  P}[-\Upsilon(\X)] + \expect_{\X\sim
  \mathcal{Q}}[\utility (\Upsilon(\X))]\}\label{defEU2}\:\:.
\end{eqnarray}
Making the correspondence between (\ref{defEU2}) and the right
hand-side of (\ref{eqFUND21}) gives the statement of the Lemma. 

\section{Proof of Lemma \ref{lemRISK}}\label{proof_lemRISK}
We recall the utility $\utility = u$ (for the sake of readability),
and note that it depends on the state of the world / observation $\ve{x}$,
\begin{eqnarray}
u_{\ve{x}} (z) & = &
\log_{(\chi^\bullet)_{\frac{1}{\escort{Q}(\ve{x})}}}(z)\:\:.
\end{eqnarray}
It follows from Lemma \ref{lemGAME} and the definition of $\chi$-logarithms,
\begin{eqnarray}
u_{\ve{x}}'(z) &  = & \frac{1}{\chi^{\bullet}_{\frac{1}{\tilde{Q}(\ve{x})}}(z)}\nonumber\\
 & =& \frac{1}{\tilde{Q}(\ve{x})} \cdot \chi^{-1}
 \left(\frac{\tilde{Q}(\ve{x})}{z}\right)\:\:,\\
u_{\ve{x}}''(z) & = &  -\frac{1}{z^2} \cdot (\chi^{-1})'
 \left(\frac{\tilde{Q}(\ve{x})}{z}\right) \nonumber\\
 & =& -\frac{1}{z^2 \cdot \chi' \left( \chi^{-1}
 \left(\frac{\tilde{Q}(\ve{x})}{z}\right)\right)} \:\: (\leq 0)\:\:.
\end{eqnarray}
Putting back all parameters, we obtain the following Arrow-Pratt
measure of absolute risk aversion:
\begin{eqnarray}
a_{u\ve{x}}(z) & = & -\frac{u''_{\ve{x}}(z)}{u'_{\ve{x}}(z)}\nonumber\\
 & = & \frac{\tilde{Q}(\ve{x})}{z^2 \cdot \chi^{-1}
 \left(\frac{\tilde{Q}(\ve{x})}{z}\right) \cdot \chi' \left( \chi^{-1}
 \left(\frac{\tilde{Q}(\ve{x})}{z}\right)\right)}\nonumber\\
 & = & \frac{1}{z}\cdot \frac{\frac{\tilde{Q}(\ve{x})}{z}}{\chi^{-1}
 \left(\frac{\tilde{Q}(\ve{x})}{z}\right) \cdot \chi' \left( \chi^{-1}
 \left(\frac{\tilde{Q}(\ve{x})}{z}\right)\right)}
\end{eqnarray}
Since $\chi : \mathbb{R}_+ \rightarrow \mathbb{R}_+$ and is non
decreasing, we see that 
\begin{eqnarray}
a_{u\ve{x}}(z) & \geq & 0, \forall \chi, z, \ve{x}\:\:,
\end{eqnarray}
and therefore player DM is \textit{always} risk-averse (this proves
point (i)).
Define
\begin{eqnarray}
g(z) & \defeq & \frac{z \cdot (\chi^{-1})'(z)}{\chi^{-1}(z)}\:\:,
\end{eqnarray}
so that
\begin{eqnarray}
a_{u\ve{x}}(z) & = & \frac{1}{z}\cdot g\left( \frac{z}{\tilde{Q}(\ve{x})}\right)\:\:,
\end{eqnarray}
and
\begin{eqnarray}
r_{u\ve{x}}(z) & = & z \cdot a_{u\ve{x}}(z) = g\left( \frac{z}{\tilde{Q}(\ve{x})}\right)\:\:.
\end{eqnarray} 
This proves point (ii). Notably, at the optimum, the dependency on the subjective beliefs
disappears since the optimum (Theorem \ref{thVIGGEN4}) of
$\Upsilon^* = -T^*$ yields:
\begin{eqnarray}
r_{u\ve{x}}(\Upsilon^*(\ve{x})) & = & g\left( \frac{\chi(Q(\ve{x}))}{Z
    \cdot \chi(P(\ve{x}))}\cdot
  \frac{1}{\tilde{Q}(\ve{x})}\right) = g\left( \frac{1}{\chi(P(\ve{x}))}\right)\:\:.
\end{eqnarray} 

\begin{remark}
We can make a connection with a more traditional view of portfolio
allocation. The first order conditions for \eqref{defEU} gives us
\begin{gather}
    (\forall{\ve{x}\in\mathcal X})\quad
    \utility'(\Upsilon^*(\ve{x}))\cdot  \mathcal{Q}(\ve{x})- P(\ve{x})= 0 \implies \utility'(\Upsilon^*(\ve{x})) =\frac{ P(\ve{x})}{ \mathcal{Q}(\ve{x})}.\label{eq:marginal_utility}
\end{gather}
From \eqref{eq:marginal_utility} we see that at optimality (equilibrium), Decision Maker picks a portfolio such that his marginal utility over the risky asset under each state $\ve{x} \in\mathcal X$ is equal to the corresponding odds ratio. If Decision Maker (for whatever reason) suddenly believes a certain state $\ve{x}_0\in\mathcal X$ is more likely ($\mathcal{Q}(\ve{x}_0)$ goes up), then  with usual assumptions about decreasing marginal utility, it's intuitive that he will respond by consuming more of $\Upsilon(\ve{x}_0)$. Similarly if compare decision makers with differing risk aversion in the risky asset, the more risk averse decision maker must hold much stronger beliefs to consume at the same level as the less risk averse decision maker. The optimality condition \eqref{eq:marginal_utility} is illustrated in Figure~\ref{fig:marginal_utility}.
    
\begin{figure}
    \begin{tikzpicture}
        \begin{axis}[width=6cm,height=4cm,xtick=\empty,ytick=\empty,xmin=0,xmax=4, ymin=0, ymax=2.5, ylabel={$u(\Upsilon(x))$},xlabel={$\Upsilon(x)$}, clip=false,
            axis line style={->},
            axis lines=middle,
            every axis x label/.style={at={(current axis.right of origin)},anchor=west},
            every axis y label/.style={at={(current axis.above origin)},anchor=south}
        ]
            \addplot[thick, smooth, samples=400, blue, domain=0:4] {x^(1/2.5) + 0.25*x^(1/2)} node[anchor = west] {$u_1$};;
            \coordinate (phi) at (axis cs: 1, 1.25);
            \coordinate (phi2) at (axis cs: 2.44845, 1.82191);
            \addplot[smooth, samples=400, domain=0:4, gray] {0.525*x + 0.725};
            \addplot[smooth, samples=400, domain=0:4, gray] {0.313621*x + 1.05403};
            \draw[dashed] (phi) -- (phi |- {{(0,0)}});
            \draw[dashed] (phi2) -- (phi2 |- {{(0,0)}});
            \draw (phi) node[draw, blue, thick, circle, fill=blue, inner sep=1pt] {};
            \draw (phi2) node[draw, blue, thick, circle, fill=blue, inner sep=1pt] {};
            \path (phi |- {{(0,0)}}) edge[shorten <=5pt,shorten >=5pt, gray, ->, bend left] (phi2 |- {{(0,0)}});
        \end{axis}
        \draw (phi |- {{(0,0)}}) node[draw, black, thick, circle, fill=black, inner sep=1pt] {} ;
        \draw (phi2 |- {{(0,0)}}) node[draw, black, thick, circle, fill=black, inner sep=1pt] {} node[anchor=north] {$\Upsilon^*(x)$};
    \end{tikzpicture}
    \quad
    \begin{tikzpicture}
        \begin{axis}[width=6cm,height=4cm,xtick=\empty,ytick=\empty,xmin=0,xmax=4, ymin=0, ymax=2.5, ylabel={$u(\Upsilon(x))$},xlabel={$\Upsilon(x)$}, clip=false,
            axis line style={->},
            axis lines=middle,
            every axis x label/.style={at={(current axis.right of origin)},anchor=west},
            every axis y label/.style={at={(current axis.above origin)},anchor=south}
        ]
            \addplot[thick, smooth, samples=400, blue, domain=0:4] {x^(1/2.5)} node[anchor = west] {$u_2$};
            \addplot[thick, smooth, samples=400, blue, domain=0:4] {x^(1/2.5) + 0.25*x^(1/2)} node[anchor = west] {$u_1$};
            \coordinate (phi) at (axis cs: 2.44845, 1.82191);
            \coordinate (phi2) at (axis cs: 2.44845, 1.43073);
            \addplot[smooth, samples=400, domain=0:4, gray] {0.233736*x + 0.858436};
            \addplot[smooth, samples=400, domain=0:4, gray] {0.313621*x + 1.05403};
            \draw[dashed] (phi) -- (phi |- {{(0,0)}});
            \draw (phi) node[draw, blue, thick, circle, fill=blue, inner sep=1pt] {};
            \draw (phi2) node[draw, blue, thick, circle, fill=blue, inner sep=1pt] {};
        \end{axis}
        \draw (phi |- {{(0,0)}}) node[draw, black, thick, circle, fill=black, inner sep=1pt] {} node[anchor=north] {$\Upsilon^*(x)$};
        \draw (phi2 |- {{(0,0)}}) node[draw, black, thick, circle, fill=black, inner sep=1pt] {} ;
    \end{tikzpicture}
    \centering
    \caption{Illustration of the optimality conditions for the insurance problem \eqref{defEU}, wherein we pick the risky asset $\Upsilon$ such that for every $\ve{x}\in\mathcal{X}$, marginal utility is equal to the odds ratio. That is, the tangent with slope $P(\ve{x})/\mathcal{Q}(\ve{x})$. (First Diagram) If DM suddenly believes a certain state state $x\in \cal X$ is more likely, this pushes the optimal function $\Upsilon^*$ to the right. (Second Diagram) If we compare two decision makers utilities  $u_1$ and $u_2$ such that $u_2$ is more risk averse than $u_1$. We see that for $u_2$ to consume at the same level as $u_1$, $u_2$ must believe $x\in\cal X$ is far more likely. } \label{fig:marginal_utility}
\end{figure}
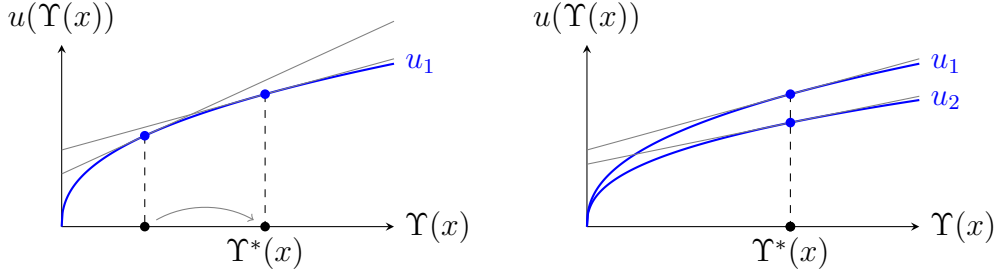
\end{remark}

\section*{--- Appendix on experiments}\label{exp_expes}

\section{Architectures}\label{exp_archis}

We consider two architectures in our experiments: DCGAN~\cite{rmcUR} and the multilayer feedforward network (MLP) used in~\cite{nctFG}. Suppose the size of input images is isize-by-isize, the details of architectures are given as follows:

\paragraph{Generator of DCGAN}:

  ConvTranspose(input=100, output=8$\times$isize, stride=1) $\rightarrow$
  BatchNorm$\rightarrow$
   Activation$\rightarrow$ Conv(input=8$\times$isize, output=4$\times$isize, stride=2, padding=1)$\rightarrow$
    BatchNorm$\rightarrow$
    Activation$\rightarrow$
     ConvTranspose(input=4$\times$isize, output=2$\times$isize, stride=2, padding=2)$\rightarrow$
    BatchNorm$\rightarrow$
    Activation$\rightarrow$    
    ConvTranspose(input=2$\times$isize, output= isize, stride=2, padding=1)$\rightarrow$
    BatchNorm$\rightarrow$
    Activation  $\rightarrow$ Conv(isize, number of channel, stride=2, padding=1) $\rightarrow$
    Last Activation
    
\paragraph{Discriminator of DCGAN}:

  Conv(1, 2$\times$isize, stride=2) $\rightarrow$
  BatchNorm$\rightarrow$
   LeakyReLU$\rightarrow$ Conv(input=2$\times$isize, output=4$\times$isize, stride=2, padding=1)$\rightarrow$
    BatchNorm$\rightarrow$
    LeakyReLU$\rightarrow$
     Conv(input=4$\times$isize, output=8$\times$isize, stride=2, padding=2)$\rightarrow$
    BatchNorm$\rightarrow$
    LeakyReLU$\rightarrow$    
    Conv(input=8$\times$isize, output= 1, stride=2, padding=1)$\rightarrow$
    Link function  
  
\paragraph{Generator of MLP}:

 $z$  $\rightarrow$ Linear(100, 1024) $\rightarrow$ BatchNorm $\rightarrow$ Activation $\rightarrow$ Linear(1024, 1024) $\rightarrow$ BatchNorm $\rightarrow$ Activation $\rightarrow$ Linear(1024, isize$\times$isize) $\rightarrow$ last Activation

\paragraph{Discriminator of MLP}:

$x$ $\rightarrow$ Linear(isize$\times$isize, 1024) $\rightarrow$ ELU $\rightarrow$ Linear(1024, 1024) $\rightarrow$ ELU $\rightarrow$  Linear(1024, 1) $\rightarrow$ Link function

\section{Experimental setup for varying the activation function in the
generator}\label{exp_gen}

\paragraph{Setup.} We train adversarial networks with varying activation functions for the generators on the MNIST~\cite{lecun1998gradient} and LSUN~\cite{yu15lsun} datasets. In particular, we compare ReLU, Softplus, Least Square loss as an example of prop-$\tau$, and $\mu$-ReLU with varying $\mu$ in $[0, 0.1, ... , 1]$ by using them as the activation functions in all hidden layers of the generators. For all models, we fix the learning rate to 0.0002 and batch size to 64 throughout all experiments after tuning on a hold-out set. 

\paragraph{MNIST.} We evaluate the activation functions by using both
DCGAN and the MLP used in~\cite{nctFG} as the architectures. As
training divergence, we adopt both GAN and Wasserstein distance (WGAN)
because GAN belongs to variational $f$-divergence formulation while WGAN does not. The link function of the discriminators is specific to the respective divergence, which is sigmoid for GAN and linear for WGAN. We sample random noise $z \in \text{Uniform}_{100} (0, 1)$ for MLP and  $z \in \text{Gaussian} (0, 1)$ for DCGAN, which is found slightly better than sampling from $\text{Uniform}_{100} (-1, 1)$. As the best practice, we apply Adam~\cite{corr/KingmaB14} to optimize models with GAN and RMSprop~\cite{tieleman2012lecture} to optimize WGAN based models. For GAN, we train one batch for discriminator and one batch for generator iteratively during training. For WGAN, we apply weight clipping with 0.01 and train five batches for discriminator and one batch for generator interchangeably during training.

We train all models on the full MNIST training data set and evaluate the performance on the test set by using the kernel density estimation (KDE). Since the size of images accepted by DCGAN should be n-fold of 16, all images are rescaled to 32-by-32 for all models. Following~\cite{nctFG}, we apply three-fold cross validation to find optimal bandwidth for the isotropic Gaussian kernel of KDE on a hold-out set.  To estimate the log probability of the test set, we sample 16k images from the models in the same way as~\cite{nctFG}. We observe that the initialization of model parameters has significant influence on performance. Therefore, we conduct three runs with different random seeds for each experimental setting and report the mean and standard deviation of the results.

\paragraph{LSUN.} We also evaluate all activation functions in consideration for the generator on LSUN natural scene images. We train DCGAN with GAN as the divergence on the \textit{tower} category of images, which are rescaled and center-cropped to 64-by-64 pixels, as in~\cite{rmcUR}. Due to the center-cropped images, we apply $\tanh$ as last activation of generators instead of sigmoid for GAN based models.

\section{Visual results on MNIST}\label{exp_mnist}

\begin{table}
    \centering
\begin{tabular}{cc|cc}\\ \hline \hline
 $\mu$ & & $\mu$ & \\ \hline
$0$
& \includegraphics[width=0.45\columnwidth]{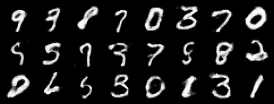}&   $0.1$
& \includegraphics[width=0.45\columnwidth]{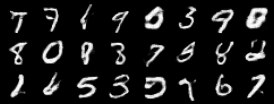}\\  
$0.2$
& \includegraphics[width=0.45\columnwidth]{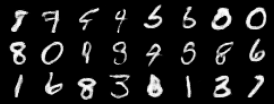}&   $0.3$
& \includegraphics[width=0.45\columnwidth]{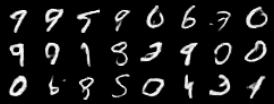}\\  
$0.4$
& \includegraphics[width=0.45\columnwidth]{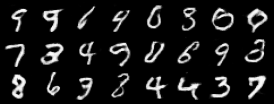}&   $0.5$
& \includegraphics[width=0.45\columnwidth]{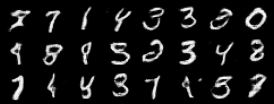}\\  
$0.6$
& \includegraphics[width=0.45\columnwidth]{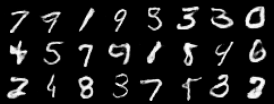}&   $0.7$
& \includegraphics[width=0.45\columnwidth]{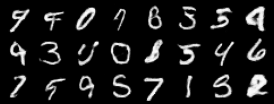}\\  
$0.8$
& \includegraphics[width=0.45\columnwidth]{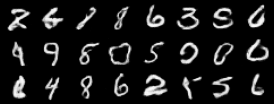}&   $0.9$
& \includegraphics[width=0.45\columnwidth]{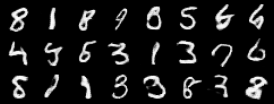}\\  \hline
$1$ & \multicolumn{3}{c}{\includegraphics[width=0.9\columnwidth]{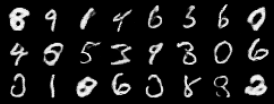}}\\  \hline\hline
\end{tabular}
\caption{MNIST results for GAN$\_$DCGAN at varying $\mu$ ($\mu=1$ is ReLU).}\label{tab:expes_visu_mnist_gan_dcgan}
\end{table}

\begin{table}
    \centering
\begin{tabular}{cc|cc}\\ \hline \hline
 $\mu$ & & $\mu$ & \\ \hline
$0$
& \includegraphics[width=0.45\columnwidth]{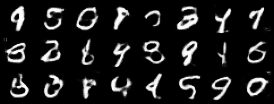}&   $0.1$
& \includegraphics[width=0.45\columnwidth]{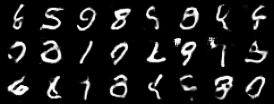}\\  
$0.2$
& \includegraphics[width=0.45\columnwidth]{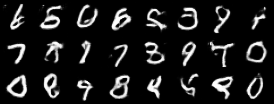}&   $0.3$
& \includegraphics[width=0.45\columnwidth]{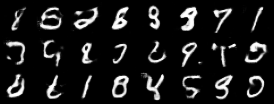}\\  
$0.4$
& \includegraphics[width=0.45\columnwidth]{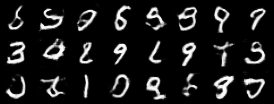}&   $0.5$
& \includegraphics[width=0.45\columnwidth]{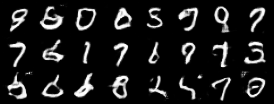}\\  
$0.6$
& \includegraphics[width=0.45\columnwidth]{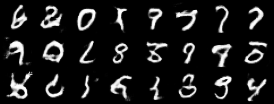}&   $0.7$
& \includegraphics[width=0.45\columnwidth]{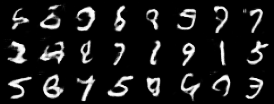}\\  
$0.8$
& \includegraphics[width=0.45\columnwidth]{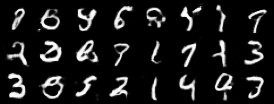}&   $0.9$
& \includegraphics[width=0.45\columnwidth]{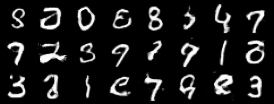}\\  \hline
$1$ & \multicolumn{3}{c}{\includegraphics[width=0.9\columnwidth]{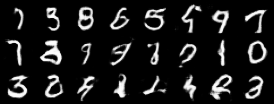}}\\  \hline\hline
\end{tabular}
\caption{MNIST results for WGAN$\_$DCGAN at varying $\mu$ ($\mu=1$ is ReLU).}\label{tab:expes_visu_mnist_wgan_dcgan}
\end{table}

\begin{table}
    \centering
\begin{tabular}{cc|cc}\\ \hline \hline
 $\mu$ & & $\mu$ & \\ \hline
$0$
& \includegraphics[width=0.45\columnwidth]{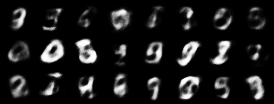}&   $0.1$
& \includegraphics[width=0.45\columnwidth]{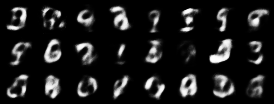}\\  
$0.2$
& \includegraphics[width=0.45\columnwidth]{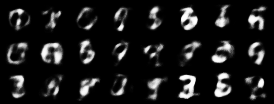}&   $0.3$
& \includegraphics[width=0.45\columnwidth]{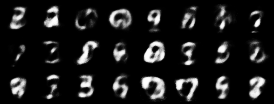}\\  
$0.4$
& \includegraphics[width=0.45\columnwidth]{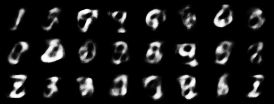}&   $0.5$
& \includegraphics[width=0.45\columnwidth]{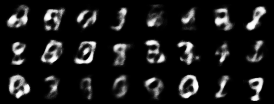}\\  
$0.6$
& \includegraphics[width=0.45\columnwidth]{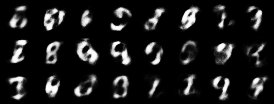}&   $0.7$
& \includegraphics[width=0.45\columnwidth]{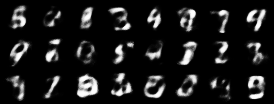}\\  
$0.8$
& \includegraphics[width=0.45\columnwidth]{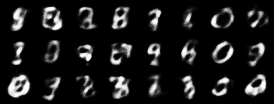}&   $0.9$
& \includegraphics[width=0.45\columnwidth]{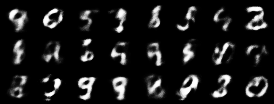}\\  \hline
$1$ & \multicolumn{3}{c}{\includegraphics[width=0.9\columnwidth]{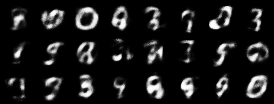}}\\  \hline\hline
\end{tabular}
\caption{MNIST results for WGAN$\_$MLP at varying $\mu$ ($\mu=1$ is ReLU).}\label{tab:expes_visu_mnist_wgan_mlp}
\end{table}

\begin{table}
    \centering
\begin{tabular}{cc|cc}\\ \hline \hline
 $\mu$ & & $\mu$ & \\ \hline
$0$
& \includegraphics[width=0.45\columnwidth]{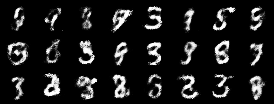}&   $0.1$
& \includegraphics[width=0.45\columnwidth]{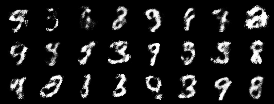}\\  
$0.2$
& \includegraphics[width=0.45\columnwidth]{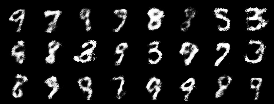}&   $0.3$
& \includegraphics[width=0.45\columnwidth]{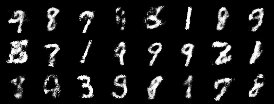}\\  
$0.4$
& \includegraphics[width=0.45\columnwidth]{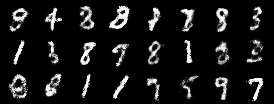}&   $0.5$
& \includegraphics[width=0.45\columnwidth]{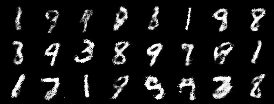}\\  
$0.6$
& \includegraphics[width=0.45\columnwidth]{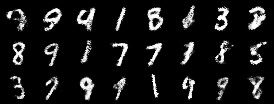}&   $0.7$
& \includegraphics[width=0.45\columnwidth]{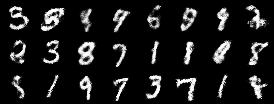}\\  
$0.8$
& \includegraphics[width=0.45\columnwidth]{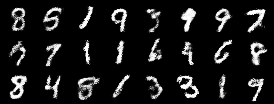}&   $0.9$
& \includegraphics[width=0.45\columnwidth]{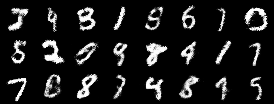}\\  \hline
$1$ & \multicolumn{3}{c}{\includegraphics[width=0.9\columnwidth]{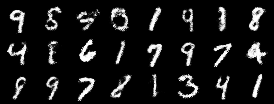}}\\  \hline\hline
\end{tabular}
\caption{MNIST results for GAN$\_$MLP at varying $\mu$ ($\mu=1$ is ReLU).}\label{tab:expes_visu_mnist_gan_mlp}
\end{table}

\begin{table}
    \centering
\begin{tabular}{cc|cc}\\ \hline \hline
 $\mu$ & & $\mu$ & \\ \hline
$0$
& \includegraphics[width=0.45\columnwidth]{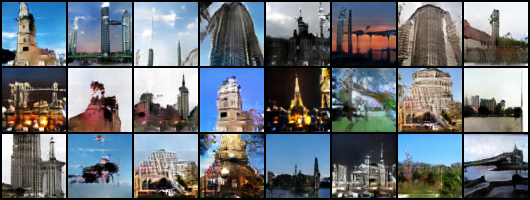}&   $0.1$
& \includegraphics[width=0.45\columnwidth]{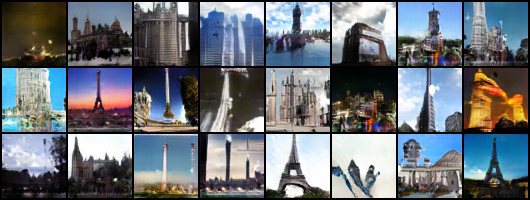}\\  
$0.2$
& \includegraphics[width=0.45\columnwidth]{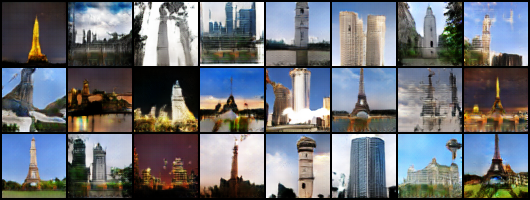}&   $0.3$
& \includegraphics[width=0.45\columnwidth]{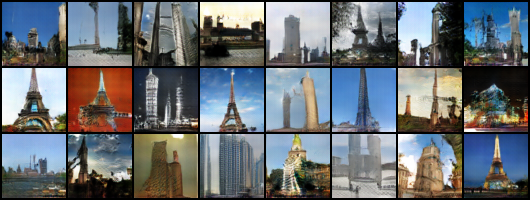}\\  
$0.4$
& \includegraphics[width=0.45\columnwidth]{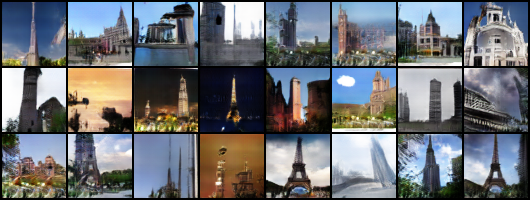}&   $0.5$
& \includegraphics[width=0.45\columnwidth]{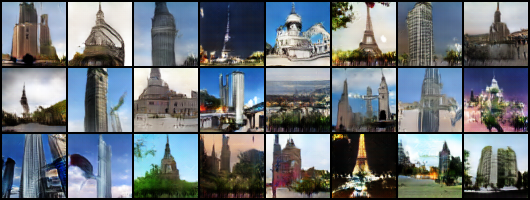}\\  
$0.6$
& \includegraphics[width=0.45\columnwidth]{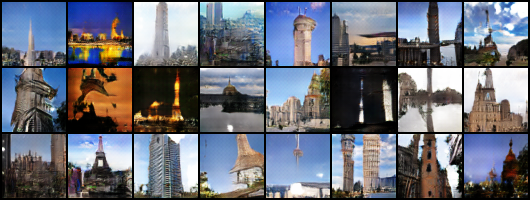}&   $0.7$
& \includegraphics[width=0.45\columnwidth]{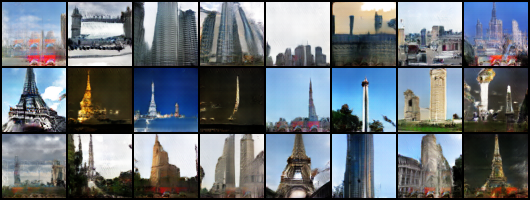}\\  
$0.8$
& \includegraphics[width=0.45\columnwidth]{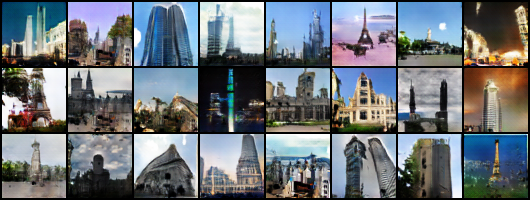}&   $0.9$
& \includegraphics[width=0.45\columnwidth]{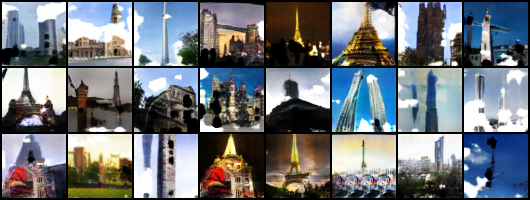}\\  \hline
$1$ & \multicolumn{3}{c}{\includegraphics[width=0.9\columnwidth]{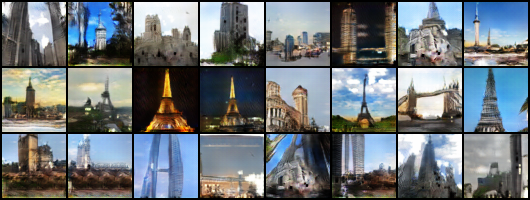}}\\  \hline\hline
\end{tabular}
\caption{LSUN results for GAN$\_$DCGAN at varying $\mu$ ($\mu=1$ is ReLU).}\label{tab:expes_visu_lsun}
\end{table}


\end{document}